\documentclass[11pt]{article}

\usepackage{amssymb,amsmath,amsthm}
\usepackage{bbm}
\usepackage[noend]{algpseudocode}
\usepackage{graphicx}
\usepackage{verbatim}
\usepackage{url,xspace,hyperref}

\usepackage{caption}
\usepackage{subcaption}
\usepackage{multirow}
\usepackage{multicol}
\usepackage{latexsym}
\usepackage{amsmath,amssymb,enumerate}
\usepackage{algorithm,algorithmicx}
\usepackage{float}
\usepackage{xcolor}
\usepackage{mathrsfs}
\usepackage{cleveref}
\usepackage{booktabs}
\usepackage{breqn}
\usepackage{xspace}

\usepackage{fullpage}
\usepackage{geometry}
\usepackage{setspace}
\usepackage{tcolorbox}

\tcbuselibrary{skins,breakable}
\tcbset{enhanced jigsaw}

\newtheorem{assumption}{Assumption}[section]
\newtheorem{theorem}{Theorem}[section]
\newtheorem*{theorem*}{Theorem}
\newtheorem{definition}{Definition}[section]

\newtheorem{claim}[theorem]{Claim}

\newtheorem{lemma}[theorem]{Lemma}

\newtheorem{remark}[theorem]{Remark}

\def\sse{\subseteq}

\newcommand{\pr}{\Pr} 
\newcommand{\E}{\mathbb{E}}

\newcommand{\h}{\mathbf{h}} 
 
\newcommand{\hp}{\widehat{p}} 
 
\newcommand{\bp}{\bar{p}} 
\newcommand{\R}{\mathbb{R}}
\newcommand{\bE}{\mathbf{E}}
\newcommand{\cE}{\mathcal{E}}

\newcommand{\cC}{\mathcal{C}}

\newcommand{\cP}{\mathcal{P}}
\newcommand{\calE}{\mathcal{E}}
\newcommand{\D}{\mathcal{D}}
\newcommand{\bD}{\mathbf{D}}

\newcommand{\Ot}{\widetilde{O}}
\newcommand{\OPT}{\mathtt{OPT}}

\newcommand{\ALG}{\mathtt{ALG}}
\newcommand{\TV}{\mathtt{TV}}
\newcommand{\TVD}{\operatorname{\mathtt{TV}-\mathtt{Dist}}}
\newcommand{\ESD}{\mathtt{EmpStocDom}}

\newcommand{\ignore}[1]{}

\newcommand{\bld}[1]{\mathbf{{#1}}}

\newcommand{\bX}{\mathbf{x}}
\newcommand{\varX}{\mathsf{X}}
\newcommand{\varY}{\mathsf{Y}}
\newcommand{\var}[1]{\mathsf{#1}}
\newcommand{\bx}{\mathbf{x}}

\newcommand{\calD}{\mathcal{D}}
\newcommand{\calI}{\mathcal{I}}

\newcommand{\1}{\mathbf{I}}

\newcommand{\bern}{\mathtt{Bern}}

\renewcommand{\color}[2][]{}
\newenvironment{tbox}{\begin{tcolorbox}[
		enlarge top by=5pt,
		enlarge bottom by=5pt,
		 breakable,
		 boxsep=0pt,
                  left=4pt,
                  right=4pt,
                  top=10pt,
                  arc=0pt,
                  boxrule=1pt,toprule=1pt,
                  colback=white
                  ]
	}
{\end{tcolorbox}}

\allowdisplaybreaks

\usepackage[round]{natbib}
\definecolor{mydarkblue}{rgb}{0,0.08,0.45}
\definecolor{mydarkred}{rgb}{0.75, 0.27, 0.34}
\hypersetup{
    colorlinks=true,
    citecolor=mydarkblue,
    linkcolor=mydarkred,
    }

\newif\ifConfVersion

\title{Semi-Bandit Learning  for Monotone Stochastic Optimization} 
\author{Arpit Agarwal\thanks{Department of Computer Science \& Engineering, Indian Institute of Technology Bombay, Mumbai, India.} \and Rohan Ghuge\thanks{Department of Information, Risk, and Operations Management, University of Texas at Austin, Austin, USA.} \and Viswanath Nagarajan\thanks{Department of Industrial and Operations Engineering, University of Michigan, Ann Arbor, USA. Research supported in part by NSF grant  CCF-2418495.}  \and Zhengjia Zhuo$^\ddagger$}

\begin{document}

\maketitle

\thispagestyle{empty}
\begin{abstract}
Stochastic optimization is a widely used approach for optimization under uncertainty, where uncertain input parameters are modeled by random variables. Exact or  approximation  algorithms have been obtained for several fundamental problems in this area. However, a significant limitation of this approach is that it requires full knowledge of the  underlying probability distributions. {\em Can we still get good (approximation) algorithms  if these distributions are unknown, and the algorithm  needs to learn them  through repeated interactions?} In this paper, we resolve this question for a large class of ``monotone'' stochastic problems, by providing a generic  online learning algorithm with $\sqrt{T \log T}$ regret relative to  the best approximation algorithm (under known distributions). Importantly, our online algorithm works in a semi-bandit setting,  where in each period, the algorithm only observes samples from the random variables that were actually probed. Moreover,  our result  extends to settings with censored and binary feedback, where the policy only observes truncated or thresholded versions of the probed variables.  
Our framework applies to several fundamental problems  such as prophet inequality, Pandora's box,   stochastic knapsack,
single-resource revenue management and sequential posted pricing.

\end{abstract}

\newpage

\setcounter{tocdepth}{2}
 {\small

    \tableofcontents
    \thispagestyle{empty}

 }
\newpage

\setcounter{page}{1}

\section{Introduction}\label{sec:intro}

Stochastic optimization problems have been a subject of intense investigation,  
as they offer a powerful lens through which we can handle  uncertain inputs. In stochastic problems, uncertain input parameters are modeled by random variables, which are  usually independent. Solutions (or policies) to  a stochastic  problem are   sequential decision processes, where the next decision   depends on all previously observed information.  
This approach has been applied to many domains: optimal stopping~\citep{weitzman1979,samuel1984comparison,kleinberg2012matroid}, revenue management~\citep{GT19}, mechanism design~\citep{Myerson81,ChawlaHMS10},
submodular optimization~\citep{asadpour2016maximizing,GolovinK-arxiv,im2016minimum},
stochastic probing~\citep{GN13,AdamczykSW16,gupta2017adaptivity} and various other stochastic combinatorial optimization problems such as   knapsack~\citep{deshpande2016approximation,DGV08} and matching~\citep{ChenIKMR09,BGLMNR12}. 
A fundamental assumption in all these results is that the underlying probability distributions are known to the algorithm. While the known-distributions assumption holds in some settings, it is not satisfied in many practical applications, e.g., in the absence of historical data. 
Our main goal in this paper is to relax this assumption and handle stochastic problems with {\em unknown distributions}. 
We start by providing some concrete examples.

\medskip \noindent
{\bf Series testing.} There is  a system with $n$ components, where each component $i$ is ``working'' independently with some probability $p_i$. All $n$ components must be working for the system to be functional. Moreover, it costs $c_i$ to test component $i$ and determine whether/not it is working. 
The goal is to test components sequentially to determine whether/not the system is functional, at the minimum expected cost. Clearly, testing will continue until some component is found to be not working, or we have tested all components and found them to be working. In the standard setting, where the probabilities $\{p_i\}_{i=1}^n$ are  known  upfront, the greedy policy that tests components in decreasing order of $\frac{1-p_i}{c_i}$ is optimal~\citep{Butterworth72}.

\medskip \noindent
{\bf Pandora's box.} There are 
$n$ items, where each item $i$ has
an independent random value $\varX_i$ with a known distribution $\calD_i$. The realized value of $\varX_i$ can be revealed by inspecting item $i$, which incurs cost $c_i$. 
The goal is to inspect a subset $S$ of the items  
(possibly adaptively) to maximize the expected utility $\E \left[\max_{i \in S} \varX_i - \sum_{i \in S} c_i\right]$. We emphasize that a solution to this problem can be quite complex: the choice of the next item to inspect may depend on the realizations of all previous items. Nevertheless, there is an elegant optimal solution to this problem when the distributions $\{\calD_i\}_{i=1}^n$ are known~\citep{weitzman1979}.

\medskip \noindent
{\bf Single Resource Revenue Management.} 
An airline  has $C$  seats  for sale across $n$ fare classes, where each seat sold in class $i$ has a fixed revenue $p_i$. The  customer demand for each fare class $i$ is  an  independent random variable $\varX_i\sim \calD_i$. Customers preferring lower fares arrive earlier---a common assumption in practice. The airline  allocates seats  to different fare classes sequentially. If  $m_i$ seats are offered  to  fare class $i$ then the associated revenue is $p_i\cdot \E[\min(X_i, m_i)]$. The number of seats offered to class $i$ may depend on the actual sales in prior classes. The objective is to maximize the expected total revenue. There is an exact dynamic programming algorithm for this problem~\citep{GT19}.

\medskip \noindent {\bf Sequential Posted Pricing.  } 
There are $n$ buyers for an item, where each buyer $i$ has a random valuation $\varX_i$ with known distribution $\calD_i$. The seller wants to choose   a take-it-or-leave-it price $p_i$ for each buyer $i$ so as to maximize the expected revenue  when the   buyers are offered these prices in some order. The first buyer $i$ whose value $\varX_i$ exceeds their posted price $p_i$ buys the item, and the seller receives  revenue $p_i$. There are two  variants  of this problem, depending on whether the buyer ordering is fixed or can be chosen by the algorithm. 
There is an exact dynamic program for the fixed-order setting and a $\left(1-\frac1e\right)$ approximation algorithm for the free-order  setting~\citep{Yan11}.

\medskip
In each of the above problems, what if  the underlying probabilities/distributions  are unknown? Is it still possible to obtain good performance guarantees? In order to address this question formally, we utilize the framework of \emph{online learning}. Here,  the algorithm interacts with an unknown-but-fixed  input distribution $\calD$ over multiple time periods.
In each period 
$t =1,\cdots T$, the online algorithm comes up with a solution/policy $\sigma^{t}$ to the stochastic problem  and receives some {\em feedback} based on the performance of  $\sigma^t$ on the (unknown) distribution $\calD$.

The type of feedback received is a crucial component in this learning-based framework. The simplest setting is {\em full feedback}, where the algorithm receives one sample from every random variable (r.v.). However, this is unrealistic for stochastic problems because the policy  $\sigma^t$ in period $t$ only observes the r.v.s corresponding to some {\em subset} of the items. For example, in series testing,   $\sigma^t$ tests the $n$ components in some order until the first non-working component is found: if the first $k-1$ components are working and the $k^{th}$ component is not working then we would only observe the outcomes/samples of the first  $k$ r.v.s (and the remaining $n-k$ r.v.s are unobserved). In order to address  this aspect, we consider the  more natural setting of {\em semi-bandit feedback}, where in each period $t$, the algorithm only receives samples from the r.v.s that the policy  $\sigma^t$ actually observed. 
We further study two other  restricted feedback models— {\em censored} and {\em binary} feedback—that naturally arise in applications like revenue management and posted pricing, respectively.

Our goal is to minimize the expected \emph{regret} of the online algorithm, which is the difference
between the total $T$-period objective  of our algorithm and the optimum (which knows the distribution $\calD$).  Obtaining $o(T)$  regret dependence with respect to $T$ means that our online algorithm,  which doesn't know $\calD$, asymptotically approaches the optimum.  
An additional  difficulty arises from the fact that many  problems that we consider are NP-hard (e.g., knapsack and submodular optimization). In such cases, we cannot hope to approach the optimum via a polynomial algorithm. Here, we use the notion of $\alpha$-regret, where the online algorithm approaches the $\alpha$ approximately optimal value.  

Our main result is
a general method 
for transforming any offline  $\alpha$-approximation algorithm (with known distributions) to an online learning algorithm with $\alpha$-regret of $\widetilde{O}(\sqrt{T})$.  It is well-known that the $\sqrt{T}$ regret bound cannot be improved, even in very special cases. Our method works for a wide range of stochastic problems that satisfy a natural \emph{monotonicity} condition. This  includes all the  above examples as well as  several other fundamental   stochastic optimization problems. 

\subsection{Problem Setup}\label{sec:problem}
\def\s{\bot}
We first set-up notation for stochastic optimization problems in the known distribution setting.

\paragraph{Stochastic Optimization} Consider a stochastic problem  $\cP$ where the input consists of $n$ items with an independent real-valued r.v. $\varX_i$  associated with each item $i\in [n]=\{1,2,\cdots n\}$.
There is a \emph{known} probability distribution $\calD_i$ for $\varX_i$; that is, $\varX_i \sim \calD_i$.
We will primarily consider the setting  where each $\calD_i$ has finite support, and denote the set of outcomes of all r.v.s by $O$. (We also consider continuous distributions
in \S\ref{sec:continuous}.)  
In order to determine the realization of any r.v. $\varX_i$, we need to select or \emph{probe} item $i$.

A solution or  {\em policy}  for $\cP$  is given by a decision tree $\sigma$, where each node is labeled by an item to probe next, and the branches out of a node correspond to the random realization of the probed item. Each node in decision tree $\sigma$ also 
corresponds to the current ``state'' of the policy, which is given by the sequence of previously-probed items along with their realizations; we  will refer to nodes and states of the policy interchangeably. The root node of $\sigma$ is the starting state of the policy, at which point no item has been probed. Leaf nodes in $\sigma$ are also called {\em terminal} nodes/states. 
The {\em policy execution} under any realization $\bX= ( x_1, \ldots, x_n)$ corresponds to a root-leaf path $\sigma_\bX$ in the decision tree $\sigma$, where at any node labeled by item $i$, path $\sigma_\bX$ follows the branch corresponding to outcome $x_i$.
Note that running policy $\sigma$ under $\bX$ corresponds to traversing path $\sigma_\bX$. 
We also define  $S(\sigma,\bX)$ as the  sequence of items probed by policy $\sigma$ under realization $\bX$.

The {\em cost}  of  policy $\sigma$ depends on the realization  $\bX$, and we use $f(\sigma, \bX)\ge 0$ to denote the cost  of  policy $\sigma$ under realization $\bX$. Specifically, $f(\sigma, \bX)$ depends only on the sequence of probed items $S(\sigma,\bX)$ and their realizations. We assume that  cost is accrued only at terminal/leaf nodes: this is without loss of generality as every policy execution ends at a leaf node.

We make no further assumptions about the cost function $f$ beyond  boundedness: it does not have to be linear, submodular etc. 
Our goal is to find a policy $\sigma$ that minimizes the expected cost  $f(\sigma) := \E_{\bX}[f(\sigma, \bX)]$.  We also allow for constraints in problem  $\cP$, which must be satisfied under all realizations (i.e., with probability $1$). Let $\cC$ denote the set of feasible policies (that satisfy the constraints in $\cP$). 
{Note that any probing restrictions (e.g., constraints on which item may be probed next given prior observations) are captured by the feasible policy set $\cC$.
}
Sometimes, we work with {\em randomized} policies, where each node in the decision tree corresponds to a   probability distribution	 over  items (rather than a single item). Randomized policies are not any stronger than deterministic ones: there is always an optimal deterministic policy for this class of problems.   The stochastic problem $\cP$ is as follows (maximization problems are defined similarly).
\[    
\text{minimize} \quad f(\sigma) = \E_{\bX}\left[f(\sigma, \bX)\right] \quad  \text{subject to} \quad \sigma \in \cC.
\]

Observe that the number of nodes in policy $\sigma$ may be exponential. So, we are interested in ``efficient'' policies that can be implemented in polynomial-time for any realization $\bX$. For some stochastic problems like series systems and Pandora's box, efficient optimal policies are known. On the other hand, there are many problems like stochastic knapsack, matching and   set-cover, where optimal policies may not be efficient: in these cases, we will focus on (efficient) {\em approximately} optimal policies. 
We use $\sigma^*$ to denote an optimal  policy, and denote its expected cost by $\OPT$. 
{
To ground this abstract setup, we note that the problems described in \S\ref{sec:intro} (series testing, Pandora's box, revenue management, and posted pricing) all fit naturally into this framework: the items $[n]$ correspond to the random variables being probed, the decision tree $\sigma$ captures the adaptive probing order, and the cost function $f(\sigma, \bX)$ encodes the problem-specific objective. We refer the reader to \Cref{sec:apps} for a full list of applications.}

An $\alpha$-approximation algorithm for a 
stochastic problem $\cP$ takes as input the probability distributions $\{\calD_i\}_{i=1}^n$ (along with any objective/constraint parameters) and returns a policy that has expected cost at most $\alpha\cdot \OPT$. We  use the convention that  for minimization problems, the approximation ratio $\alpha\ge 1$, whereas for maximization problems $\alpha\le 1$.

\paragraph{The Online Semi-Bandit Setting.} 
In this setting, the distributions $\calD_1, \ldots, \calD_n$ are {\em unknown} (other parameters such as the cost function and constraints are known). 
Moreover, we have to repeatedly solve the  stochastic problem $\cP$ 
over $T$ time periods. 
We assume throughout the paper that $T \geq n$.
The distributions $\{\calD_i\}_{i=1}^n$ remain the same across all $T$ periods. 
The goal  is to simultaneously learn the distributions and
converge to a good policy over time. 
At time $t \leq T$, we use all prior observations to compute policy $\sigma^t$. The policies $\sigma^t$ may be different for  each time $t\in [T]$. 
We measure our learning algorithm in terms of \emph{expected total regret}.  For minimization problems, we define it as follows.

\begin{align}
    R(T) &= \E_{\bx^1, \ldots, \bx^T}\left[ \sum_{t=1}^T \left(f(\sigma^t, \bx^t) - f(\sigma^*, \bx^t)\right) \right] \notag \\
    &= \E_{\bx^1, \ldots, \bx^T}\left[ \sum_{t=1}^T f(\sigma^t, \bx^t)   \right] - T\cdot \OPT. \label{eq:exp-regret1}
\end{align}
Here $\bx^{t}$ represents the realization at time $t$. 
For conciseness, we use $\h^{t-1}$ to denote the {\it history} until time $t$; that is, $\h^{t-1} = (\bx^1, \ldots, \bx^{t-1})$. Note that policy $\sigma^t$ is itself random because it depends on prior observations, i.e., on the history $\h^{t-1}$. The  algorithm's cost at time $t$ is $f(\sigma^t, \bx^t)$, which depends additionally on the realizations $\bx^t$ at time $t$. We can also re-write \eqref{eq:exp-regret1} as follows.
\begin{align}
    R(T) & =  \sum_{t=1}^T \E_{\h^{t-1}} \E_{\bx^t}\left[f(\sigma^t , \bx^t) \right] - T\cdot \OPT. \label{eq:exp-regret2}
\end{align}
Note that many of the stochastic optimization problems are known to be NP-hard, even when the distributions are known. Thus, we do not expect to learn policies that approach $\OPT$. To get around this issue, we define the notion of $\alpha$-regret as follows.

\begin{equation}\label{eq:appx-exp-regret1}
    \alpha\text{-}R(T) = \sum_{t=1}^T \E_{\h^{t-1}}\left[ \E_{\bx^t}\left[f(\sigma^t , \bx^t) \right] \right] - \alpha \cdot T\cdot  \OPT .
\end{equation}

We consider the setting of \emph{semi-bandit}  feedback, where  the algorithm only observes the realizations of items that it probes. That is,  at time $t$, the algorithm only observes $\left(x^t_i : i\in S(\sigma^t, \bx^t)\right)$ where 
$S(\sigma^t, \bx^t)$ is the set of items probed by 
policy $\sigma^t$ before it terminates.

\paragraph{Censored/Binary Feedback} We also extend our results to more restricted feedback settings, where the policy may not be able to directly observe $\varX_i$ for a  probed item $i$. In particular, we consider two such extensions. In {\em censored feedback},   the policy only observes the truncated   realization $\min(x_i , \tau)$ where $\tau$ is some threshold. This is  useful  in revenue management   applications with random demand. 
In {\em binary feedback}, the policy only observes whether or not the realization $x_i$ is more than the threshold $\tau$. This is useful in  posted pricing and mechanism design.

\subsection{Main Result  and Techniques}
\label{sec:results}
For  any stochastic problem instance $\calI$ and distribution $\bld{D} = \{\calD_i\}_{i=1}^n$, let $\OPT_{\calI}(\bld{D})$  denote the optimal cost of instance $\calI$ with r.v.s having distribution $\bld{D}$. 
We first define a monotonicity 
condition for a stochastic problem $\cP$ which will be used in our main result.
\begin{definition}[Monotonicity]\label{def:monotonicity} A stochastic problem $\cP$ is {\em up-monotone} if for any instance $\calI$ and probability distributions $\bld{D}$ and $\bld{E}$  where $\bld{E}$  stochastically dominates $\bld{D}$, we have $\OPT_\calI(\bld{E})\le  \OPT_\calI(\bld{D})$. Similarly, the problem is said to be {\em down-monotone} if for any instance $\calI$ and probability distributions $\bld{D}$ and $\bld{E}$  where $\bld{E}$  stochastically dominates $\bld{D}$, we have $\OPT_\calI(\bld{E})\ge  \OPT_\calI(\bld{D})$. 
\end{definition}

\begin{theorem}\label{thm:main}
Suppose that a stochastic problem $\cP$ has an  $\alpha$-approximation algorithm, and it is either up-monotone or down-monotone. Then, there is a  polynomial time semi-bandit  learning algorithm  for $\cP$ (with unknown distributions) that has $\alpha$-regret $\Ot( n (\sqrt{k T}+k)f_{\max})$. Here, $n$ is the number of items, $k$ is the maximum support size of any distribution, $f_{\max}$ is the maximal value of the objective function $f$ (over all realizations) and $T$ is the number of time periods. 
\end{theorem}
Above, the $\Ot$ notation is used to suppress (poly) logarithmic factors.  
If we do not care about polynomial running time then we can just set $\alpha=1$, by computing the optimal policy for $\cP$ via a stochastic dynamic program, and obtain $1$-regret $\Ot( n (\sqrt{k T}+k)f_{\max})$ for any monotone stochastic problem. {The only general result of this type known prior to our work is a $T^{2/3}$ regret bound for these problems using  sample complexity bounds of~\citet{GuoHT+21}; specifically, this yields $\Ot\left((nkT)^{2/3} \right)$ regret under semi-bandit feedback (see Appendix~\ref{app:sampling-algs}). While our result has better dependence in $T$, the  bound from \citet{GuoHT+21} has better dependence in $n$.}

The regret bound in Theorem~\ref{thm:main} is nearly optimal in the following sense. There is an $\Omega(\sqrt{T})$ lower bound on the regret even in a very special case: see Appendix~\ref{app:lower-bound}. Furthermore, if we move beyond our setting of independent-identically-distributed (i.i.d.) distributions across periods, to the ``adversarial'' setting with different distributions for each period, then there is a linear $\Omega(T)$ lower-bound on regret for prophet inequality~\citep{GatmiryKS+22}, which shows that sublinear regret is not possible for our class of problems in the adversarial case.

\medskip
\noindent
{\bf Technical Overview.} 

At a high-level, our algorithm is based on the \emph{principle of optimism in the face of uncertainty} which is well-studied in the multi-armed bandits literature: see e.g., the  UCB algorithm~\citep{Auer+02}. Given observations from previous rounds, our algorithm first constructs a modified empirical distribution $\bld{E} = \{\calE_i\}_{i=1}^n$ which {\em stochastically dominates} the true distribution $\bld{D}$. It then executes   policy $\sigma$ given by the offline (approximation) algorithm under distribution $\bld{E}$.

We now discuss how to bound the regret of the policy $\sigma$ under the true distribution $\bld{D}$. 
To keep things simple here, we  assume that $\alpha=1$.  We also assume that the problem is up-monotone and has a minimization objective (our analysis for  down-monotonicity and/or maximization objective is identical).
The difficulty in our setting is that the decision to probe/observe an item is {\em random}: it depends on the choice of policy $\sigma$ and the   underlying distribution $\bD$ of  items.  In contrast, the classical multi-armed bandits setting allows direct control to the algorithm on which item to probe: so the  UCB  algorithm can control  the rate of \emph{exploration} of individual items and  bound the regret of individual items in terms of this rate. How do we bound the rate of  exploration  of different items when we do not have direct control on which items are probed? Moreover, how do we bound the regret contribution of individual items in terms of their rate of exploration?

The key insight in our analysis is the following ``stability'' bound,  which  answers the above questions. Given product distributions $\bld{D}$ and $\bld{E}$, where $\bld{E}$ stochastically dominates $\bld{D}$, we show\footnote{For   policy $\sigma$ and distribution $\bD$,   $f(\sigma | \bld{D})$  is the expected objective of $\sigma$ when the r.v.s have distribution $\bD$.}
\begin{align*}
f(\sigma | \bld{D}) - f(\sigma^* | \bld{D}) \,\,&\leq  \,\,f(\sigma | \bld{D}) - f(\sigma | \bld{E})  \,\,\leq\,\,    f_{\max} \sum_{i=1}^n {Q}_{i}(\sigma) \cdot \epsilon_i,
\end{align*}
where  $\sigma^*$ (resp. $\sigma$) is an  optimal policy under $\bld{D}$ (resp.  $\bld{E}$),   $Q_i(\sigma)$ is the probability that policy $\sigma$ probes item $i$, and $\epsilon_i$ is the {\em total-variation}
distance between $\cE_i$ and $\D_i$. The first inequality above follows from the monotonicity property. The second inequality requires analyzing the decision tree of $\sigma$ carefully  and is a key technical contribution of this paper (see \S\ref{sec:stable}).

The above stability bound can be interpreted as follows: the contribution of item $i$ in the total regret is the probability that it is probed 
times the total variation error in estimating $\D_i$. The dependence on the probability of probing is  crucial here. This allows us  to analyze the regret using a \emph{charging argument} where we charge item $i$ for regret in a \emph{pay-per-use} manner.
The paths in the execution tree of the online algorithm that probe $i$    will pay for the error in estimation of $\D_i$, but then $\epsilon_i$
will also reduce along these paths.
Moreover, we can bound the overall regret irrespective of the policy $\sigma$ that is used at periods.
The error $\epsilon_i \approx \sqrt{1/m}$ when $i$ has been probed $m$ times in the past,
which gives a total regret contribution of $\sum_{m=1}^T \sqrt{1/m} = O(\sqrt{T})$ 
for any   item.
Finally, we  lose an extra logarithmic factor  to ensure that our empirical distribution $\bE$ stochastically dominates $\bD$ with high probability. We contrast our algorithm (and stability bound) with another algorithm that balances the \emph{explore-exploit} trade-off in a different manner: the \emph{explore-then-commit} algorithm.

Using the sample complexity bound  of \cite{GuoHT+21}, it is easy to show that this algorithm achieves $T^{2/3}$ regret; see Appendix~\ref{app:sampling-algs} for details.
However, each item is explored 
an equal number of times regardless of its quality  and we might end up exploring some ``low quality'' items too many times.
Our stability bound highlights a crucial difference
between our algorithm and explore-then-commit: 
we only explore items that are actually probed by  the (current) optimal policy. If there are low-quality  items 
that are rarely used in the optimal policy, they 
will  be explored  infrequently in our algorithm: so their regret contribution is low. This  explains why our approach  achieves a near-optimal $\sqrt{T\log T}$ regret bound compared to the $T^{2/3}$ bound of the sampling-based approach. 

\subsection{Additional Results and Applications}

We can also extend our main result to stochastic problems with continuous (or mixed) distributions. Here, we need to make some additional (mild) assumptions on the objective function. Informally, we require that for any fixed policy and  leaf-node,  the  function value does not vary significantly with the variables $\bx$. This assumption is satisfied whenever the function value at a leaf-node is linear or monotone, which suffices for  applications such as Prophet inequality and Pandora's box. In this case, we are able to obtain $\alpha$-regret bounds similar to Theorem~\ref{thm:main}. The algorithm for the continuous setting  is   identical to the discrete case. The main new idea in the analysis is the use of a different distance-measure between distributions (the Kolmogorov–Smirnov distance) and modifying the proof of the ``stability'' bound appropriately. {We also assume that the $\alpha$-approximation algorithm returns a $k$-threshold policy (defined in \Cref{sec:continuous}), which is a natural generalization of the discrete setting where any policy has at most $k$ branches per node; see the remark at the start of \Cref{sec:continuous} for further discussion.}

Another extension of our main result is to censored and binary feedback. This extension requires   discrete distributions and a more restrictive assumption on the objective function. Roughly speaking, we assume that the function value can always be determined based on the observed feedback (either censored or binary). This assumption is satisfied for the single-resource revenue management and sequential posted pricing problems. Again, we obtain regret bounds as in   Theorem~\ref{thm:main}. While the high-level algorithm is the same, we need to design new sampling algorithms to  construct the  stochastically dominating distribution $\bld{E}$. This requires more work   because the feedback from a probed item $i$ does not directly provide us with a  sample of $\varX_i$. All the details of this extension are in  \Cref{sec:binary/censored}.  We note that the discrete-distribution assumption is necessary here in order to obtain $\sqrt{T}$ regret: see Appendix~\ref{app:lower-bound} for a regret lower bound of $\Omega(T^{2/3})$ under binary-feedback and continuous distributions.  

We demonstrate the generality of our framework by obtaining polynomial time online learning algorithms for a number of stochastic
optimization problems arising in domains of optimal stopping, revenue management, submodular optimization, knapsack and matching.  For many problems, we get better regret bounds using  our framework than what was known previously using problem-specific methods. 
We summarize our results in \Cref{table:apps} and refer to \Cref{sec:apps} for further details.

\subsection{Related Work}

As mentioned earlier, stochastic optimization problems (under known distributions) have been studied extensively. Several papers have extended the classic 
prophet inequality~\citep{samuel1984comparison} and Pandora's box~\citep{weitzman1979} results to more complex settings, e.g., \citet{kleinberg2012matroid,RubinsteinS17,KleinbergWW16,singla2018price,FeldmanSZ21,BoodaghiansFLL23}. 
Moreover, good approximation bounds have been achieved for stochastic versions of various combinatorial optimization problems~\citep{DGV08,GolovinK-arxiv,Ma18,GuptaKNR15,im2016minimum,goemans2006stochastic,GN13,gupta2017adaptivity,deshpande2016approximation,ChenIKMR09,BGLMNR12}.

There has also been extensive work in  online learning, where one  considers unknown distributions (in the stochastic learning setting), see e.g., books~\citet{cesa2006prediction}, \citet{lattimore2020bandit} and \citet{hazan2022introduction}. All our stochastic problems can be modeled as Markov decision processes (MDPs) with unknown transition probabilities, and there have been some  works on achieving sublinear regret in this setting  \citep{AzarOM17,JinABJ18,JinJLSY20}. 
However, the regret bounds from these works have a   polynomial dependence on the ``state space'' of the MDP, which in our setting, is exponential in $n$ (the number of r.v.s). It is also known that any online algorithm for arbitrary MDPs incurs an $\Omega(\sqrt{S\cdot T})$ regret, where $S=\exp(n)$ is the size of the state space~\citep{JakschOA10,OsbandR16b}. In contrast, we obtain
regret bounds that  depend polynomially on $n$.

Although both stochastic optimization and online learning have been subjects of comprehensive research, 
there has been 
relatively limited work at the intersection of these two domains.  \cite{GatmiryKS+22} obtained online learning algorithms for  Pandora's box and prophet inequality under  {\em bandit} feedback, which is  even more restrictive than  semi-bandit feedback. Here, the algorithm only observes the realized objective value of its policy. \cite{GatmiryKS+22} gave algorithms with $1$-regret $O\left(n^3\sqrt{T}\log(T)\right)$ and $O\left(n^{4.5}\sqrt{T}\log(T)\right)$ for prophet inequality and Pandora's box respectively.

{ 
While our semi-bandit feedback model is more relaxed than bandit feedback, we think that it already captures the core issue of partial observations in learning policies. Moreover, under semi-bandit feedback we are able to obtain a general framework that applies to several stochastic optimization problems. Regarding bandit feedback for general stochastic problems, recent work of \citet{TajdiniJJ24} and \citet{zhuo2026learningmarkovdecisionprocesses} established regret lower bounds that are exponentially dependent on the problem size for stochastic submodular maximization and  MDPs respectively, suggesting that semi-bandit feedback is necessary for polynomial regret bounds in this generality.

}

There have also been several works studying 
sample complexity
bounds for  stochastic optimization problems.
The goal in these works is to understand how much data is necessary and sufficient to guarantee near-optimal algorithms.
Such results have been obtained for single-parameter revenue maximization~\citep{cole2014sample, roughgarden2016ironing}
and  prophet inequality~\citep{correa2019prophet}.
{ The recent work of~\citet{GuoHT+21} gives
optimal sample complexity bounds for stochastic optimization problems that exhibit  \emph{strong monotonicity}, which is a stronger condition than our monotonicity definition.
Informally, strong monotonicity requires that the optimal policy under one distribution also performs near-optimally under any stochastically dominating distribution, which is a strictly stronger condition than our monotonicity definition.
}

Another relevant line of work  is on combinatorial multi-armed bandits (CMAB), which also involves   semi-bandit feedback~\citep{KvetonWAEE14,KvetonWAS15,ChenWYW16,ChenHLLLL16}. Here, there are $n$ base arms that produce stochastic outcomes drawn from an unknown fixed distribution. There is also a collection ${\cal F}\sse 2^{[n]}$ of allowed ``super arms''. In each period $t$,   the algorithm  selects a super-arm $S^t\in {\cal F}$ and observes the realizations of all arms $i\in S^t$. 

The result closest to ours is by \citet{ChenHLLLL16}, which considers a  class of non-linear objectives satisfying a ``monotone'' condition and obtains a UCB-type algorithm achieving $\alpha$-regret of  $O(n\sqrt{T\log T})$ where $\alpha$ is the approximation ratio for the offline problem.   Our setting is  more general because we need to select   {\em policies} (not static subsets). When we select a policy, we do not know  which arms will  actually be  observed. So, we only have an indirect control on what arms will be observed in each round.
While our algorithm can be seen as a natural extension of \cite{ChenHLLLL16}, our analysis requires new ideas: in particular in proving the  stability Lemma~\ref{lem:stable} and its use in bounding overall regret.

There are also some   online-to-offline reductions that work in the  adversarial setting (which is harder than our  stochastic setting). In particular, \cite{KalaiV05} considered combinatorial optimization with {\em linear} objective functions, and obtained $1$-regret of 
$O(n\sqrt{T})$ under full-feedback (assuming an exact offline algorithm). Then, \cite{KakadeKL09} extended this result to linear problems with only an $\alpha$-approximation algorithm, and obtained 
$\alpha$-regret  of $O(n\sqrt{T})$ under full-feedback and $O(nT^{2/3})$  under bandit-feedback. Recent work of  \cite{niazadeh2021online,NieNZAQ23} considers certain  combinatorial optimization problems with non-linear objectives, assuming an  $\alpha$-approximation  via a greedy-type  algorithm. For such problems, \cite{niazadeh2021online} obtained $\sqrt{T}$ regret under   full-feedback and $T^{2/3}$ regret under bandit-feedback. While these  results work in the harder  adversarial online  setting, our result handles a much wider class of problems: we learn policies (rather than just  subsets) and handle  arbitrary objectives. As noted before,  there is an $\Omega(T)$ regret lower-bound for some of our applications in the adversarial setting.

\subsection{Preliminaries}\label{sec:prelim}
\def\sd{\preceq_{\mathsf{SD}}}
We present some preliminaries that are required in our algorithm and proofs.

\begin{definition}[Stochastic Dominance]\label{def:stoch-dom}
A probability distribution $\calE$ (over $\R$)  stochastically dominates another distribution $\calD$ if, for all $a\in \R$,  we have: $\Pr_{\varX \sim \calE}  (\varX  \geq a) \geq \Pr_{\varY \sim \calD}(\varY \geq a)$. 
\end{definition}

We use $\calD \sd \calE$ to denote that distribution  $\calE$  stochastically dominates $\calD$. 
For product distributions  $\bld{D} = \{\calD_i\}_{i=1}^n$ and $\bld{E}= \{\calE_i\}_{i=1}^n$, we say that $\bld{E}$  stochastically dominates $\bld{D}$ if $\calE_i$ stochastically dominates $\calD_i$ for all $i\in [n]$; we also denote this by $\bld{D}\sd  \bld{E}$.

Let $\TVD(\calD, \calE)$ denote the \emph{total variation distance} between discrete distributions $\calD$ and $\calE$.
The total variation distance is half of the $\ell_1$ distance between the two distributions, i.e., $\TVD(\calD, \calE) = \frac{1}{2} \cdot \vert\vert \calD - \calE \vert\vert_1$. 
The following standard result (see, for example, Lemma~B.8 in~\citet{ghosal2017fundamentals}) bounds the total variation distance between product distributions.

\begin{lemma}\label{lem:TV-product-bound}
Given product distributions $\bld{D}= \{\calD_i\}_{i=1}^n$ and $\bld{E}= \{\calE_i\}_{i=1}^n$ over $n$ r.v.s, we have 
$$ \TVD(\bld{D}, \bld{E}) \leq \sum_{i \in [n]}\TVD(\calD_i, \calE_i).$$
\end{lemma}

Consider independent r.v.s $\varX_1, \ldots, \varX_n$ where $\var{T}_i$ denotes the domain of $\varX_i$. 
Let $h$ be a function from $\bld{T} = \var{T}_1 \times \cdots \times \var{T}_n$  to $[0, U]$; that is, $h$ is a function on the \emph{outcomes} of the random variables that  is bounded by $U$.
Thus, for any $\mathbf{x} \in \bld{T}$, $h(\mathbf{x})$ denotes the value of $h$ on the outcome $\mathbf{x} = (x_1, \ldots x_n)$. 
Given a product distribution $\bld{P}$ over the r.v.s, define $h(\bld{P}) := \E_{\mathbf{x} \sim \bld{P}}\left[ h(\mathbf{x}) \right]$. The following useful fact bounds the difference in function value at two different distributions.

\begin{lemma}\label{lem:function-val-tv}
Given  discrete    
distributions $\bld{D}$ and $\bld{E}$ over $n$ random variables $\varX_1, \ldots, \varX_n$,   and a $[0,U]$ bounded function $h$ on the outcomes of these r.v.s (as above),  
we have 
\begin{equation}\notag
    |h(\bld{D}) - h(\bld{E})| \,\, \leq \,\, U \cdot \TVD(\bld{D}, \bld{E}).      
\end{equation}
\end{lemma}

{
The following inequality bounds the deviation of the empirical CDF from the true CDF uniformly over all thresholds, and will be used to bound the total variation distance between our empirical distribution and the true distribution.
}
\begin{theorem}[DKW Inequality from \citet{Massart1176990746}]\label{thm:dkw}

   Let $X_1, \ldots, X_N$ be i.i.d. samples with cumulative distribution function $F(\cdot)$. Let $\overline{F}(x):=\frac{1}{N} \sum_{i \in[N]} \mathbf{1}_{X_i \leq x}$ be the  empirical CDF. Then, for every $\varepsilon>0$, we have

$$
\Pr\left[\sup_x \left|\overline{F}(x)-F(x) \right| \, > \, \varepsilon\right] \,\, \leq \,\, 2 \exp \left(-2 N \varepsilon^2\right)
$$
\end{theorem}

\section{The Online Framework}
In this section we present our main algorithm. We will assume access to an $\alpha$-approximation algorithm $\ALG$ for the 
stochastic problem $\cP$.
For concreteness, we assume that $\cP$ is a minimization problem.
Thus, given an instance of  $\cP$ (with a probability distribution for
each item), $\ALG$ finds a policy of expected cost at most $\alpha$ times the optimum. 

We also assume that $\cP$ is up-monotone; see Definition~\ref{def:monotonicity}. 
The online framework for down-monotone problems and   maximization problems is almost identical: the changes are explained in \Cref{subsec:down-mon} and \Cref{sec:max-probs}. Our online algorithm is based on the principle of optimism in the face of uncertainty, and is very simple to describe. At each time step $t$, we construct an ``optimistic''  empirical distribution $\bld{E}^t$ that  stochastically dominates 
the true (unknown) distribution $\bld{D}$.
Then, we run the offline algorithm $\ALG$ on this empirical distribution $\bld{E}^t$ to obtain policy $\sigma^t$, which is  the online policy at time $t$.

\paragraph{Computing stochastically dominating distributions.} Given i.i.d. samples of any random variable, we show that it is possible to compute a stochastically dominating empirical distribution that has a small total-variation distance from the true distribution. 
\begin{theorem}
\label{thm:emp-stoch-dom}

There is an algorithm $\ESD$ that, given $m$ i.i.d.\ samples from a distribution $\calD$ with finite support-size $k$, and parameter $\delta > 0$,  computes a distribution $\calE$ that satisfies the following properties with probability at least $1-\delta$.
\begin{itemize}
\item $\calE$ stochastically dominates $\calD$, and 
\item the total-variation distance $\TVD(\calE, \calD) < \sqrt{\frac{5k\log(2k/\delta)}{m}} +\frac{4k\log(2k/\delta)}{m}$. 
\end{itemize}  
\end{theorem}

The algorithm $\ESD$ and proof of this result are deferred to \Cref{subsec:emp-stoc-dom}. The main idea is to ``shift'' some probability mass in the usual empirical distribution from low to high values. In a previous version of this paper, we obtained a slightly simpler proof with a (worse) $\TVD$ bound of $k\sqrt{\frac{\log (2/\delta)}{2m}}$.

\begin{algorithm}
\caption{\textsc{Online-to-Offline} framework}
\label{alg:greedy-ucb}
\begin{algorithmic}[1]

\For{$t = 1, \ldots T$}
\State for each item $i\in [n]$, obtain distribution ${\calE}^t_i$ by running $\ESD$ with $\delta =\frac{2}{(nT)^3}$ on all samples of r.v. $\varX_i$  observed so far.
\State obtain policy $\sigma^t$ by running the offline algorithm $\ALG$ on distribution $\bld{E}^t =\left\{{\calE}^t_i\right\}_{i=1}^n$.
\State run  policy $\sigma^t$ on realization $\bx^t$ and observe  the probed items $S(\sigma^t, \bx^t)$.  
\EndFor
\end{algorithmic}
\end{algorithm}

Algorithm~\ref{alg:greedy-ucb} describes our online framework, which uses  algorithm $\ESD$.
We now analyze Algorithm~\ref{alg:greedy-ucb}, and prove Theorem~\ref{thm:main}. 
The basic idea in the analysis is that 
as we get more and more samples with increasing $t$, the total 
variation distance between $\bld{E}^t$ and $\bld{D}$
reduces,
and the policy $\ALG(\bld{E}^t)$ will become closer and 
closer to the optimal policy.

\begin{proof}{\it of \Cref{thm:main}}
We will assume, without loss of generality, that we have observed at least one sample from each random variable. This can be ensured by adding $n$ special time periods and choosing for each time $i\in [n]$, a policy $\sigma^i$ that first probes item $i$; note that this  contributes at most $n\cdot f_{max}$ to the total regret.

Let $N_j^t$ be the number of times r.v. $\varX_j$ has been sampled before time $t$. By our assumption, $N_j^1 = 1$ for all $j \in [n]$. Note that $N_j^t$ is a random variable and depends on the history $\h^{t-1}=\left( \bx^1,\cdots \bx^{t-1}\right)$. When needed, we will use $N_j^t(\h^{t-1})$ to make this dependence explicit.

The following result follows directly from   Theorem~\ref{thm:emp-stoch-dom}.

\begin{lemma}\label{lem:good-event}
With probability at least $1-\frac{1}{ nT}$, for each $j\in[n]$ and $t\in [T]$,  we have that 
$\calE_j^t$  stochastically dominates $\calD_j$  and 
$$\TVD\left( \calE^t_j,  \calD_j\right) < 15 \sqrt{\frac{k\log(2knT)}{N^t_j}} +15\frac{k\log(2knT)}{N^t_j}  . $$ 
\end{lemma}

\begin{proof}
For any $j \in [n]$ and $t\in [T]$, let $B_j^t$ denote the event that the stated condition {\em does not} hold for $j$ and $t$. Using Theorem~\ref{thm:emp-stoch-dom} with $\delta = 2/(nT)^3$, we obtain for any $j\in [n]$ and $t\in [T]$, 
\[
\pr\left(B_j^t \ \wedge \ N_j^t = m\right) \le  \delta,\qquad \forall m\ge 1.
\]
So, for any $j$ and $t$,  we have by union bound,
\[\pr(B_j^t)  \leq \sum_{m=1}^T \pr\left(B_j^t \ \wedge \ N_j^t = m\right) \leq \sum_{m=1}^T \delta = \delta T.\]
Again, by union bound, 
\[\pr\left(\vee_{j=1}^n \vee_{t=1}^T  B_j^t\right)  \leq \sum_{j=1}^n \sum_{t=1}^T  \pr(B_j^t)  \le nT^2\delta\le \frac{1}{nT}.\]
The last inequality uses $n\ge 2$. This completes the proof. 
\end{proof}
Let $G$   denote the {\it ``good'' event}   corresponding to the condition in Lemma~\ref{lem:good-event} holding for all $j$ and $t$. First, we complete the proof assuming that $G$ holds (this assumption is removed later).

We now state a crucial ``stability'' property for the stochastic problem $\cP$. Recall that $Q_i(\sigma)$ denotes the probability that policy $\sigma$ probes item $i$ under distribution $\bD$.
\begin{lemma}[Stability lemma]\label{lem:stable}
Consider a stochastic problem that is up-monotone. Suppose that $\bE=\{\calE_i\}_{i=1}^n$ and $\bD=\{\calD_i\}_{i=1}^n$ are product distributions  
such that $\bD\sd \bE$ and $\TVD(\calE_i, \calD_i) \leq \epsilon_i$ for each $i \in [n]$. 
If $\sigma$ is the policy returned by $\ALG(\bE)$  
and $\sigma^*$ is an optimal  policy under $\bD$, then:
\begin{align}
f(\sigma) - \alpha\cdot  f(\sigma^*) \,\, &= \,\, \E_{\bx \sim \bD}\left[f(\sigma, \bx) - \alpha\cdot f(\sigma^*, \bx)\right] \leq \,\, f_{\max} \sum_{i=1}^n {Q}_{i}(\sigma) \cdot \epsilon_i 
    \,, \label{eq:per-time-bound}
\end{align} 
where ${Q}_{i}(\sigma)$ denotes the probability that item $i$ is probed by policy $\sigma$ under distribution $\bld{D}$.
\end{lemma}

This lemma relies on the monotonicity assumption, and is proved in \Cref{sec:stable}.
We now  complete the proof of Theorem~\ref{thm:main} using Lemma~\ref{lem:stable}.

\subsection{Completing the Proof of Theorem~\ref{thm:main}} \label{sec:path-regret}

\paragraph{Bounding the regret as a sum over time $t$.} To bound the overall regret,  it suffices to bound the regret at each time $t\in [T]$, defined as:

\begin{align*} 
R^t(\h^{t-1}) \,\, &:= \,\, \E_{\bx^t}\left[f(\sigma^t , \bx^t) - \alpha\cdot f(\sigma^*, \bx^t)\right] = \,\, f(\sigma^t) - \alpha\cdot f(\sigma^*) ,
\end{align*}
where $\h^{t-1}$ is the history at time $t$, policy $\sigma^t=\ALG(\bE^t)$ and policy $\sigma^*$ is 
the optimal policy for the stochastic problem instance (under the true distribution $\bD$). By~\eqref{eq:appx-exp-regret1}, the overall regret is $\alpha\text{-}R(T) =  \sum_{t=1}^T \E_{\h^{t-1}}\left[ R^t(\h^{t-1})\right]$.
For each time $t\in [T]$ and history $\h^{t-1}$, we apply  Lemma~\ref{lem:stable} with distributions $\bE^t(\h^{t-1})$ and $\bD$,  and parameters $\epsilon_i^t (\h^{t-1}) =  15 \sqrt{\frac{k\log(2knT)}{N_{i}^t(\h^{t-1})}} +15\frac{k\log(2knT)}{N_{i}^t(\h^{t-1})}  $ for all $i\in[n]$. To keep notation simple, we use the function $g(m):=15 \sqrt{\frac{k\log(2knT)}{m}} +15\frac{k\log(2knT)}{m}  $.  So, $\epsilon_i^t =g(N^t_i)$.   
Note that  $\TVD(\calE^t_i, \calD_i) \leq \epsilon_i^t(\h^{t-1})$ and $\calD_i\sd \calE_i$  for each $i \in [n]$ because we assumed the good event $G$. Hence, Lemma~\ref{lem:stable} implies that:
\begin{equation}
\label{eq:regret-t-stable}
R^t(\h^{t-1})  \le  f_{\max} \sum_{i=1}^n {Q}_{i}(\sigma^t) \cdot \epsilon_i^t(\h^{t-1}) .
\end{equation}

We note that $R^t$, $\sigma^t$, $N^t$ and $\epsilon^t$ all depend on the history $\h^{t-1}$; to keep notation simple,  we drop the explicit dependence on $\h^{t-1}$. Using the regret definition and \eqref{eq:regret-t-stable},

\begin{align*}
\alpha\text{-}&R(T) \,\, \le \,\,  f_{\max} \sum_{t=1}^T \E_{\h^{t-1}}\left[  \sum_{i=1}^n {Q}_{i}(\sigma^t) \cdot \epsilon_i^t\right] = \,\, f_{\max} \sum_{i=1}^n \sum_{t=1}^T \E_{\h^{t-1}}\left[ {Q}_{i}(\sigma^t) \cdot \epsilon_i^t\right]
\end{align*}

It now suffices to show that for any $i$, 
\begin{align}
& \sum_{t=1}^T   \E_{\h^{t-1}}\left[  {Q}_{i}(\sigma^t) \cdot  \epsilon_i^t  \right] \le O\left(\sqrt{k\log (nkT) \, T} +   k \log^2(nkT)\right). \label{eq:path-regret-sum}
\end{align}
Indeed, combining \eqref{eq:path-regret-sum} with the above bound on regret, we get $\alpha\text{-}R(T) =\Ot( n (\sqrt{k T}+k)f_{\max})$, which  completes the proof of Theorem~\ref{thm:main} assuming event $G$.

\paragraph{Proving \eqref{eq:path-regret-sum} as a sum over decision paths.} We refer to the full history $\h^T=\left(\bx^1,\cdots , \bx^T\right)$ of the algorithm as its {\em decision path}. The main idea in this proof is to  view the left-hand-side in \eqref{eq:path-regret-sum}  as a sum over decision paths rather than a sum over time. To this end, define $Z_i(\h^T)$ as:
\begin{equation}
\sum_{t=1}^T  \1\left[\text{item $i$ probed by }\sigma^t(\h^{t-1})\right]\cdot g(N^t_i(\h^{t-1})).
\end{equation}

Above, $\1$ is the indicator function.
By linearity of expectation,  $\E_{\h^T} \left[ Z_i(\h^T) \right]$ equals
{\small \begin{align}
&\sum_{t=1}^T \E_{\h^T} \left[ \1\left[i \text{  probed by  }\sigma^t(\h^{t-1})\right]\cdot g(N^t_i(\h^{t-1})\right] \notag \\
& =\,\,\sum_{t=1}^T \E_{\h^{t-1}, \bx^t} \left[ \1\left[i \text{  probed by  }\sigma^t(\h^{t-1})\right] \cdot g(N^t_i(\h^{t-1}) \right] \notag \\
& =\sum_{t=1}^T \E_{\h^{t-1}} \left[ g(N^t_i) \cdot \pr_{\bx^t} [i \text{  probed by  }\sigma^t]\right] \notag\\ 
&  = \,\, \sum_{t=1}^T \E_{\h^{t-1}} \left[  Q_{i}(\sigma^t)\cdot  g(N^t_i) \right],\label{eq:Zi}
\end{align}}
where the first equality uses the fact that event \{$i$   probed by  $\sigma^t$\}  only depends on $\h^t = (\h^{t-1}, \bx^t)$, the 
second  equality uses the fact that $N^t_i$ only depends on $\h^{t-1}$ and that $\bx^t$ is independent of $\h^{t-1}$, and the last equality is by the definition of $Q_{i}(\sigma^t)$ and the fact that $\bx^t \sim \bD$.

Therefore, it suffices to upper bound $\E_{\h^T}[Z_i(\h^T)]$, which we can prove in a per-realization manner.
Indeed, $Z_i(\h^T)$ equals:
\begin{align*}
  & \sum_{t=1}^T  \1\left[i\text{ probed by }\sigma^t(\h^{t-1})\right] g(N^t_i)\,\, \le \,\,\sum_{t=1}^T g(t) \\
&\le \,\, 15\sqrt{k\log(knT)} \sum_{t=1}^T \frac1{\sqrt{t}} + 15k\log(knT) \sum_{t=1}^T \frac1{ t}  \\
& = \,\, O\left(\sqrt{k\log(knT) \, T} + k\, \log^2(nkT)\right). 
\end{align*}
The first inequality uses the fact that $N^t_i$ equals the number of probes of item $i$ in the first $t-1$ time steps: so $N^{t+1}_i=N^t_i+1$ whenever item $i$
is probed by $\sigma^t$. This completes the proof of  \eqref{eq:path-regret-sum}.

{
\paragraph{Removing the ``good'' event assumption.} Recall from Lemma~\ref{lem:good-event} that $G$ is the event that $\calE_j^t$ stochastically dominates $\calD_j$ and $\TVD(\calE_j^t, \calD_j) \le g(N_j^t)$ for all $j\in[n]$ and $t\in[T]$. In the analysis above, we assumed that event $G$ holds in \eqref{eq:regret-t-stable}.}
We now modify \eqref{eq:regret-t-stable} as follows (which holds irrespective of $G$). 
\[
R^t(\h^{t-1})  \le    f_{\max} \sum_{i=1}^n {Q}_{i}(\sigma^t) \cdot \epsilon_i^t(\h^{t-1}) + f_{max}\cdot \1[\overline{G}].
\]

We used the fact that the maximum 1-step regret is $f_{max}$. Combined with  \eqref{eq:path-regret-sum}, which handles the first term, we have $\alpha\text{-}R(T) =  \sum_{t=1}^T \E[R^t(\h^{t-1}) ]$ is at most
\(\Ot( n (\sqrt{k T}+k) ) f_{\max} \,+\, T f_{\max} \cdot \pr[\overline{G}] \,\, \le \,\, \Ot( n (\sqrt{k T}+k) )    f_{\max}.
\) The last inequality uses Lemma~\ref{lem:good-event}. 
This completes the proof of Theorem~\ref{thm:main}.
\hfill 
\end{proof}

\subsection{Proving the Stability Lemma}\label{sec:stable}
We now prove Lemma~\ref{lem:stable}.
Recall that there are two product distributions $\bE=\{\calE_i\}_{i=1}^n$ and $\bD=\{\calD_i\}_{i=1}^n$ with $\bD\sd \bE$ and  $\TVD(\calE_i, \calD_i) \leq \epsilon_i$ for each $i \in [n]$. 
Let $\sigma=\ALG(\bE)$ 
be the policy returned by the (offline) $\alpha$-approximation algorithm,
and $\sigma^*$ be
an optimal policy under distribution $\bD$.

We want to upper bound
$f(\sigma) - \alpha\cdot  f(\sigma^*)$. 
Let $\cC$ denote the set of all feasible policies to the given instance. Then,
\(
\OPT(\bD) = \min_{\tau \in \cC} \E_{\bx\sim \bD}[f(\tau,\bx)]  \text{ and }  \OPT(\bE) = \min_{\tau \in \cC} \E_{\bx\sim \bE}[f(\tau,\bx)].
\)
We  use $f(\tau|\bld{U}) := \E_{\bx\sim \bld{U}}[f(\tau,\bx)]$ to denote the expected cost of any policy $\tau$ when the r.v.s have product distribution $\bld{U}$. Note that $f(\tau) = f(\tau|\bD)$ where $\bD$ is the true distribution. We now have:
\begin{align*}
&f(\sigma) - \alpha\cdot  f(\sigma^*) \,\, = \,\, f(\sigma| \bD) - \alpha\cdot  f(\sigma^*| \bD) \notag \\ &= \,\, f(\sigma| \bD) - f(\sigma| \bE) +f(\sigma| \bE) - \alpha\cdot  f(\sigma^*| \bD) \notag \\
& \le \,\, f(\sigma| \bD) - f(\sigma| \bE) + \alpha\cdot \OPT(\bE) - \alpha\cdot  f(\sigma^* | \bD) \\ 
&= \,\, f(\sigma | \bD) - f(\sigma | \bE) + \alpha\cdot \OPT(\bE) - \alpha\cdot  \OPT(\bD) \\
&\le\,\, f(\sigma | \bD) - f(\sigma | \bE).
\end{align*}
Above, the first inequality uses the fact that $\sigma$ is an $\alpha$-approximate policy to the instance with distribution $\bE$, the penultimate equality uses the fact that $\sigma^*$ is an optimal policy 
for the instance with distribution $\bD$, and the final inequality is by up-monotonicity (Definition~\ref{def:monotonicity}) and $\bD\sd \bE$. 
{
We note that stochastic dominance and monotonicity are only used to establish $\OPT(\bE) \le \OPT(\bD)$ in this final inequality; the rest of the proof holds more generally whenever one can construct a distribution $\bE$ satisfying this condition.
}

\def\cT{{\cal T}}
\def\cN{{\cal N}}

\paragraph{Equivalent view of adaptive policy $\sigma$.}
Let $N$ denote the number of nodes in decision tree  $\sigma$ (this  may be exponential, but it is  only used in the analysis). Note that there is a {\em partial ordering} of the nodes of $\sigma$ based on   ancestor-descendent relationships in the tree. 
We index the nodes in $\sigma$ from $1$ to $N$ according to this partial order, so that $u<v$ for any node $v$ that is a child of  node $u$. Recall   that each node $v\in \sigma$ is labeled by one of the $n$  items. Note that the same item $i\in [n]$ may label multiple nodes of $\sigma$; however, in any policy execution (i.e., root-leaf path in $\sigma$) we will encounter at most one node labeled by item $i$.
Hence, we can equivalently view policy $\sigma$ as having  an item $\varX_v$ with {\em independent} distribution $\D_v$ at each node $v\in \sigma$. (This involves  making several independent copies of each   item, which does not affect the policy execution as at most one copy of each item is seen on any root-leaf path.)

\paragraph{Bounding $| f({\sigma} | \bld{D}) - f({\sigma}| \bld{E})|$.} Based on the above view of $\sigma$, we use $\calD_v$ (resp.  $\calE_v$) to denote the independent distribution at each node $v\in \sigma$ under the joint distribution $\bD$ (resp. $\bE$). Using the above indexing of the nodes in $\sigma$, we   define the following {\em hybrid} 
product distributions: for each $v\in [N]$ let $\bld{H}^v = \calD_1 \times \cdots \times \calD_{v} \times \calE_{v+1} \times \cdots \times \calE_N$.  
Observe that $f(\sigma \mid \bld{H}^N) = f(\sigma \mid \bld{D})$ and $f(\sigma \mid \bld{H}^0) = f(\sigma \mid \bld{E})$.
Using a telescoping sum, we can write:
\[
f(\sigma \mid \bD) - f(\sigma \mid \bE) = \sum_{v=1}^N f(\sigma \mid \bld{H}^v) - f(\sigma \mid \bld{H}^{v-1}). 
\]
Crucially, we now show that for every node $v\in [N]$,  
\begin{equation}\label{eq:stab-lemma-key}
    \mid f(\sigma \mid \bld{H}^v) - f(\sigma \mid \bld{H}^{v-1})\mid \,\, \leq \,\, f_{\max} \cdot Q_v(\sigma) \cdot \epsilon_v,
\end{equation}
where $Q_v(\sigma)$ is the probability that    $\sigma$ reaches node $v$ under distribution $\bD$ and $\epsilon_v = \TVD(\calD_v,\calE_v)$.

We first complete the proof of the lemma using \eqref{eq:stab-lemma-key}. Note that for any node $v$ labeled by item $i\in [n]$, we have $\epsilon_v=\TVD(\calD_i, \calE_i)\le \epsilon_i$.  
\begin{align}
&    \big| f(\sigma \mid \bD) - f(\sigma \mid \bE) \big| \notag \\  
    &\le \,\, \sum_{v=1}^N \big| f(\sigma \mid \bld{H}^v) - f(\sigma \mid \bld{H}^{v-1})\big| \notag \\  
    &\le \,\, f_{max} \cdot \sum_{v=1}^N Q_v(\sigma) \cdot \epsilon_v \notag \\
    &= \,\, f_{max} \cdot \sum_{i=1}^n \epsilon_i \sum_{v\in [N] :\mbox{ labeled }i} Q_v(\sigma) \notag \\
    &= f_{max} \cdot \sum_{i=1}^n \epsilon_i \cdot \Pr[\sigma \mbox{ reaches a node labeled }i]\notag \\
    &=\,\, f_{max} \cdot \sum_{i=1}^n \epsilon_i \cdot Q_i(\sigma). \label{eq:stable-abs-diff}
\end{align}
The second equality uses the fact that the events ``$\sigma$ reaches $v$'' where node $v$ is labeled by $i$, are mutually disjoint. The last equality uses the fact that the probability $Q_i(\sigma)$ of probing $i$ equals the probability of reaching some node $v\in [N]$ labeled $i$.

Towards proving \eqref{eq:stab-lemma-key}, we  introduce the following notation.
Let $Q_v(\sigma \mid \bld{U})$ denote the probability that  $\sigma$ reaches node $v$ under distribution $\bld{U}$.
Note that the event ``$\sigma$  reaches node $v$'' only depends on the random realizations at the ancestor nodes of $v$, which are all contained in $\{u\in [N] : u < v\}$ (by our indexing of nodes). Hence, 
\begin{equation}
    \label{eq:hybrid-reach-prob}
 Q_v(\sigma \mid \bld{H}^v) = Q_v(\sigma \mid \bld{H}^{v-1}) =Q_v(\sigma \mid \bld{D}) = Q_v(\sigma),\end{equation}
because each node     $\{u\in [N] : u < v\}$ has the same  distribution ($\calD_u$) under $\bD$, $\bld{H}^{v}$ and $\bld{H}^{v-1}$.

Below, let ${\cal R}_v$ denote the event that $\sigma$  reaches node $v$. Using the fact that each node $w\in [N]\setminus \{v\}$ has the same  distribution (either $\calD_w$ or $\calE_w$) under {\em both}  $\bld{H}^{v}$ and $\bld{H}^{v-1}$, we get:
\begin{equation}\label{eq:stab-lemma-conditional-1}
 \E_{\bX\sim \bld{H}^v}\left[ f(\sigma, \bX ) \mid \overline{\cal R}_v \right] -  \E_{\bX\sim \bld{H}^{v-1}}\left[ f(\sigma, \bX ) \mid \overline{\cal R}_v\right]   \,\, = \,\,0. 
\end{equation}
We now bound the difference in expectation conditional  on event ${\cal R}_v$. 
Let $\rho$ denote  the subtree of $\sigma$ rooted at $v$ (including $v$ itself). We also use $\rho$ to denote the set of nodes in this subtree.  Let $L$ denote all  leaf nodes in   subtree $\rho$, and for each $\ell\in L$  let $f_\ell$ be the function value accrued at leaf node $\ell$. 
For any realization of the r.v.s in $\rho$,  $\{x_w : w\in \rho\}$, we use $\ell(\{x_w : w\in \rho\})$  to denote the (unique) leaf that is reached when subtree $\rho$ is executed under this realization. 
This corresponds to   following the branch labeled $x_w$ out of  each  node $w\in \rho$ (starting from node $v$). 
We now define function $h$ that maps any realization $\{x_w : w\in \rho\}$ to 
the function value $f_k$ at the leaf $k=\ell(\{x_w : w\in \rho\})$. By our assumption on the objective function of the  stochastic problem, $h$ is bounded between $0$ and $f_{max}$. We now define two product distributions on nodes of $\rho$: 
\(\bld{V}_1 = \langle \calD_v , \{ \calE_{w} : w\in \rho\setminus v\} \rangle  \mbox{ and } \bld{V}_2 = \langle \calE_v , \{ \calE_{w} : w\in \rho\setminus v\} \rangle.\)
Note that \(
    h(\bld{V}_1) =   \E_{\bX\sim \bld{H}^v}\left[ f(\sigma, \bX ) \mid  {\cal R}_v \right]  \mbox{ and }   h(\bld{V}_2) =   \E_{\bX\sim \bld{H}^{v-1}}\left[ f(\sigma, \bX ) \mid  {\cal R}_v \right]
\). Applying Lemma~\ref{lem:function-val-tv} to function $h$ and product distributions $\bld{V}_1$ and $\bld{V}_2$,
\begin{align}
 &\bigg| \E_{\bX\sim \bld{H}^v}\left[ f(\sigma, \bX ) \mid   {\cal R}_v \right] -  \E_{\bX\sim \bld{H}^{v-1}}\left[ f(\sigma, \bX ) \mid   {\cal R}_v\right] \bigg|\notag \\ 
 &\le  \,\, f_{max}\cdot \TVD(\bld{V}_1, \bld{V}_2) \le \,\, f_{max}\cdot \TVD(\calD_v , \calE_v). \label{eq:stab-lemma-conditional-2}
\end{align}
 The last inequality  is by  Lemma~\ref{lem:TV-product-bound} and the fact that $\bld{V}_1$ and $\bld{V}_2$ only differ  at node $v$.

We now combine the conditional expectations from \eqref{eq:stab-lemma-conditional-1} and \eqref{eq:stab-lemma-conditional-2}. 
Using \eqref{eq:hybrid-reach-prob},   the probability of event ${\cal R}_v$  under {\em both} distributions  $\bld{H}^{v-1}$ and $\bld{H}^{v}$ is the same, which equals $Q_v(\sigma)$. So, 
\begin{align*}
    &\mid f(\sigma \mid \bld{H}^v) - f(\sigma \mid \bld{H}^{v-1})\mid = \bigg| \E_{\bX\sim \bld{H}^v}\left[ f(\sigma, \bX ) \right] -  \E_{\bX\sim \bld{H}^{v-1}}\left[ f(\sigma, \bX ) \right] \bigg|\\
    & = Q_v(\sigma  ) \cdot \bigg| \E_{\bX\sim \bld{H}^v}\left[ f(\sigma, \bX ) \mid   {\cal R}_v \right] - \, \E_{\bX\sim \bld{H}^{v-1}}\left[ f(\sigma, \bX ) \mid   {\cal R}_v\right] \bigg|\\
    &\le Q_v(\sigma ) \cdot f_{max}\cdot \TVD(\calD_v , \calE_v) = f_{max}\cdot Q_v(\sigma)\cdot \epsilon_v. 
    \end{align*}
    This completes the proof of \eqref{eq:stab-lemma-key} and Lemma~\ref{lem:stable}.

\subsection{Calculating a Stochastically Dominating Empirical Distribution}\label{subsec:emp-stoc-dom}
We now prove Theorem~\ref{thm:emp-stoch-dom}, that given i.i.d. samples from an unknown distribution $\calD$, constructs (w.h.p.) a stochastically dominating distribution $\cE$ with small total-variation distance.

In order to motivate this algorithm, consider the simple case that $\calD$ is a Bernoulli distribution, i.e., 
$\calD=\bern(p)$ where $p$ is the (unknown) probability of realizing to value $1$. We first compute the empirical mean $\widehat{p}$ of $p$  from the $m$ i.i.d. samples. { Then, we set $\calE=\bern\left(\widehat{p} + \sqrt{\frac{2\hp(1-\hp)\log(2/\delta)}{m}} + \frac{7\log(2/\delta)}{3m}\right)$, where $\delta>0$ is the algorithm's failure probability. By Bernstein inequality~\citep{bernstein1924modification}, with probability $1-\delta$ we have $ |p- \widehat{p}| \lesssim \sqrt{\frac{\sigma^2 \log(2/\delta)}{m}}+\frac{\log(2/\delta)}{m}$ where $\sigma^2 = p(1-p)$ is the variance of the Bernoulli distribution (intuitively, $\hp$ and $p$ are roughly the same). It then follows that $\calD\sd \calE$ with probability at least $1-\delta$.} Indeed, this is the strategy used in many algorithms 
that use the principle of optimism under uncertainty, such as the upper-confidence-bound (UCB) algorithm~\citep{Auer+02}. We will prove Theorem~\ref{thm:emp-stoch-dom} by giving such a construction for any 
distribution $\calD$ with finite support. We note that similar constructions and analysis have been used before, for e.g. in \citep{GuoHZ19, GuoHT+21}, but their notion  of approximating $\bD$ is   different from   our $\TVD$ condition. 

Consider any distribution $\calD$ with support $\{a_1,\ldots, a_k\}$ where $a_1 <  \cdots < a_k$, and let $\{p_i\}_{i=1}^k$ denote the respective probabilities and $\varX$ the corresponding r.v. That is, $\pr(\varX =a_i)=p_i$ for all $i\in [k]$. The algorithm takes as input $m$ i.i.d. samples $x_1,\ldots, x_m \sim \calD$.
It defines the empirical distribution as $\hp_i := \frac1m \cdot \sum_{j = 1}^m \1[x_j = a_i]$ for all $i \in [k]$.
Since this empirical distribution might not stochastically dominate $\calD$, the algorithm shifts some mass from lower to higher values in the support.
However, this shift of mass needs to ensure that the
resulting distribution is close to  $\calD$ in terms of $\TVD$. Algorithm~\ref{alg:emp_stoc_dom} describes the entire procedure.
The proof  proceeds by showing that 
the distribution 
$\calE$ returned by our algorithm stochastically 
dominates $\calD$ and is also a good approximation in
terms of $\TVD$.

\begin{algorithm}
\caption{\textsc{EmpStocDom} \label{alg:empstocdom}}
\label{alg:emp_stoc_dom}
\begin{algorithmic}[1]
\State \textbf{Input:} support $S := \{a_1,\cdots, a_k\}$, samples  $x_1, \ldots, x_m\sim \calD$, confidence parameter $\delta$
\State for $i \in [k]$, let $\hp_i \leftarrow \frac1m\cdot  \sum_{j = 1}^m \1[x_j = a_i]$ 
\State let $c = \frac{\log(2k/\delta)}{m}$ and $\epsilon =  \sqrt{2kc} +\frac{7kc}{3} $.
\State set $\bp_k = \min(\hp_k+\epsilon, 1)$ 
\If{$\bp_k$ = 1}
\State set $\bp_i = 0$ for all $i \in [k-1]$
\Else
\State let $y $ be the minimum index such that $\sum_{i=1}^y \hp_i \geq \epsilon$, and set
\begin{equation}
\label{eq:emp-shift}
\bp_i = \begin{cases} \hp_i &\text{if  } y < i \le k-1 \\
\sum_{j=1}^y \hp_j -\epsilon & \text{ if }i=y\\
0 & \text{if } i < y 
\end{cases}
\end{equation}
\EndIf
\State \textbf{Output:} distribution $\calE$ over support $S$ with probabilities $\{\bp_i\}_{i=1}^k$
\end{algorithmic}
\end{algorithm}

We will now show that the distribution $\calE$ returned by our algorithm stochastically 
dominates $\calD$ and is also a good approximation in
terms of $\TVD$. 

\paragraph{$\calE$ is a valid distribution.}  We will show that $\bp_i\ge0$ for all $i\in [k]$ and $\sum_{i=1}^k\bp_i=1$. If $\bp_k = 1$ then it is easy to see that $\calE$ is  valid: all other probabilities $\bp_i=0$. If $\bp_k < 1$ then $\hp_k<1-\epsilon$, which implies $\sum_{i = 1}^{k-1} \hp_i = 1- \hp_k > \epsilon$.
So, index $y\in [k-1]$, and  

\begin{align*}
\sum_{i = 1}^k \bp_i \,\, &= \,\, \bp_y  + \sum_{i=y+1}^{k-1} \bp_i + \bp_k \\  
&= \,\,  \sum_{i=1}^y\hp_i -\epsilon + \sum_{i=y+1}^{k-1} \hp_i + \hp_k +\epsilon \, = \, 1.
\end{align*}

\paragraph{Stochastic dominance and TV bound.} 
Instead of  using  the classical Bernstein inequality, we need to use the ``empirical Bernstein'' because we only get access to the sample variance. This approach  has also been used in \citet{AzarOM17} for UCB-VI with Bernstein bonus. 
\begin{theorem}\label{thm:Bern}[\citet{maurer2009empirical}]
    Let $X_1, X_2, \dots, X_m$ be i.i.d. random variables such that $\mathbf{E}[X_i] = \mu$, and $|X_i| \leq 1$. Then, we have with probability at least $1-\delta$:
\[
\left| \frac{1}{m}\sum_{j=1}^m X_j - \mu \right| \leq \sqrt{\frac{2V_m \log2/\delta }{m}} + \frac{7\log2/\delta}{3m}.
\]
where $V_m = \frac{1}{m-1}\sum_{j=1}^m (X_j -\frac{1}{m}\sum_{i=1}^m X_i)^2$ is the sample variance.
\end{theorem}

For each $i\in [k]$, the empirical probability $\hp_i$ in Algorithm~\ref{alg:emp_stoc_dom} is    the average of $m$ independent Bernoulli random variables with mean $p_i$; the sample variance of  this random variable is $\frac{m}{m-1} \hp_i(1-\hp_i)$. Theorem~\ref{thm:Bern} shows that with probability $1- \frac{\delta}{k}$ we have:
\[
|\hp_i - p_i| < \epsilon_i :=  \sqrt{\frac{2\hp_i(1-\hp_i) \log2k/\delta }{m}} + \frac{7\log2k/\delta}{3m}.
\] 

{
Henceforth, we assume the ``good'' event that 
$|\hp_i - p_i| < \epsilon_i$ for all $i\in[k]$. This event  occurs with probability at least $1-\delta$, by a union bound argument. Under this event, we  will show that $\calD\sd \calE$ and $\TVD(\calE, \calD)  \leq  \frac{3}{2}\epsilon$, where $\calE$ and $\epsilon$ are defined in Algorithm~\ref{alg:emp_stoc_dom}. This would complete the proof of Theorem~\ref{thm:emp-stoch-dom}. }

\paragraph{Bounding TV distance between $\calD$ and the empirical distribution.} Let $\widehat{\calE}$ denote the distribution with probabilities $\{\hp_i\}_{i=1}^k$, which is just the empirical distribution. We will show that (under the good event), 
\begin{equation}
\label{eq:tv-emp}
\TVD(\widehat{\calE}, \calD) \leq  \frac12 \epsilon.
\end{equation}
Indeed, by the definition of total variation distance, we have:
\begin{align*}
&\TVD(\widehat{\calE}, \calD) = \frac12 \sum_{i=1}^k |\hp_i-p_i |\leq \frac{1}{2} \sum_{i=1}^k \epsilon_i \\& \le  \frac{7k\log2k/\delta}{6m}+\frac{1}{2}\sum_{i=1}^k \sqrt{\frac{2 \log (2k/\delta) }{m} \hp_i } \\
& \leq  \frac{7k\log2k/\delta}{6m}+\frac{k}{2} \sqrt{\frac{2  \log(2k/\delta) }{m}\, \frac1k \sum_{i=1}^k\hp_i} \\
& = \frac{7k\log2k/\delta}{6m}+\frac{k}{2} \sqrt{\frac{2\log(2k/\delta) }{km}} \le \frac12 \epsilon.
\end{align*}
The second inequality uses the definition of $\epsilon_i$ and $\hp_i\le 1$. The third inequality is by 
Jensen's inequality for the concave function $\sqrt{x}$. The last equality uses the fact that $\sum_{i=1}^k \hp_i =1$ and $\hp_i \in [0,1]$ for all $i \in [k]$. This proves \eqref{eq:tv-emp}.

\paragraph{Bounding TV distance between $\widehat{\calE}$ and $\calE$.}  We consider two cases:
\begin{itemize}
\item If $\bp_k = 1$ then we have $\hp_k\ge 1-\epsilon$ and $\bp_i=0$ for $i\in [k-1]$,  which implies:
\begin{align*}
\TVD(\widehat{\calE}, \calE) \,\, &= \,\, \frac12 \cdot \sum_{i =1}^k |\bp_i - \hp_i| \\ 
&= \,\, \frac12 \left( |1-\hp_k| + \sum_{i =1}^{k-1} \hp_i \right) \,\, \le \,\, \epsilon.
\end{align*}

\item If $\bp_k < 1$ then $\bp_k-\hp_k=\epsilon$ and using the definition of $\bp$ from \eqref{eq:emp-shift}, 
\begin{align*}
&\TVD(\widehat{\calE}, \calE) \,\, = \,\, \frac12 \cdot \sum_{i =1}^k |\bp_i - \hp_i| \\ 
&= \,\, \frac12 \left( \epsilon +   \left| \sum_{i =1}^{y} \hp_i - \epsilon - \hp_y \right| + \sum_{i =1}^{y-1} \hp_i \right) \,\, = \,\, \epsilon,
\end{align*}
where the last equality used the choice of index $y$ to conclude $ \left| \sum_{i =1}^{y} \hp_i - \epsilon - \hp_y \right| = \epsilon - \sum_{i=1}^{y-1} \hp_i$.
\end{itemize}
Hence, in either case, we have $\TVD(\widehat{\calE}, \calE) \le \epsilon$, which combined with \eqref{eq:tv-emp} implies 
\(
\TVD(\calE, \calD) \,\, \le \,\, \TVD(\widehat{\calE}, \calE)  + \TVD(\widehat{\calE}, \calD) \,\, < \,\, \frac32\epsilon.
\)

\paragraph{Showing $\calD\sd \calE$.} Again, we consider  two cases:

\begin{itemize}
\item If $\bp_k = 1$ then  $\calE$ puts all its mass on the highest value $a_k$: so it is clear that $\calD\sd \calE$. 
\item 
If $\bp_k < 1$, consider any $v \in \R$.
If $v \leq a_y$,
\[\Pr(\var{Z} \geq v) = 1 \geq \Pr(\var{X} \geq v),\]
$\var{X}\sim \calD$ and $\var{Z}\sim \calE$. 
If $v > a_y$, then let $i \ge y+1$ be the smallest index such that $a_i \geq v$.
We have 
\begin{align*}
\Pr(\var{Z} \geq v) \,\, &= \,\, \sum_{j=i}^{k} \bp_j =\sum_{j=i}^{k-1} \hp_j + \hp_k + \epsilon \\ 
&\geq \,\, \sum_{j=i}^{k} (p_i - \epsilon_i) + \epsilon \geq\Pr(\varX \geq v)\,,
\end{align*}
where the first inequality follows due to the good event. The second inequality uses the same analysis in \eqref{eq:tv-emp}. Hence, $\calD\sd \calE$. 

\end{itemize}

\subsection{Framework for Down-Monotone Problems}\label{subsec:down-mon}
The algorithmic framework described above assumed that the stochastic problem is up-monotone. 
We now describe the changes needed to handle down-monotone problems. See Definition~\ref{def:monotonicity} for the formal definition of up/down monotonicity. The only change to Algorithm~\ref{alg:greedy-ucb} is that, at each step $t$, we now need to compute distribution $\bE^t$ that is {\em stochastically dominated} by the true (unknown) distribution $\bD$. That is, we need $\bE^t\sd \bD$ rather than $\bD\sd \bE^t$. Such a distribution can be constructed in exactly the same way as Algorithm~\ref{alg:emp_stoc_dom}: we just order the values in the support in {\em decreasing} order $a_1 > \cdots > a_k$ (instead of increasing order). Using the same proof as in \S\ref{subsec:emp-stoc-dom}, we obtain:
\begin{theorem}
\label{thm:emp-stoch-dom2}
There is an algorithm that, given $m$ i.i.d.\ samples from a distribution $\calD$ with finite support-size $k$, and parameter $\delta > 0$,  computes a distribution $\calE$ that satisfies the following properties with probability at least $1-\delta$:
\begin{itemize}
\item $\calE$ is stochastically dominated by $\calD$, and 
\item the total-variation distance $\TVD(\calE, \calD) < \sqrt{\frac{5k\log(2k/\delta)}{m}} +\frac{4k\log(2k/\delta)}{m}$. 
\end{itemize}  
\end{theorem}

For the analysis, we use the following analogue of the stablility lemma with down-monotonicity.
\begin{lemma}\label{lem:stable2}
Consider a stochastic problem that is down-monotone. Suppose that $\bE=\{\calE_i\}_{i=1}^n$ and $\bD=\{\calD_i\}_{i=1}^n$ are product distributions  
such that $\bE\sd \bD$ and $\TVD(\calE_i, \calD_i) \leq \epsilon_i$ for each $i \in [n]$. If $\sigma$
is the policy returned by $\ALG(\bE)$, 

and $\sigma^*$ is an optimal   policy under $\bD$,
then:
\begin{equation*} 
f(\sigma) - \alpha\cdot  f(\sigma^*) = \E_{\bx \sim \bD}\left[f(\sigma, \bx) - \alpha\cdot f(\sigma^*, \bx)\right] \leq    f_{\max} \sum_{i=1}^n {Q}_{i}(\sigma) \cdot \epsilon_i 
    \,,
\end{equation*}
where ${Q}_{i}(\sigma)$ denotes the probability that item $i$ is probed by policy $\sigma$ under distribution $\bld{D}$.
\end{lemma}
The proof of this lemma is identical to that in \Cref{sec:stable} for up-monotonicity.
Using Theorem~\ref{thm:emp-stoch-dom2} and Lemma~\ref{lem:stable2}, the proof of Theorem~\ref{thm:main} (for down-monotone problems) is then identical to that for up-monotonicity (done earlier).

\subsection{Framework for Maximization Problems}
\label{sec:max-probs}
The algorithm/analysis  so far assumed that the stochastic problem $\cP$ has a minimization objective. 
We now describe the changes needed to handle a maximization problem $\cP$. Given a product distribution $\bld{U}$, we have
$$\OPT(\bld{U}) = \max_{\tau\in \cC} f(\tau|\bld{U}),$$
where  $\cC$ is the set of all feasible policies and  $f(\tau|\bld{U}) := \E_{\bx\sim \bld{U}}[f(\tau,\bx)]$ is  the expected value  of  policy $\tau$ under distribution $\bld{U}$. Note that $f(\tau) = f(\tau|\bD)$ where $\bD$ is the true distribution. 
As before, we assume access to a $\alpha$-approximation algorithm 
$\ALG$ for problem $\cP$; that is,
given any problem instance (with item distributions),  $\ALG$ finds a policy of expected value at least $\alpha$ times the optimum; here $\alpha\le 1$. Furthermore, we  assume that $\cP$ is down-monotone (up-monotone problems can be handled similarly). We note that the algorithm is unchanged: we still use   Algorithm~\ref{alg:greedy-ucb}.
The main  change is in the expression for regret, which is reflected in the modified stability lemma: 

\begin{lemma}\label{lem:stable-max}
Consider a stochastic maximization problem that is down-monotone. Suppose that $\bE=\{\calE_i\}_{i=1}^n$ and $\bD=\{\calD_i\}_{i=1}^n$ are product distributions  
such that $\bD\sd \bE$ and $\TVD(\calE_i, \calD_i) \leq \epsilon_i$ for each $i \in [n]$. 
If $\sigma$ is the policy returned by $\ALG(\bE)$  
and $\sigma^*$ is an optimal  policy under $\bD$, then:
\begin{equation}\label{eq:per-time-bound}
\alpha\cdot  f(\sigma^*) - f(\sigma)    = \E_{\bx \sim \bD}\left[\alpha\cdot f(\sigma^*, \bx) - f(\sigma, \bx)\right] \leq    f_{\max} \sum_{i=1}^n {Q}_{i}(\sigma) \cdot \epsilon_i 
    \,,
\end{equation}
where ${Q}_{i}(\sigma)$ denotes the probability that item $i$ is probed by policy $\sigma$ under distribution $\bld{D}$.
\end{lemma}
The proof of this  lemma follows the same approach as for Lemma~\ref{lem:stable}. 

\begin{align}
\alpha\cdot  f(\sigma^*)  - f(\sigma) & =  \alpha\cdot  f(\sigma^*| \bD) - f(\sigma| \bD)  = \alpha\cdot  f(\sigma^*| \bD) -  f(\sigma| \bE) +f(\sigma| \bE) - f(\sigma| \bD)   \notag \\
& \le \alpha\cdot  f(\sigma^* | \bD)  -  \alpha\cdot \OPT(\bE) + f(\sigma| \bE) - f(\sigma| \bD)    \label{eq:stable-1-max} \\ 
& =  \alpha\cdot \OPT(\bD) -  \alpha\cdot \OPT(\bE) + f(\sigma| \bE) - f(\sigma| \bD)    
  \,\,\le\,\, f(\sigma | \bE) - f(\sigma | \bD) \label{eq:stable-2-max}\\
  & \le f_{max} \cdot \sum_{i=1}^n \epsilon_i \cdot Q_i(\sigma). \label{eq:stable-3-max}
\end{align}
Inequality \eqref{eq:stable-1-max} uses the fact that $\sigma$ is an $\alpha$-approximate policy to the instance with distribution $\bE$. In \eqref{eq:stable-2-max}, the equality uses the fact that $\sigma^*$ is an optimal policy 
for the instance with distribution $\bD$, and the inequality is by down-monotonicity (Definition~\ref{def:monotonicity}) and $\bD\sd \bE$. Finally, \eqref{eq:stable-3-max} follows from the upper bound on $\bigg| f(\sigma | \bD) - f(\sigma | \bE)  \bigg|$ from \eqref{eq:stable-abs-diff}. Lemma~\ref{lem:stable-max} then implies Theorem~\ref{thm:main}  for maximization, down-monotone problems:  the proof is identical to that for minimization problems. Finally, the same guarantee holds when $\cP$ is up-monotone (per the discussion in \S\ref{subsec:down-mon}).

\section{Handling General Distributions}\label{sec:continuous}
In  this section, we extend  results from the discrete setting  to continuous (or even mixed) distributions. Here,  each random variable \( \varX_i \sim \mathcal{D}_i \) is supported over a bounded interval \([a, b]\). 

We will assume that there is an \( \alpha \)-approximation algorithm \( \ALG \) for the stochastic problem \( \mathcal{P} \) that always returns a {$k$-threshold} policy, defined as follows. The policy is represented by a decision tree \( \sigma \), where each internal node is labeled by an item \( i \in [n] \) to probe, along with a list of thresholds \( \langle \tau_1, \dots, \tau_{k-1} \rangle \) that define a partition of \([a, b]\), into intervals  \( \{J_1, \dots, J_k\} \). Each branch \( s \in [ k] \) is followed when the observed value \( x_i \in J_s \).
This ensures that each node has at most $k$ outgoing branches.
We also assume that for any $k$-threshold policy $\sigma$, the cost function \( f(\sigma, \bx) \) can be expressed in the following form:

\begin{equation}\label{eq:representcts}
f(\sigma,\bx)= \sum_{\ell : \text{leaf of } \sigma} \left( \prod_{i \in P_\ell} \1_{x_i \in J_\ell^i} \right) f_{\sigma, \ell}(\bx),    \qquad \forall~\bx \in [a,b]^n,
\end{equation}
where $ f_{\sigma, \ell}(\bx) $ is the cost function associated with the leaf $\ell$, $ P_\ell $ is the set of items probed along the path to $\ell$ and $J^i_\ell$ is the interval corresponding to the branch out of item $i$ leading to  leaf $\ell$. 
We note that, given a realization $\bx$, only one summand in \eqref{eq:representcts} will be non-zero (the one corresponding to the leaf $\ell$ that $\bx$ traces).

In this section, we use $f_{\sigma}(\bx):[a,b]^n\rightarrow \R_+$ to denote the function $f(\sigma,\bx)$ for a fixed policy $\sigma$. 

{

\begin{remark}
In the discrete setting with support size $k$, any policy naturally has at most $k$ branches per node (one per possible outcome), so the $k$-threshold structure arises automatically. In the continuous setting, an arbitrary policy could have uncountably many branches per node, making it hard to represent or analyze directly. The $k$-threshold assumption restricts each node to at most $k$ branches defined by thresholds, rendering the policy class  computationally and analytically tractable. 
We note that this assumption is not restrictive for our applications: the optimal policies for prophet inequality and Pandora's box are already $2$-threshold policies, and more generally, known approximation algorithms for monotone stochastic problems typically return threshold-based policies.
\end{remark}
}

\medskip
\noindent{\bf Assumption on the objective function.} We need to make some continuity and boundedness assumptions on the leaf functions $f_{\sigma, \ell}$. In order to state them formally, we start with some definitions.
 
 \begin{definition}[BV function]\label{def:BV funtion}
    The total variation of a 1-dimensional  function $g:[a,b]\rightarrow \R$ is:
    $$
    V_a^b(g) = \sup_{\mathcal{P}} \sum_{i=0}^m |g(x_i) - g(x_{i-1})|.
    $$
    where the supremum is taken over all finite partitions ${\cal P} = \langle a=x_0 < x_1 < \cdots x_m=b\rangle$ of  the interval.     We say that $g$ is a bounded variation (BV) function  if and only if $V_a^b(g) < \infty$.
\end{definition}

We list several useful properties of BV functions \citep{heil2019introduction}.
\begin{itemize}
     \item Any monotone function $g$ over $[a,b]$ is  BV  and $V_a^b(g) = |g(b)- g(a)|$.
     \item If $g$ is a Lipchitz function over $[a,b]$ then it is BV and  $V_a^b(g) \le L(b-a)$, where $L$ is the Lipchitz constant.
      \item For any BV function $g$ over $[a,b]$, we have $|g(x)| \,\le\, |g(a)| \,+\, V_a^b(g)$ for all  $x\in [a,b]$.

      \end{itemize}

We now introduce  a  generalization of BV to  high dimensions. 
Informally, this requires every 1-dimensional ``slice'' of $f$ to be  BV. 
\begin{definition}[Coordinate-wise BV]
A $n$-dimensional function $f(\bx) : [a,b]^n \rightarrow \mathbb{R}$ is said to have coordinate-wise bounded variation  if  the following conditions hold.
\begin{itemize}
    \item For  any index $ i \in [n] $  and  $\bx_{-i}= (x_1, \dots,x_{i-1}, x_{i+1},\dots x_n)\in [a,b]^{n-1}$,   the 1-dimensional function $f(y|\bx_{-i}) := f(x_1, \dots, x_{i-1}, y, x_{i+1}, \dots, x_n)$  is BV.
    
    \item Moreover, the total variation $(\TV)$ of $f(y|\bx_{-i})$ is uniformly bounded for all index $i$ and $ \bx_{-i}$, i.e., there is a finite value $ \TV(f) $ such that 
$$
\sup_{i, \bx_{-i}} V_{a}^{b}(f(y | \bx_{-i})) \leq \TV(f).
$$

\end{itemize}

\end{definition}

Similarly, we  define coordinate-wise continuous functions as follows.
\begin{definition}[Coordinate-wise Continuous]
A $n$-dimensional function $ f(\bx) : [a,b]^n \rightarrow \mathbb{R} $  is coordinate-wise continuous if the 1-dimensional function $f(y| \bx_{-i}) := f(x_1, \dots, x_{i-1}, y, x_{i+1}, \dots, x_n)$  is continuous for all $ i \in [n] $  and  $\bx_{-i}= (x_1, \dots,x_{i-1}, x_{i+1},\dots x_n)\in [a,b]^{n-1}$.
 \end{definition}

We assume that the objective function is non-negative and bounded, with $f(\sigma, \bx)\le f_{\max}$. Additionally, we make the following (mild) assumptions on the ``leaf functions''  $f_{\sigma, \ell}$  for any $k$-threshold policy. 
\begin{assumption}\label{assumptioncts} 
For any $k$-threshold policy $\sigma$  and leaf $\ell$,
we assume that $f_{\sigma, \ell}(\bx)$ is coordinate-wise BV and coordinate-wise continuous. Furthermore, we assume that the total variation of $f_{\sigma, \ell}$ is uniformly bounded over all  $ \sigma $ and  $\ell$, i.e., there is a finite value $\TV_{\max}(f)$, called the total-variation of $f$, such that 
$$ \sup_{\sigma, \ell} \, \TV(f_{\sigma, \ell})  \, \leq \, \TV_{\max}(f).     
$$
\end{assumption} 

\medskip
\noindent {\bf Preliminaries on probability distribution functions.} We will be working with general distributions (discrete or continuous) here. 
{
We use the Kolmogorov--Smirnov (KS) distance instead of total variation (TV) distance here, since the TV distance between two distinct continuous distributions can equal $1$ even when they are arbitrarily close, making TV-based bounds vacuous in this setting.
}To this end, we now introduce some definitions. 

\begin{definition}[Cumulative Distribution Function]
    Given a probability distribution $P$ over $\mathbb{R}^n$, we denote by $F_P$ the  cumulative distribution function (CDF) of $P$, defined as: 
    $$
    F_P(x_1 \dots x_n) = \Pr_{(\varX_1 \dots \varX_n) \sim P}[ \varX_1 \leq x_1,  \ldots,  \varX_n \leq x_n].
    $$

    Moreover, if $P = P_1 \times \dots \times P_n$ where $\{P_i\}_{i \in [n]}$ is a product distribution, we have:
    $$
     F_P(x_1 \dots x_n)= \prod_{i=1}^n F_{P_i}(x_i).
    $$
\end{definition}

\begin{definition}
    The Kolmogorov--Smirnov(KS) distance between two distributions $P$ and $Q$ on $\mathbb{R}$ is defined as :
$$
d_{\mathrm{KS}}(P, Q) := \sup_{x \in \mathbb{R}} |F_P(x) - F_Q(x)|.
$$
\end{definition}

In our proof, we require a few standard results from analysis (stated below).

\begin{theorem}[Combination of Lemma 5.2.14 and Theorem 5.2.15 from \citet{heil2019introduction}]\label{thm:bv-decomp}
 If $f:[a, b] \rightarrow \mathbb{R}$, then the following two statements are equivalent.
 \begin{itemize}
     \item  $f \in \mathrm{BV}[a, b]$.
     \item There exist monotone non-decreasing functions $f^+$ and $f^-$ such that $f=f^+-f^-$. And $V_a^b(f) =  f^+(b) + f^-(b) - f^+(a) - f^-(a) $
 \end{itemize}

\end{theorem}

\begin{theorem}[Theorem 12.14 from \citet{gordon1994integrals}]\label{thm:prod-integral}
 Let $f$ and $\phi$ be bounded functions defined on $[a, b]$. If $f$ is Riemann-Stieltjes integrable with respect to $\phi$ on $[a, b]$, then $\phi$ is Riemann-Stieltjes integrable with respect to $f$ on $[a, b]$ and

$$
\int_a^b \phi d f=f(b) \phi(b)-f(a) \phi(a)-\int_a^b f d \phi
$$
\end{theorem}

\begin{theorem}[Theorem 12.15 from \citet{gordon1994integrals}]\label{thm:int-exists}
 Let $f$ and $\phi$ be bounded functions defined on $[a, b]$. If $f$ is continuous on $[a, b]$ and $\phi$ is of bounded variation on $[a, b]$, then $f$ is Riemann-Stieltjes integrable with respect to $\phi$ on $[a, b]$.
\end{theorem}

The following theorem allows us to bound the difference in expectation of a BV function over two different distributions. 
  \begin{theorem}\label{thm:ksbnd}
  Let $g:[a,b]\rightarrow \R$ be a BV function, and $ P $ and $ Q $  be probability distributions on $\R$.  If the Riemann–Stieltjes integrals $ \int_a^b g(x)\,dF_P(x) $ and $ \int_a^b g(x)\,dF_Q(x) $ exist, then:
$$
\left| \int_a^b g \, dF_P - \int_a^b g \, dF_Q \right| \leq \left(2g_{\max}+V_a^b(g)\right) \cdot d_{\mathrm{KS}}(P, Q),
$$
where $g_{\max} := \max_{x \in [a,b]} |g(x)| $.
\end{theorem}

\begin{proof}
By the assumption, we have:
\begin{align}
\left| \int_a^b g \, dF_P - \int_a^b g \, dF_Q \right|  \,\, &= \,\,    \left|g(b)\left(F_P(b)-F_Q(b)\right) -g(a)\left(F_P(a)-F_Q(a)\right)  -\int_a^b (F_P - F_Q) \, dg \right| \nonumber \\
&\leq \,\,  \left|g(b)\left(F_P(b)-F_Q(b)\right)\right| + \left|g(a)\left(F_P(a)-F_Q(a)\right)\right| +\left|\int_a^b  F_P - F_Q  \, dg\right| \nonumber \\
&\leq \,\,  2g_{\max} \cdot d_{\mathrm{KS}}(P, Q) +\left|\int_a^b  F_P - F_Q  \, dg\right|  .\nonumber 
\end{align}
The equality follows from the integration by parts formula (Theorem~\ref{thm:prod-integral}). 

Using the fact that $g$ is a BV function,  we apply the Jordan decomposition (Theorem~\ref{thm:bv-decomp}) to write $ g = g^+ - g^- $, where $ g^+ $ and $ g^- $ are non-decreasing functions, and $V_a^b(g) =  g^+(b) + g^-(b) - g^+(a) - g^-(a)$. We now have

{ \begin{align*}
&\left|\int_a^b  F_P(x) - F_Q(x) \, dg(x) \right| = \bigg| \,\, \int_a^b F_P(x) - F_Q(x)  \, dg^+(x) - \int_a^b  F_P(x) - F_Q(x)  \, dg^-(x) \bigg|\\
&\leq \,\, \int_a^b \left| F_P(x) - F_Q(x) \right| \, dg^+(x) + \int_a^b \left| F_P(x) - F_Q(x) \right| \, dg^-(x) \\
&\leq \,\, \sup_{x \in \mathbb{R}} \left| F_P(x) - F_Q(x) \right| \cdot (g^+(b) - g^+(a))  + \sup_{x \in \mathbb{R}} \left| F_P(x) - F_Q(x) \right| \cdot (g^-(b) - g^-(a)) \\
&= V_a^b(g) \cdot d_{\mathrm{KS}}(P, Q),
\end{align*}}

This establishes the theorem.
\end{proof}

\paragraph{Online to offline framework.} The overall algorithm and analysis is  similar to that in the discrete case. Here, we focus on stochastic problems satisfying Assumption~\ref{assumptioncts} and up-monotonicity. We also assume that there is  an $\alpha$-approximation algorithm via $k$-threshold policies. These results also extend to down-monotone problems in a straightforward way. 

At each time step $t=1,\cdots T$, we first construct a (modified) empirical distribution $\bE^t$ that stochastically dominates the true distribution $\bD$ w.h.p. Then, we obtain a $k$-threshold policy $\sigma^t$ by running the $\alpha$-approximation  algorithm on distribution $\bE^t$. Finally, we run  policy $\sigma^t$ on a realization $\bx^t \sim \bD$ and observe semi-bandit feedback. In the next two subsections, we prove the  sampling lemma (that  constructs a stochastically dominating distribution) and the stability lemma (that bounds the single-step regret).

\subsection{Sampling Algorithm}\label{lem:ctssampling}
Suppose the underlying distribution $\calD$ has CDF  $F(x)$. We will show how to construct an empirical distribution $\calE$ that dominates $\calD$ where $\calD$ and $\calE$ are close enough in terms of KS distance. 
\begin{theorem}\label{thm: cts sampling}
There is an efficient algorithm  that, given $m$ i.i.d. samples from a   distribution $\calD$ with bounded support $[a,b]$, computes a distribution $\calE$  satisfying the following properties with probability at least $1-\delta$.
\begin{itemize}
\item $\calE$ stochastically dominates $\calD$.
\item The KS distance $d_{KS}(\calE, \calD) < 2 \sqrt{\frac{\log(2/\delta)}{2m}}$. 
\end{itemize}  
\end{theorem}
\begin{proof}
The main idea is to first construct the empirical distribution, and subsequently move an $\epsilon$ mass from low values to the maximal value $b$. See \Cref{alg:emp_stoc_domcts} for a formal description of our construction.

\begin{algorithm}
\caption{\textsc{EmpStocDomCts}}
\label{alg:emp_stoc_domcts}
\begin{algorithmic}[1]
\State \textbf{Input:} support $S := [a,b]$, samples  $x_1, \ldots, x_m\sim \calD$, confidence parameter $\delta$
\State let $\epsilon = \sqrt{\frac{\log(2/\delta)}{2m}}$
\State  let $F_m(x) \leftarrow \frac1m\cdot  \sum_{j = 1}^m \1[x_j \leq x]$ 
\State let $\widehat{F}_m(x)= (F_m(x)-\epsilon)^+$ for all $x<b$, and set $\widehat{F}_m(x)=1$ for all $x\geq b$.
\State \textbf{Output:} distribution $\calE$ over support $S$ with CDF $\widehat{F}_m(x)$
\end{algorithmic}
\end{algorithm}

\paragraph{$\calE$ is a valid distribution.} By our construction, $\widehat{F}_m(x)$ is a right continuous and non-decreasing function with $\widehat{F}_m(x) =0$ for all $x < a$ and $\widehat{F}_m(x)=1$ for all $x \geq b$. Hence it is a valid CDF.

\paragraph{Stochastic dominance and $d_{KS}$ bound.} Using the choice of $\epsilon$ in the algorithm, by the Dvoretzky-Kiefer-Wolfowitz (DKW) Inequality (Theorem~\ref{thm:dkw}), we have:
$$
\Pr \left(\sup_{x\in \mathbb{R}} \, |F_m(x) - F(x)| \,\, \geq \,\, \epsilon \right) \,\, \leq \,\, \delta .
$$
Henceforth, we assume the ``good'' event that $|F_m(x) - F(x)| < \epsilon$ for all $x \in [a,b]$, which occurs with probability at least $1-\delta$. Under this event, we  will show that $\calD\sd \calE$ and $d_{KS}(\calE, \calD) \leq 2\epsilon$. This would complete the proof of Theorem~\ref{thm: cts sampling}.  
\paragraph{Bounding $d_{KS}(\widehat{F}_m, F)$.}
First we notice that 
$$
\sup_{x \in \mathbb{R}} |\widehat{F}(x) - F_m(x)| \,\, = \,\, \sup_{x <b} |\widehat{F}(x) - F_m(x)| \,\, = \,\, \sup_{x < b} \left( F_m(x) - (F_m(x)-\epsilon)^+\right) \,\, \leq \,\, \epsilon.
$$
The first equality  uses  $\widehat{F}(x)=F_m(x)=1$ for $x \geq b$.
Moreover, by triangle inequality, we have :
$$
d_{KS}(\calE, \calD) \,\, = \,\, \sup_{x \in \mathbb{R}} |\widehat{F}(x) - F(x)| \,\, \leq \,\, \sup_{x \in \mathbb{R}} |\widehat{F}(x) - F_m(x)| \, + \, \sup_{x \in \mathbb{R}} |F(x) - F_m(x)| \,\, \leq \,\,  2 \epsilon .
$$

\paragraph{Showing $\calD\sd \calE$.} 
Let $\varY \sim \widehat{F}$ and $\varX \sim F$. We want to show $\Pr(\varY \geq t) \geq \Pr(\varX \geq t)$ for any $t \in \mathbb{R}$. 
Under the good event assumption, we have $F_m(x) - \epsilon \leq F(x)$ for all $x \in \R$. So, $\widehat{F}(x)=(F_m(x) - \epsilon)^+ \le F(x)$ for all $x<b$. Combined with the fact that $\widehat{F}(x)=F_m(x)=1$ for $x \geq b$, it follows that  $\varY$ stochastically dominates $\varX$.
\end{proof}

\subsection{Stability Lemma}
In this subsection, we  present a  version of the stability lemma that holds for general distributions and $k$-threshold policies. 
\begin{lemma}[New Stability lemma]\label{lem:stable-cont}
Consider a stochastic problem that is up-monotone and which satisfies Assumption~\ref{assumptioncts}. Suppose that $\bE=\{\calE_i\}_{i=1}^n$ and $\bD=\{\calD_i\}_{i=1}^n$ are product distributions such that $\bD\sd \bE$ and $d_{KS}(\calE_i, \calD_i) \leq \epsilon_i$ for each $i \in [n]$. If $\sigma=\ALG(\bE)$ is a $k$-threshold policy and $\sigma^*$ is the optimal policy under $\bD$, then 
\begin{equation}\label{eq:per-time-bound}
f(\sigma) - \alpha\cdot  f(\sigma^*) = \E_{\bx \sim \bD}\left[f(\sigma, \bx) - \alpha\cdot f(\sigma^*, \bx)\right] \leq  k\cdot  (2f_{\max}+ \TV_{\max}(f)) \cdot \sum_{i=1}^n {Q}_{i}(\sigma) \cdot \epsilon_i
    \,,
\end{equation}
where ${Q}_{i}(\sigma)$ denotes the probability that item $i$ is probed by policy $\sigma$ under distribution $\bld{D}$.
\end{lemma}
\begin{proof}
     We follow the same proof structure used in the discrete setting (Lemma~\ref{lem:stable}). In fact, the proof of the new stability lemma remains the same except for Inequality~\eqref{eq:stab-lemma-conditional-2}. Recall that $v$ denotes some node in the decision tree of ($k$-threshold) policy $\sigma$, and that these nodes are partially-ordered based on the tree. The product distribution $\bld{H}^v = \calD_1 \times \cdots \times \calD_{v} \times \calE_{v+1} \times \cdots \times \calE_N$ and  ${\cal R}_v$ denotes the event that $\sigma$  reaches node $v$. Let $i\in [n]$ denote the item labeling node $v$. Instead of \eqref{eq:stab-lemma-conditional-2} we will now show the following inequality, which suffices to prove the lemma (as before). 
\begin{equation}\label{eq:cont-stab-lem}
    \left| \E_{\bx \sim \bld{H}^v} \left[ f_\sigma(\bx) | \mathcal{R}_v \right] - \E_{\bx \sim \bld{H}^{v-1}} \left[ f_\sigma(\bx) | \mathcal{R}_v \right] \right| \,\, \leq \,\, (2f_{\max}+ \TV_{\max}(f)) \cdot \epsilon_i, \quad \text{for any node } v.
\end{equation}
In order to prove this, we first apply the law of total expectation to show:
\begin{align}
&\E_{\bx \sim \bld{H}^v} \left[ f_\sigma(\bx) | \mathcal{R}_v \right] - \E_{\bx \sim \bld{H}^{v-1}} \left[ f_\sigma(\bx) | \mathcal{R}_v \right] \nonumber \\
=\,\, &\E_{\bx_{-i}} \left[ \E_{x_i \sim \calD_i} \left[ f_\sigma(\bx) | \mathcal{R}_v, \bx_{-i} \right] \right] - \E_{\bx_{-i}} \left[ \E_{x_i \sim \calE_i} \left[ f_\sigma(\bx) | \mathcal{R}_v, \bx_{-i} \right] \right] \nonumber \\
=\, &\E_{\bx_{-i}} \left[ \1[{\cal R}_v]\cdot \left( \E_{x_i \sim \calD_i} \left[ f_\sigma(x_i,\bx_{-i}) \right] - \E_{x_i \sim \calE_i} \left[ f_\sigma(x_i, \bx_{-i}) \right]\right)  \right]. \label{eq:ctsKS}
\end{align}
The first equality uses the fact that $ \bld{H}^v $ and $ \bld{H}^{v-1} $ differ only in the distribution of $\varX_v $, and hence share the same marginal distribution for $\varX_{-v} $. The second equality follows from the independence of the variables $ \{\varX_j\}_{j=1}^n $ and the fact that event $\mathcal{R}_v$ is  determined by the realization $\bx_{-i}$ (recall that node $v$ corresponds to item $i$).

We now condition on any realization $\bx_{-i}$ corresponding to event $\mathcal{R}_v$, i.e., policy $\sigma$  reaches node $v$ under this realization. We will show that, conditioned on $\bx_{-i}$,
\begin{equation}\label{eq:cont-stab-lem-2}
  |  \E_{x_i \sim \calD_i} \left[ f_\sigma(x_i, \bx_{-i}) \right] - \E_{x_i \sim \calE_i} \left[ f_\sigma(x_i, \bx_{-i}) \right] | \le k\cdot \left(2f_{max} + \TV_{\max}(f)\right)\cdot \epsilon_i
\end{equation}
Combined with \eqref{eq:ctsKS} this would prove \eqref{eq:cont-stab-lem}, which in turn implies the lemma.

It now remains to prove \eqref{eq:cont-stab-lem-2}. 
As $\sigma$ is a $k$-threshold policy, there are at most $k$ choices for the  leaf node reached under $\bx=(x_i, \bx_{-i})$ as  $x_i$ varies. Indeed,  there are only $k$ branches out of any node and each branch out of node $v$ leads to a unique leaf conditioned on $\bx_{-i}$. 
Let $\{J_s\}_{s=1}^k$ denote  the partition of $[a,b]$ corresponding to the thresholds at node $v$; each branch out of $v$ corresponds to $x_i\in J_s$ for some $s \in [k]$. For each $s\in [k]$ let $\ell_s$ denote the unique leaf in $\sigma$ corresponding to any realization  $(x_i, \bx_{-i})$ where $x_i\in J_s$. Combined with the representation in~\eqref{eq:representcts}, conditioned on $\bx_{-i}$, we can write the function as:
$$f_\sigma(x_i, \bx_{-i}) = \sum_{s=1}^k \1_{x_i \in J_s} \cdot f_{\sigma, \ell_s}(x_i, \bx_{-i}) = \sum_{s=1}^k g_s(x_i),$$
where we define the 1-dimensional function $g_s(y):=\1_{y \in J_s} \cdot f_{\sigma,\ell_s}(y, \bx_{-i})$ for each $s\in [k]$. 
By Assumption~\ref{assumptioncts} it follows that $g_s$ is a BV function with $V_a^b(g_s)\le \TV_{\max}(f)$.  We want to apply 
Theorem~\ref{thm:ksbnd} to function $g_s$ on  interval $J_s$ with distributions $\calD_i$ and $\calE_i$. However, $J_s$ itself may not be a closed interval. To resolve this, we use a standard approximation argument. Let $J_{s,r} \subseteq J_s$ be a sequence of ascending closed intervals such that $ \lim\limits_{r\rightarrow \infty}J_{s,r} = J_s$. Then we have:
\begin{align*}
     &\left|  \int_{J_s} g_s(x) \, dF_{\calD_i} - \int_{J_s} g_s(x) \, dF_{\calE_i} \right |   = \lim_{r\rightarrow \infty} \left|\int_{J_{s,r}} g_s(x )dF_{\calD_i} - \int_{J_{s,r}} g_{s}(x) \, dF_{\calE_i}\right| \\ 
     & \le  (2 f_{max} + \TV_{\max}(f)) \cdot d_{KS}(\calD_i, \calE_i) \,\,\le\,\, (2 f_{max} + \TV_{\max}(f)) \cdot \epsilon_i.  
\end{align*}
The equality is by the continuity of integration over ascending intervals. The first inequality is by Theorem~\ref{thm:ksbnd} to function $g_s$ restricted to the closed interval $J_{s,r}$ with distributions $\calD_i$ and $\calE_i$. In order to apply this result we need the existence of  the Riemann–Stieltjes integrals:
$$
\int_{J_{s,r}} g_s(x) \, dF_{\calD_i}
\quad \text{and} \quad 
\int_{J_{s,r}} g_s(x) \, dF_{\calE_i}, \quad \text{for all $r$} .$$
The existence  follows from Theorem~\ref{thm:int-exists}: note that $g_s$ is continuous over the closed interval $J_{s,r}$ and   both $ F_{\calD_i} $ and $ F_{\calE_i} $ are BV  functions (they are monotone). It now follows that, conditioned on $\bx_{-i}$,
\begin{align*}
& \left|  \E_{x_i \sim \calD_i} \left[ f_\sigma(x_i, \bx_{-i}) \right] - \E_{x_i \sim \calE_i} \left[ f_\sigma(x_i, \bx_{-i}) \right]  \right| =  \left| \E_{x_i \sim \calD_i} \left[ \sum_{s=1}^k g_s(x_i)  \right]  - 
\E_{x_i \sim \calE_i} \left[ \sum_{s=1}^kg_s(x_i)  \right]\right| \\
& \le \sum_{s=1}^k \left| \E_{x_i \sim \calD_i} \left[ g_s(x_i)  \right]  - 
\E_{x_i \sim \calE_i} \left[ g_s(x_i)  \right]\right|  
\leq\  k \cdot (2f_{\max}+ \TV_{\max}(f)) \cdot \epsilon_i.  
\end{align*}
This completes the proof of \eqref{eq:cont-stab-lem-2} and the lemma.  
\end{proof}

 \subsection{Overall Regret}
We can now combine the sampling and stability lemmas to obtain our  result for general distributions. The proof is identical to that of Theorem~\ref{thm:main} where we replace Theorem~\ref{thm:emp-stoch-dom} with Theorem~\ref{thm: cts sampling} and Lemma~\ref{lem:stable} with Lemma~\ref{lem:stable-cont}. 
\begin{theorem}\label{thm:main-cont} Consider a stochastic problem $\cP$ that is either up-monotone or down-monotone, and satisfies Assumption~\ref{assumptioncts}. Suppose that $\cP$  has an  $\alpha$-approximation algorithm via $k$-threshold policies. Then, there is a  polynomial-time  semi-bandit learning algorithm  for $\cP$  with  $\alpha$-regret $O(k n  (f_{\max}+ \TV_{\max}(f)) \cdot  \sqrt{T\log (nT)})$. Here, $n$ is the number of items, $T$ is the number of periods, and $f_{\max}$ and $\TV_{\max}(f)$ are  the maximal-value and total-variation of the objective function. 
\end{theorem}

{
\begin{remark}
A Lipschitz bound on the policy objective cannot in general remove the dependence on $k$ because the function $f(\sigma, \bx)$ is not Lipschitz in the distribution around threshold values.
To see this, consider a policy $\sigma_\tau$ for single-item prophet inequality having threshold $\tau$. Suppose we have two distributions: $\calD$ a point mass at $\tau - \epsilon$ and $\calE$ a point mass at $\tau + \epsilon$. Then $f(\sigma_\tau | \calD) = 0$ but $f(\sigma_\tau | \calE) = \tau + \epsilon$, a difference of nearly $\tau$ despite $\calD$ and $\calE$ being arbitrarily close as $\epsilon \to 0$. 
Since a $k$-threshold policy has $k$ such points of discontinuity, each contributing $O(f_{\max})$ to the objective, the dependence on $k$ is unavoidable in general. 
\end{remark}
}

\begin{remark}\label{rk:coordinate monotone}
If each leaf function $f_{\sigma, \ell}(\bx)$ is \emph{coordinate-wise monotone}, i.e.,  every one-dimensional ``slice'' of $f_{\sigma, \ell}$ is monotone, 
then we have a simpler result. Formally, a $n$-dimensional function 
$f(\bx): [a,b]^n \to \mathbb{R}$
is coordinate-wise monotone if, for all $i \in [n]$, the one-dimensional function $f(y | \bx_{-i})$
is monotone for every possible $
\bx_{-i}.
$ In this case, we have $\TV_{\max}(f) \le f_{\max}$ 
(see the properties listed after Definition~\ref{def:BV funtion}), 
and the regret bound simplifies to
$
O\bigl(k n f_{\max} \cdot \sqrt{T \log (nT)}\bigr).
$    
 \end{remark}

\section{Censored and Binary Feedback}\label{sec:binary/censored}
We further generalize our framework to accommodate settings with more restrictive feedback, enabling us to model more complex and realistic decision-making problems. Unlike the standard \emph{semi-bandit} setting discussed  so far, the algorithm here does not directly observe the realization $x_i$ of a  probed item $i$. 
We consider two  feedback models: censored and binary.
The first  model,  called \emph{censored feedback}, assumes that the algorithm  only observes the realization of each random variable up to  a chosen ``threshold''. The second model, termed \emph{binary feedback}, is even more limited: the algorithm only receives  binary information indicating whether/not the  random variable's  realization is  above  its chosen threshold.

We also focus on a restricted class of  ``threshold based'' stochastic problems. We assume that all random variables are  discrete with a finite support set   $\{a_1,\ldots,  a_k\}$.
A solution/policy for a threshold-based stochastic problem
is given by a decision tree $\sigma$ where each internal node is labeled   by a pair $(c, i)$ where $i\in [n]$ is an item and $c\in [k]$ corresponds to a threshold value. When the policy $\sigma$ reaches  node $(c, i)$, it probes item $i$ {\em   with threshold} $a_c$ and observes the following  partial  realization of $X_i$ (depending on the feedback model). 

\begin{itemize}
    \item In  \emph{censored feedback}, the policy observes the truncated realization of $\min\{\varX_i, a_c\}$ and the   branches out  node $(c,i)$ correspond to this observation.

    \item In \emph{binary feedback}, the policy observes the indicator $\1_{\varX_i \geq a_c}$ and the   branches out  node $(c,i)$ correspond to this observation.
    \end{itemize}    
Given a policy $\sigma$ and realization $\bX$, we  use  $S(\sigma, \bX)$ to denote   the  sequence of item-threshold pairs that are probed. 
We also make the following assumption on the objective function:
\begin{assumption}
\label{asmp:censored-f} Given any threshold policy $\sigma$, the function value at each leaf node of $\sigma$ is a constant (which may differ across leaf nodes).
\end{assumption}

We note that this assumption is more restrictive compared to the previous results (Theorems~\ref{thm:main} and \ref{thm:main-cont}).
This is because at any  leaf node of policy $\sigma$ we may not know the precise realizations of probed items. Nevertheless, there are some interesting problems   that satisfy this assumption (see below).

In the online setting, at time $t\in [T]$, if the  algorithm implements  policy $\sigma^t$ and the realizations are $\bx^t\sim \bD$ then  it observes the tuple
    \[
    \left\{\min(x_{i}^t ,a_c)\, :\, (c,i)\in S(\sigma^t, \bx^t)\right\}
    \]
    under censored feedback, and the tuple
       \[
    \left\{\1_{x_i^t \geq a_c} \,:\, (c,i)\in S( \sigma^t, \bx^t)\right\}
    \]
    under binary feedback. The goal is to minimize $\alpha$-regret as before.

\paragraph{Censored feedback application: Single Resource Revenue Management}
A classical instance of \emph{censored feedback} arises in the single resource revenue management setting~\cite[4--6]{GT19}. Consider an airline with $C$ units of seat capacity available for sale across multiple fare classes. The  customer demand for each fare class is  an  independent random variable.
We allocate portions of the capacity to different fare classes. When we choose to offer $m$ seats to a specific fare class, we observe only the number of tickets sold---namely, the minimum of the demand and $m$. The airline's revenue (which is the objective function) also depends only on the number sold. This corresponds to censored feedback.

\paragraph{Binary feedback application: Sequential Posted Pricing (SPM).  } Consider a seller who wants to sell a single item to a group of $n$ buyers who have random valuations. The buyers arrive sequentially, and the seller presents each buyer $i$ with a take-it-or-leave-it price $p_i$. The first buyer $i$ whose value exceeds the posted price $p_i$ buys the item and the seller receives $p_i$ dollars as revenue. Although the seller observes whether/not a particular buyer buys the item,  the exact  valuations remain  unknown.  This corresponds to binary feedback.

\medskip

Naturally, one might attempt to incorporate  censored  and binary feedback   into the \emph{semi-bandit} framework by introducing additional independent copies of each random variable (one for each threshold). However, this approach can be problematic for several reasons. First, it may significantly increase the complexity of the original problem, making it difficult to find a good approximation result even in the offline setting with known distributions. Second, it can fail to preserve the monotone property of the original problem, which is crucial for our analysis.

Instead, we show how to  extend our techniques from prior sections to the new \emph{censored/binary} feedback settings.
Our main results here are the following.

\begin{theorem}[Censored Feedback]\label{thm:main for censored}
    Consider a stochastic problem $\mathcal{P}$ that is either up-monotone or down-monotone, and satisfies Assumption~\ref{asmp:censored-f}.     
    Suppose that $\mathcal{P}$ has an $\alpha$-approximation algorithm via threshold policies. 
        Then, there is a polynomial time censored feedback learning algorithm for $\mathcal{P}$ (with unknown distributions) that has $\alpha$-regret  $O(n k f_{\max } \sqrt{kT\log (kn T)})$. Here, $n$ is the number of items, $k$ is the maximum support size, $f_{\max }$ is the maximal value of the objective function  and $T$ is the number of time periods.
\end{theorem}

\begin{theorem}[Binary Feedback]\label{thm:main for binary} Consider a stochastic problem $\mathcal{P}$ that is either up-monotone or down-monotone, and satisfies Assumption~\ref{asmp:censored-f}.     
    Suppose that $\mathcal{P}$ has an $\alpha$-approximation algorithm via threshold policies. 
    Then, there is a polynomial time binary feedback learning algorithm for $\mathcal{P}$ (with unknown distributions) that has $\alpha$-regret  $O(n f_{\max} \sqrt{kT\log(knT)})$. Here, $n$ is the number of items, $k$ is the maximum support size, $f_{\max }$ is the maximal value of the objective function  and $T$ is the number of time periods.
\end{theorem}

The high-level approach remains the same as in the  semi-bandit   setting. In the following two subsections, we establish the key sampling theorems and stability lemmas tailored to the new feedback models. We primarily focus on  minimization up-monotone  problems under the censored feedback setting; the same ideas apply to the binary-feedback setting as well. As with semi-bandit feedback, results for all the other combinations (minimization/maximization objective and up/down monotonicity) follow in an identical manner. The primary technical contribution  is  the sampling theorem, which   addresses the challenge posed by limited observations.

\subsection{New Sampling Algorithms}\label{sec:sampling theorem}

In this subsection, we focus on up-monotone stochastic problems. The analysis for down-monotone stochastic problems is similar. We will first present the sampling theorem for the censored feedback setting. Then, we  provide a similar result for the binary feedback case. The key challenge in these settings is that, due to the limited nature of censored/binary feedback, we cannot directly control the total variation distance between the empirical and the true underlying distribution. Instead, we design an algorithm that controls the total variation distance at different threshold levels.

\begin{definition}[Truncated Distribution]\label{def:truncated distribution}
   Given a discrete r.v.  $\varX\sim \calD$ with support set $ \{a_1, \ldots, a_k\}$, the truncated distribution $\calD_c$ is the distribution of r.v. $\varX_c:=\min(\varX,a_c)$. It is supported on values  $\{a_1 , \ldots, a_c\}$, with   $\Pr[\varX_c = v]=\Pr[\varX=v]$  for all $v< a_c$  and  $\Pr[\varX_c = a_c]=\Pr[\varX\ge a_c]$. 
\end{definition}

\begin{theorem}[Sampling under Censored Feedback]
\label{thm:emp-stoch-dom-censored}
Consider any r.v.  $\varX\sim \calD$ with support set $ \{a_1, \ldots, a_k\}$. There is an efficient algorithm  that, given $m_c$ i.i.d. ``censored'' samples of the form $\min(\varX, a_c)$ for each $c \in [k]$,  computes a distribution $\calE$  satisfying the following properties with probability at least $1-\delta$.
\begin{itemize}
\item $\calE$ stochastically dominates $\calD$.
\item For each $c \in [k]$,  the total-variation distance $\TVD(\calE_c, \calD_c) < k \sqrt{\frac{2\log(2k/\delta)}{m_c}}$. 
\end{itemize}  
\end{theorem}

\begin{proof} We first note that when a sample of $\varX_b := \min(\varX, a_b)$ is obtained, it also provides a sample of $\varX_c$ for all $c \leq b$. Consequently, we define
\(
n_c := \sum_{b \geq c} m_b
\)
as the effective number of samples  at level $c$. For any $c\in [k]$, we also set its confidence width as
\(
\epsilon_c := \sqrt{\frac{\log(2k/\delta)}{2n_c}}.
\)
Note that $\epsilon_1 \leq \cdots \leq \epsilon_k$.

For each $c \in [k]$, let $\widehat{\calD}_c$ denote the empirical distribution of the $n_c$  samples of $\varX_c$, and define its probability mass function (pmf) as 
$$\hat{p}_c(b) := \Pr_{\varY\sim \widehat{\calD}_c}[\varY=a_b],\quad \forall b\le c.$$
Let $q_c:=\Pr_{\varX \sim \calD}[\varX \geq a_c]$ for $c\in [k]$ denote the tail cdf of the true distribution $\calD$.
By Hoeffding’s inequality~\citep{hoeffding1963probability}, for each $c \in [k]$, we have:
\[
\Pr\left[ \left| \hat{p}_c(c) - q_c \right| > \epsilon_c \right] < \frac{\delta}{k}.
\]
Combined with a union bound,  the following ``good'' event  holds w.p. at least $1-\delta$.
\begin{equation}
    \label{eq:censored-good-E}
|\hp_c(c) - q_c| \le \epsilon_c , \quad \forall c\in [k].\end{equation}
We will assume this good event in the rest of the proof. Our algorithm to  construct the desired  distribution $\calE$ modifies the (truncated) empirical distributions $\widehat{\calD_c}$  by  increasing $c=1,2,\ldots, k$. For each $c \in [k]$, we
construct a distribution $\calE_c$ supported on $\{a_1, \ldots, a_c\}$ that stochastically dominates $\calD_c$; we construct $\calE_{c+1}$ from $\calE_c$ by moving some mass from $a_c$ to $a_{c+1}$. Below, the pmf of $\calE_c$ is denoted by    $\calE_c(b) = \Pr_{\varY\sim  {\calE}_c}[\varY=a_b]$ for $b\le c$.

\begin{algorithm}
\caption{Constructing Empirical Modified Distribution $\mathcal{E}$  \label{alg:censored calE construction} }
\begin{algorithmic}[1]
\State  Set $\calE_1(1) = 1$
\For{$c = 1,2,\cdots k-1$}
    \State \label{step:cens-sampling-0} Set $\calE_{c+1}(b) \gets \calE_c(b)$ for $b =1, \ldots,  c-1$     
    \State \label{step:cens-sampling-1} $\calE_{c+1}(c+1) \gets \min\{\calE_c(c), \hat{p}_{c+1}(c+1) + \epsilon_{c+1} \}$
    \State \label{step:cens-sampling-2} $\calE_{c+1}(c) \gets \calE_c(c) - \calE_{c+1}(c+1)$
\EndFor
\State \textbf{Output:} $\calE = \calE_k$
\end{algorithmic}
\end{algorithm}

We complete the proof of the theorem by proving the following claims via induction.

\begin{claim}\label{valid distribution}
For each $c\in[k]$,   $\calE_c$ is a valid probability distribution.
\end{claim}
\begin{proof}
The base case ($c = 1$) is trivial by construction:  $\calE_1$ is a point mass  distribution. We will now show that $\calE_{c+1}$ is a valid probability distribution assuming that $\calE_{c}$ is. It is easy to see that the pmf $\calE_{c+1}(\cdot)$ is non-negative by Steps~$3$-$5$ in \Cref{alg:censored calE construction}. Further,  
\[
\sum_{b=1}^{c+1} \calE_{c+1}(b) 
= \sum_{b=1}^{c-1} \calE_c(b) + \calE_{c+1}(c) + \calE_{c+1}(c+1)
= \sum_{b=1}^{c-1} \calE_c(b) + \calE_c(c) = 1,
\]
where the first equality is by Step~\ref{step:cens-sampling-0}, the second equality is by Step~\ref{step:cens-sampling-2} and the last equality uses that  $\calE_c$ is a valid distribution. Therefore, $\calE_{c+1}$ is also a valid distribution.
\end{proof}

\begin{claim}\label{sto dominate}
Under the good event~\eqref{eq:censored-good-E}, for each $c\in[k]$, $\calE_c$ stochastically dominates $\calD_c$ .
\end{claim}
\begin{proof}
The base case ($c = 1$) is trivial as both $\calD_1$ and $\calE_1$ are  point masses with value $a_1$. We will show that  $\calD_{c+1} \sd \calE_{c+1}$  assuming  $\calD_c \sd \calE_c$. 
It suffices to show that 
\begin{equation}
    \label{eq:censored-claim-2}
\sum_{j=b}^{c+1} \calE_{c+1}(j) \ge q_b,\quad \forall b=1, \ldots, c+1.
\end{equation}
For $b \le c$, we have
\[
\sum_{j=b}^{c+1} \calE_{c+1}(j) = \sum_{j=b}^{c-1} \calE_c(j) + \calE_{c+1}(c) + \calE_{c+1}(c+1) = \sum_{j=b}^{c-1} \calE_c(j) + \calE_c(c) = \sum_{j=b}^c \calE_c(j) \ge q_b.
\]
The first two equalities are by Steps~$3$-$5$ in \Cref{alg:censored calE construction}, and the inequality is by $\calD_c \sd \calE_c$. It remains to  prove \eqref{eq:censored-claim-2} for  $b = c+1$. By the good event~\eqref{eq:censored-good-E},  we have $q_{c+1} \le \hp_{c+1}(c+1) + \epsilon_{c+1}$. Moreover, by $\calD_c \sd \calE_c$ we have $\calE_c(c) \ge q_c \ge q_{c+1}$. So, by Step~5 in \Cref{alg:censored calE construction}, we have
\[
\calE_{c+1}(c+1)  = \min\left\{ \calE_c(c), \hp_{c+1}(c+1) + \epsilon_{c+1} \right\}\ge q_{c+1},
\]
which completes the proof of \eqref{eq:censored-claim-2}, and thus the claim. 
\end{proof}

\begin{claim}\label{claim: TV dist}
Under the good event~\eqref{eq:censored-good-E}, for each $c \in[k]$, we have
$\TV(\calE_c , \calD_c) \le k \cdot \sqrt{\frac{2 \log(2k/\delta)}{n_c}}$.
\end{claim}
\begin{proof}
Fix any $c\in [k]$. We first claim that

\begin{equation}\label{eq: TV dist}
   \sum_{ j= b}^c \calE_c(j) = \calE_b(b) \le \hp_b(b) +\epsilon_b \le q_b +2\epsilon_b,\qquad \forall b\le c.
\end{equation}
The equality $\sum_{j= b}^c \calE_c(j) = \calE_b(b)$ for $b\le c$ is by  construction of the distributions $\calE_c$s: we only move mass from low to high support values.  
The inequality is by Step~4 of \Cref{alg:censored calE construction}, and the last inequality is by the good event \eqref{eq:censored-good-E}. Now, for  any $b \le c-1$, we have
\begin{align}
\left|\calE_c(b) - \Pr(\varX = a_b)\right| &= \left|\sum_{ j= b}^c \calE_c(j) - \sum_{ j= b+1}^c \calE_c(j) - q_b + q_ {b+1} \right| \notag \\
&\leq \left|\sum_{ j= b}^c   \calE_c(j)- q_b\right| +  \left|\sum_{ j= b+1}^c \calE_c(j) - q_{b+1} \right| \leq  2\epsilon_b + 2\epsilon_{b+1} \le 4 \epsilon_{c}. \label{eq:censored-tv}      
\end{align}
The second inequality uses \eqref{eq: TV dist} and the last inequality uses   that $\epsilon_i \leq \epsilon_j$ for all $i \leq j$.

We can now bound the TV distance: 
$$
 \TVD (\calE_c, \calD_c) = \frac{1}{2} \sum_{b=1}^{c-1} |\calE_c(b) - \Pr(\varX = a_b)| + \frac{1}{2} |\calE_c(c) -q_c| \le 2(c-1) \epsilon_c  + \epsilon_c \le k \sqrt{\frac{2\log(2k/\delta)}{n_c}}.
$$
The first inequality uses \eqref{eq:censored-tv} for $b\le c-1$ and \eqref{eq: TV dist} for $c$.  This implies the TV bound in the theorem because   $n_c \geq m_c$.
Finally,  observe that when the output distribution $\calE=\calE_k$ is capped by $a_c$ we obtain the distribution $\calE_c$. 
\end{proof}
Combining \Cref{valid distribution}, \Cref{sto dominate}, and \Cref{claim: TV dist} completes the proof of Theorem~\ref{thm:emp-stoch-dom-censored}. 
\end{proof}
The sampling algorithm and analysis can be extended to binary feedback in a straightforward way. We now describe the changes needed.

\begin{definition}[Binary Compressed Distribution]\label{def:binary feedback} Given a discrete r.v.~$\varX\sim \calD$ with support set $ \{a_1, \ldots, a_k\}$, the binary compressed distribution $\calD_c^B$ is the distribution of r.v.~$\varX_c^B := \1_{\varX \geq a_c}$. We have $\Pr(\varX_c^B = 0) = \Pr(\varX < a_c)$  and  $\Pr(\varX_c^B=1) = \Pr(\varX \ge a_c)$.     
\end{definition}

\begin{theorem}[Sampling for Binary Feedback]
\label{thm:emp-stoch-dom-binary}
Consider any r.v.~$\varX \sim \calD$ with support set $ \{a_1, \ldots, a_k\}$. 
There is an efficient algorithm  that, given $m_c$ i.i.d.~``binary'' samples of the form $\1_{\varX \geq a_c}$ for each $c \in [k]$,  computes a distribution $\calE$  satisfying the following properties with probability at least $1-\delta$.
\begin{itemize}
\item $\calE$ stochastically dominates $\calD$.
\item For $c \in [k]$,  the total-variation distance $\TVD(
\calE_c^B, \calD_c^B) <  \sqrt{\frac{2\log(2k/\delta)}{m_c}}$. 
\end{itemize}  
\end{theorem}
\begin{proof}  For any $c\in [k]$, we denote by  $\widehat{\calD_c}$  the empirical distribution of the $m_c$ samples of $\1_{\varX\geq a_c}$; we use  $\hat{p}_c(c) = \Pr_{\varY\sim \widehat{\calD_c}}[\varY=1]$ to keep notation the same as for the censored case.
We also define the confidence width $\epsilon_c :=  \sqrt{\frac{\log(2k/\delta)}{2m_c}}$. 
As before, let $q_c:=\Pr_{\varX\sim \calD}[\varX \geq a_c]$ for $c\in [k]$ denote the tail cdf of $\calD$.
By Hoeffding’s inequality and union bound,  the following ``good'' event  holds w.p.~at least $1-\delta$.
\begin{equation}
    \label{eq:binary-good-E}
|\hp_c(c) - q_c| \le \epsilon_c , \quad \forall c\in [k].\end{equation}

Given these values  $\hat{p}_c(c)$, the sampling algorithm  to construct $\calE$ remains the same as the censored case (Algorithm~\ref{alg:censored calE construction}). 
Let $\calE_c^B$ denote the intermediate distributions: exactly as before, for each $c\in [k]$, $\calE_c^B$ is a valid distribution and $\calD_c^B \sd \calE_c^B$ (under the good event). Inequality~\eqref{eq: TV dist} also continues to hold. We can now   bound the TV distance between $\calE_c^B$ and $\calD_c^B$ for any $c\in[k]$.
\begin{align*}
 \TVD (\calE_c^B, \calD_c^B) &= |\calE_c^B(1)- \Pr(\varX_c^B = 1 ) | =  |\calE_c^B(1)- \Pr(\varX\geq a_c )|  \\
 &=  |\calE_c^B(1)- q_c | = | \sum_{ j= c}^k \calE_k(j) -q_c| \le 2 \epsilon_c = \sqrt{\frac{2\log(2k/\delta)}{m_c}}.    
\end{align*}
where the   inequality is by \eqref{eq: TV dist}. 
This completes the proof of Theorem~\ref{thm:emp-stoch-dom-binary}. 
\end{proof}

\subsection{Stability Lemma}\label{sec:stable lemma}
We now describe the modifications to the stability lemma under censored feedback. 
\begin{lemma}[Censored Stability Lemma]\label{lem: censored stable}
Consider a threshold-based stochastic problem that is up-monotone and satisfies Assumption~\ref{asmp:censored-f}. Suppose that $\bE=\{\calE_i\}_{i=1}^n$ and $\bD=\{\calD_i\}_{i=1}^n$ are product distributions  
such that $\bD\sd \bE$, where each distribution $\{\calD_i, \calE_i\}_{i=1}^n$ is supported on   $\{a_{1}, \ldots, a_{k}\}$. Further, suppose that $\TV(\calE_{c,i}, \calD_{c,i}) \leq \epsilon_{c,i}$ for each $i \in [n]$ and  $c \in [k]$, where   $\calD_{c,i}$ $(\text{resp. }  \calE_{c,i})$ is   the distribution $\calD_i$ $(\text{resp. } \calE_i)$ truncated at $a_{c}$. 
If $\sigma$ is the policy returned by $\ALG(\bE)$  
and $\sigma^*$ is an optimal  policy under $\bD$, then
\begin{equation}\nonumber
f(\sigma) - \alpha\cdot  f(\sigma^*) = \E_{\bx \sim \bD}\left[f(\sigma, \bx) - \alpha\cdot f(\sigma^*, \bx)\right] \leq    f_{\max} \sum_{i=1}^n \sum_{c=1}^k {Q}_{c,i}(\sigma) \cdot \epsilon_{c,i} 
    \,,
\end{equation} 
where ${Q}_{c,i}(\sigma)$ is the probability $(\text{under } \bD)$ that item $i$ is probed by policy $\sigma$ with threshold $a_{c}$.
\end{lemma}

\begin{proof}

The  analysis is almost the same as for the  semi-bandit setting  (Lemma~\ref{lem:stable}). The main modification is in  labeling the nodes of the   decision tree. 
Each node $v$ in the decision tree $\sigma$ now corresponds to a pair $(c,i)$, indicating that item $i$ is probed with threshold $a_c$. For any item $i$, in any policy execution  we will encounter at most one node labeled by item $i$ and any threshold. So, we can   view policy $\sigma$ as having an item $\mathrm{X}_v$ with independent distribution $\mathcal{D}_v \stackrel{d}{=} \calD_{c,i}$ at each node $v \in \sigma$ labeled by $(c,i)$; here ``$\stackrel{d}{=}$'' means equal in distribution. Similarly, we use $\mathcal{E}_v \stackrel{d}{=} \calE_{c,i}$ for $v$ labeled by $(c,i)$.
Exactly as in the proof of Lemma~\ref{lem:stable}, we index nodes in $\sigma$ according to   the  ancestor-descendant partial order and 
use hybrid product distributions $\mathbf{H}^v$. Again, we will show:  
\begin{equation}\label{eq:censored-stab-lemma-key}
\left| f\left(\sigma | \mathbf{H}^v\right) - f\left(\sigma | \mathbf{H}^{v-1}\right) \right| \leq f_{\max} \cdot Q_v(\sigma) \cdot \epsilon_v,
\end{equation}
where $Q_v(\sigma)$ is the probability that policy $\sigma$ reaches node $v$ under distribution $\mathbf{D}$ and $\epsilon_v = \TV(\calD_v, \calE_v) =\TV(\calE_{c,i}, \calD_{c,i}) \leq \epsilon_{c,i} $. We now complete the proof of the lemma using \eqref{eq:censored-stab-lemma-key}. We have
\begin{align*}
\left| f(\sigma | \mathbf{D}) - f(\sigma | \mathbf{E}) \right|
&\leq \sum_{v=1}^N \left| f\left(\sigma | \mathbf{H}^v\right) - f\left(\sigma | \mathbf{H}^{v-1}\right) \right| \\
&\leq f_{\max} \cdot \sum_{v=1}^N Q_v(\sigma) \cdot \epsilon_v \le \sum_{i=1}^n \sum_{c=1}^k \epsilon_{c,i} \sum_{v:\text{ labeled by } (c,i)} Q_v(\sigma)=\sum_{i=1}^n \sum_{c=1}^k \epsilon_{c,i}\cdot  Q_{c,i}(\sigma),
\end{align*}
where the final equality uses the fact that $Q_{c,i}(\sigma)$   equals the total probability of reaching some node $v$ labeled by $(c,i)$. This suffices to prove the lemma, exactly as in Lemma~\ref{lem:stable}.

The  proof of  \eqref{eq:censored-stab-lemma-key} is almost identical to that of \eqref{eq:stab-lemma-key} in Lemma~\ref{lem:stable}. Letting ${\cal R}_v$ be the event that $\sigma$  reaches node $v$, it suffices to show:
\begin{equation*}
\left| \mathbb{E}_{\mathbf{x} \sim \mathbf{H}^v} \left[ f(\sigma, \mathbf{x}) \mid \mathcal{R}_v \right] - \mathbb{E}_{\mathbf{x} \sim \mathbf{H}^{v-1}} \left[ f(\sigma, \mathbf{x}) \mid \mathcal{R}_v \right] \right| \leq f_{\max} \cdot \TV(\calE_{v}, \calD_{v}),
\end{equation*}
The proof is identical to that of \eqref{eq:stab-lemma-conditional-2}, where we use the assumption that the function value at any leaf-node of $\sigma$ is deterministic (even though we don't know the exact realizations of the probed r.v.s).

\end{proof}

By an identical proof, using distributions $\calD_v \stackrel{d}{=} \calD_{c,i}^B$ and  $\calE_v \stackrel{d}{=} \calE_{c,i}^B$ at each node $v=(c,i)$, we obtain the following under binary feedback. 
\begin{lemma}[Binary Stability Lemma]\label{lem: binary stable}
Consider a threshold-based stochastic problem that is up-monotone and satisfies Assumption~\ref{asmp:censored-f}. Suppose that $\bE=\{\calE_i\}_{i=1}^n$ and $\bD=\{\calD_i\}_{i=1}^n$ are product distributions  
such that $\bD\sd \bE$, where each distribution $\{\calD_i, \calE_i\}_{i=1}^n$ is supported on   $\{a_{1}, \ldots, a_{k}\}$. 
Further, suppose that $\TV(\calE_{c,i}^B, \calD_{c,i}^B) \leq \epsilon_{c,i}$ for each $i \in [n]$ and  $c \in [k]$, where   $\calD_{c,i}^B$ $(\text{resp. } \calE_{c,i}^B)$ is the distribution $\calD_i$ $(\text{resp. } \calE_i)$ compressed at $a_{c}$.  
If $\sigma$ is the policy returned by $\ALG(\bE)$  
and $\sigma^*$ is an optimal  policy under $\bD$, then
\begin{equation}\nonumber
f(\sigma) - \alpha\cdot  f(\sigma^*) = \E_{\bx \sim \bD}\left[f(\sigma, \bx) - \alpha\cdot f(\sigma^*, \bx)\right] \leq    f_{\max} \sum_{i=1}^n \sum_{c=1}^k {Q}_{c,i}(\sigma) \cdot \epsilon_{c,i} 
    \,,
\end{equation} 
where ${Q}_{c,i}(\sigma)$ is the probability $(\text{under } \bD)$ that item $i$ is probed by policy $\sigma$ with threshold $a_{c}$. 
\end{lemma}

\subsection{Overall Regret}\label{sec: main regret}
We now prove the main theorem for the censored feedback setting (Theorem~\ref{thm:main for censored}). The overall structure of the proof mirrors that of the semi-bandit setting: we combine a sampling theorem with the stability lemma to bound the regret. 

Recall that, $\calD_i$ is the true distribution of  item $i$ and $\calE^t_i$ is the  dominating distribution generated by Algorithm~\ref{alg:censored calE construction} in time step $t$. For any $i\in [n]$, $c\in [k]$, let $N_{c,i}^t$ be the number times the algorithm  sampled   $\min\{\varX_i, a_{c}\}$ before time  $t$.

\begin{lemma}\label{lem :union bound for censored}
With probability at least $1-\frac{1}{nT}$, we have $\calD_i \sd \calE^t_i $ and the total variation distance $\TV(\calE_{c,i}^t, \calD_{c,i}) \leq k \sqrt{\frac{6\log(knT)}{N_{c,i}^t}}$ for $c \in [k]$, $i \in [n]$, and $t \in [T]$.
\end{lemma}

\begin{proof}
    For any $c \in [k]$, $i \in [n]$, and $t \in [T]$, let $B_{c,i}^t$ denote the event that the stated condition fails to hold for the triple $(c, i, t)$. By Theorem~\ref{thm:emp-stoch-dom-censored} with $\delta = \frac{2}{k^2 n^3 T^3}$, we have:

\[
\Pr\left(B_{c,i}^t\right) \leq \sum_{m=1}^T \Pr\left(B_{c,i}^t \wedge N_{c,i}^t = m\right) \leq \delta T.
\]

Applying another union bound over all values of $c$, $i$, and $t$, we obtain:

\[
\Pr\left( \vee_{c=1}^k \vee_{i=1}^n \vee_{t=1}^T B_{c,i}^t \right)
\leq \sum_{c=1}^k \sum_{i=1}^n \sum_{t=1}^T \Pr\left(B_{c,i}^t\right)
\leq k n T^2 \delta \leq \frac{1}{nT},
\]

where the last inequality uses the assumption that $n \geq 2$.  
\end{proof}

We now define a good event $G$ that corresponds to the condition in Lemma~\ref{lem :union bound for censored} holding true  for all $c \in[k]$, $i\in[n]$ and $t \in[T]$. Next, we complete the proof assuming that $G$ holds.
\paragraph{Bounding the regret as a sum over time $t$.} To bound the overall regret,  it suffices to bound the expected regret $R^t=f(\sigma^t) - \alpha\cdot f(\sigma^*)$ at each time $t\in [T]$.  
For each time $t\in [T]$, using the good event, we can apply Lemma~\ref{lem: censored stable} with distributions $\bE^t$, $\bD$,  and parameters $\epsilon_{c,i}^t   = k\cdot \sqrt{\frac{6\log(knT)}{N_{c,i}^t}}$ to obtain 
\begin{equation}
\label{eq:censored-regret-t-stable}
R^t   \le  f_{\max} \sum_{i=1}^n \sum_{c=1}^k {Q}_{c,i}(\sigma^t) \cdot \epsilon_{c,i}^t .
\end{equation}

So, the overall regret is
$$\alpha\text{-}R(T) \le  f_{\max} \sum_{t=1}^T \E_{\h^{t-1}}\left[  \sum_{i=1}^n \sum_{c=1}^k {Q}_{c,i}(\sigma^t) \cdot \epsilon_{c,i}^t 
 \right] =  k f_{\max} \sqrt{6 \log(knT)} \cdot \sum_{t=1}^T \E_{\h^{t-1}}\left[  \sum_{i=1}^n \sum_{c=1}^k \frac{{Q}_{c,i}(\sigma^t) }{\sqrt{N^t_{c,i}}}  \right] . $$
      Recall that  $\h^{t-1}=(\bx^1,\cdots \bx^{t-1})$ is the history until time $t$. 
It now suffices to show
\begin{equation}\label{eq:censored-path-regret-sum}
\sum_{t=1}^T \E_{\h^{t-1}}\left[  \sum_{i=1}^n \sum_{c=1}^k \frac{{Q}_{c,i}(\sigma^t) }{\sqrt{N^t_{c,i}}}  \right] \le 2n\sqrt{kT}.
\end{equation}

\paragraph{Proving \eqref{eq:path-regret-sum} as a sum over decision paths.} Similar to the semi-bandit case, we define:
\begin{equation*}
Z_i^c(\h^T) := \sum_{t=1}^T \frac{\1\left[\text{ $i$ probed by policy }\sigma^t(\h^{t-1})\text{ with  threshold $a_{c}$ }\right]}{\sqrt{N^t_{c,i}(\h^{t-1})}}, \quad \forall i\in [n] \, , c \in [k]. 
\end{equation*}
Above, $\1$ is the indicator function.
By linearity of expectation,  we have
{\small \begin{align}
\E_{\h^T} \left[ Z_i^c(\h^T) \right] &= \sum_{t=1}^T \E_{\h^T} \left[ \frac{\1\left[i \text{  probed by  }\sigma^t(\h^{t-1}) \text{ with } a_{c} \right]}{\sqrt{N^t_{c,i}(\h^{t-1})}}\right] \,\,=\,\,\sum_{t=1}^T \E_{\h^{t-1}, \bx^t} \left[ \frac{\1\left[i \text{  probed by  }\sigma^t(\h^{t-1}) \text{ with } a_{c} \right]}{\sqrt{N^t_{c,i}(\h^{t-1})}}\right] \notag \\
& =\sum_{t=1}^T \E_{\h^{t-1}} \left[ \frac{1}{\sqrt{N^t_{c,i}}} \cdot \Pr_{\bx^t} [i \text{  probed by  }\sigma^t \text{ with } a_{c}]\right] \,\, =\,\, \sum_{t=1}^T \E_{\h^{t-1}} \left[ \frac{Q_{c,i}(\sigma^t)}{\sqrt{N^t_{c,i}}} \right],\label{eq:censored-Zi}
\end{align}}
where the second equality uses the fact that event \{$i$   probed by  $\sigma^t$ with $c$\}  only depends on $\h^t = (\h^{t-1}, \bx^t)$, the 
third equality uses the fact that $N^t_{c,i}$ only depends on $\h^{t-1}$ and that $\bx^t$ is independent of $\h^{t-1}$, and the last equality is by the definition of $Q_{c,i}(\sigma^t)$ and the fact that $\bx^t \sim \bD$. 

Using \eqref{eq:censored-Zi} and adding over $i\in [n]$ and $c\in [k]$, we get
$$\sum_{t=1}^T \sum_{c=1}^k \sum_{i=1}^n  \E_{\h^{t-1}}\left[   \frac{{Q}_{c,i}(\sigma^t) }{\sqrt{N^t_i}}  \right] = \sum_{i=1}^n \sum_{c=1}^k  \E_{\h^T} \left[ Z_i^c(\h^T) \right].$$

Therefore, proving \eqref{eq:censored-path-regret-sum} is equivalent to:
\begin{equation}
\sum_{i=1}^n \sum_{c=1}^k  \E_{\h^T}[Z_i^c(\h^T)] \le 2n\sqrt{kT}\label{eq:censored-path-regret-sum-2}.
\end{equation}

Now, for any $i \in [n]$ and $c \in[k]$, we obtain:
$$Z_i^c(\h^T) = \sum_{t=1}^T \frac{\1\left[i\text{ probed by }\sigma^t(\h^{t-1}) \text{ with } a_c \right]}{\sqrt{N^t_{c,i}(\h^{t-1})}} \leq \sum_{t=1}^{N^T_{c,i}} \frac{1}{\sqrt{t}} \leq 2 \sqrt{N^T_{c,i}}.$$
The first inequality uses the fact that if item $i$ gets probed with threshold $a_{c}$ at time $t$ then $N_{c,i}^t =N_{c,i}^{t-1} +1 $. Using Jensen's inequality and the fact that the total number of probes for one item $i$ is at most $T$, we have:
$$
\sum_{c=1}^k  \sqrt{N^T_{c,i}} \leq  k \sqrt{\frac{\sum_{c=1}^k N^T_{c,i}}{k}} \le \sqrt{kT}.
$$
This completes the proof of \eqref{eq:censored-path-regret-sum-2} and hence \eqref{eq:censored-path-regret-sum}.

We note that our good event $G$ holds with probability at least $1 - \frac{1}{nT}$. As  in the semi-bandit setting, we can convert the high-probability regret bound into an expected regret bound. This concludes the proof of Theorem~\ref{thm:main for censored}.

\medskip
The proof of Theorem~\ref{thm:main for binary} is identical: we just need to use the appropriate sampling result (Theorem~\ref{thm:emp-stoch-dom-binary}) and stability lemma (Lemma~\ref{lem: binary stable}).

\newcommand{\TO}{\tilde{O}}

\section{Applications}\label{sec:apps}
In this section, we show that several stochastic optimization problems are covered by our framework, resulting in $\sqrt{T\log T}$ regret online learning algorithm for all these problems. In each of these problems, the distributions of the random parameters are unknown to the online algorithm;   all other (deterministic) parameters are known.
We summarize the overall guarantees in \Cref{table:apps}, highlighting the dependence on $n$, $T$, the support size $k$ of the underlying distributions (where applicable), and the distributional assumptions needed. 
See the corresponding subsection for more details and exact bounds.

\begin{table}[h!]
    \centering
    \begin{tabular}{|c|c|c|c|}
        \hline
        \multirow{2}{*}{Application} & Approx. & Our $\alpha$-Regret   & Distributional   \\
         & Factor ($\alpha$) & Bound &  Assumptions  \\
         \hline 
         Series Testing & $1$ &$\TO(n\sqrt{T})$  & Bernoulli \\
         \hline
         Prophet Inequality & $1$ & \multirow{3}{*}{$\TO(n\sqrt{T})$}  & \multirow{3}{*}{continuous}  \\
         Multi-Item Prophet Inequality & $1-\frac{1}{\sqrt{\kappa + 3}}$ &   & \\
         Matroid Prophet Inequality & $1/2$ &  & \\
         \hline
         Pandora's Box & $1$ & $\TO(n\sqrt{T})$ & continuous \\
         \hline
         \multirow{1}{*}{Stochastic Knapsack} & $1/2-\epsilon$ & $\TO(n\sqrt{kT})$ & discrete \\

         Stochastic Orienteering & $\Omega(1/\log\log B)$  & $\TO(n\sqrt{kT})$  & discrete \\
         \hline
         Unweighted Stochastic Matching & $1/2$ & \multirow{2}{*}{$\TO(n\sqrt{T})$} & \multirow{2}{*}{Bernoulli} \\
         Weighted Stochastic Matching & $0.382$ & & \\
         \hline
         Stochastic Covering Knapsack & $3$ & $\TO(n \sqrt{kT})$ & discrete \\
         Stochastic $Q$-TSP & $O(1)$ & $\TO(n \sqrt{kT})$ & discrete \\
         \hline
         Stochastic Submodular Maximization & $1 - \frac{1}{e}$ & \multirow{2}{*}{$\TO(n\sqrt{T})$}  & \multirow{2}{*}{Bernoulli} \\
         Stochastic Submodular Cover & $1+\ln(Q)$ & & \\
         \hline
         Single Resource Revenue Management & $1$ & $\TO(nk^2\sqrt{kT})$ & discrete \\
         \hline
          FSPM\textsuperscript{\textdagger}: Cardinality Constrained & $1$ & \multirow{3}{*}{$\TO(n\sqrt{kT})$} & \multirow{3}{*}{discrete} \\
         ASPM$^*$: Cardinality Constrained & $1-\epsilon$ & & \\
         ASPM: Matroid Constrained  & $1-\frac{1}{e}$ & &  \\
         \hline
    \end{tabular}
    \caption{Summary of Regret Bounds; \textsuperscript{\textdagger} Fixed-Order Sequential Posted-Pricing Mechanisms, $^*$  Adaptive Sequential Posted-Pricing Mechanisms}
    \label{table:apps}
\end{table}

\subsection{Series Testing} \label{subsec:series}
We start with a simple  problem:   there are 
$n$ components, where each component $i$ is ``working'' independently with some known probability $p_i$. To determine if any component $i\in [n]$ is working, we need to perform a test, which costs $c_i$. All $n$ components must be working for the system to be functional. 
The goal is to test components sequentially to determine whether/not the system is functional, at the minimum expected cost. Note that testing stops once a failed component is found: so we do not observe all the outcomes and only have semi-bandit feedback.  
It is easy to show that this problem is up-monotone: this is a (simple) special case of Lemma~\ref{lem:stoch-min-ks-monotone} below. It is well-known that the natural greedy policy achieves the optimal cost~\citep{Butterworth72}.
So, using  Theorem~\ref{thm:main} with $k=2$ (all r.v.s are binary),  we obtain a polynomial time
online learning algorithm for series testing having  $1$-regret $O(nC \sqrt{T\log T})$ where $C = \sum_{i=1}^n c_i$ is the total cost.

\subsection{Prophet Inequality}\label{subsec:PI}

The Prophet Inequality~\citep{krengel1977SemiamartsAF, samuel1984comparison} 
is a fundamental problem in optimal stopping, which has also been used extensively in algorithmic game theory.    
The input consists of $n$ rewards which arrive in a given fixed sequence, say $\varX_1, \ldots, \varX_n$. 
Each reward $\varX_i$ is drawn independently from a known distribution $\calD_i$. 
We are interested in online policies, that upon observing each reward,  selects or discards it  immediately. The policy can select at most one reward, and  it terminates right after making a selection   (without observing any future reward). Note that we have  semi-bandit feedback because only some of the rewards are observed in any policy execution. The goal is to maximize the expected selected reward. The classical results obtain a $\frac12$-approximate policy relative to the ``clairvoyant'' optimal value $\E [\max_{i=1}^n \varX_i]$; there are also instances where no policy can achieve a better approximation to  this benchmark.  Here, we will compare to a more realistic non-clairvoyant benchmark: the optimal policy which is also constrained to make selection decisions in the given order (same as an algorithm).   
It is known that there is an optimal threshold-based  policy: given thresholds $\{\tau_i\}_{i=1}^n$, the policy 
selects $i$ if and only if $\varX_i > \tau_i$. This is a $2$-threshold policy, as defined in \Cref{sec:continuous}.

Moreover, the prophet inequality problem is {strongly monotone} (see Lemma~28 in \cite{GuoHT+21}),  which implies that it is down-monotone. Finally, for any policy $\sigma$ and leaf $\ell$ 
(see \eqref{eq:representcts} for details), 
the corresponding reward function is  $f_{\sigma,\ell}(\bx) = x_\ell,$
where $x_\ell$ denotes the value of the random variable at leaf $\ell$. 
It is straightforward to verify that $f_{\sigma,\ell}$ is coordinate-wise monotone 
(see Remark~\ref{rk:coordinate monotone}).
Using Theorem~\ref{thm:main-cont}, we get a polynomial time  online learning algorithm  for the the prophet inequality problem with unknown distributions that has $1$-regret $O(n U \sqrt{T\log T})$ where all r.v.s are $[0,U]$ bounded. This improves upon the $O(n^{3} U  \sqrt{T} \log T)$ bound from \cite{GatmiryKS+22}, although the previous result holds in the stronger bandit-feedback model. We note that there are other learning-based results~\citep{AzarKW14,RubinsteinWW20}  based on limited number of samples, that imply  $\frac12$-regret algorithms by comparing to the clairvoyant benchmark. Note that our  guarantee and that of \cite{GatmiryKS+22} are  stronger because they do  not incur any  multiplicative approximation factor.

\def\F{{\cal F}}
\paragraph{Combinatorial Prophet Inequalities.} The basic prophet inequality concept has also been extended to settings where there is some combinatorial feasibility constraint on the selected items. Here, we have $n$ items with reward $\varX_i\sim \D_i$ for each $i\in [n]$. In addition, there is a {\em downward-closed} set family $\F\sse 2^{[n]}$ 
that represents a generic feasibility constraint;  the selected subset  must be in $\F$.
The $n$ items arrive in  a given fixed sequence. When item $i$ arrives, if $S\cup\{i\}\not\in \F$ where  $S$ is  the set of previously selected items then item $i$ is not considered for selection (and we do not observe $\varX_i$); otherwise, the policy    observes the value of  $\varX_i$   and selects/discards   item $i$. Again, note that we only have semi-bandit feedback because only a subset of items is observed by the policy. 

The performance of an online policy is compared to the clairvoyant optimum \( \OPT^{*}=\E[\max_{S\in \F} \sum_{i\in S}\)\(\varX_i ]\). 
Many specific problems can be modeled in this manner:
\begin{itemize}
    \item When $\F=\{S: |S|\le 1\}$, we get the classic prophet inequality, which has  a  $\frac12$ approximation.      
    \item When $\F=\{S: |S|\le \kappa\}$, there is   a  $1-\frac1{\sqrt{\kappa+3}}$ approximation \citep{Alaei14,JiangMZ22}.      
\item When $\F$ corresponds to independent sets in a matroid, we obtain the matroid prophet inequality where again a $\frac12$ approximation is known~\citep{KleinbergW19}.
\item When $\F$ is the intersection of $p$ matroids, a $\frac1{e(p+1)}$ approximation is known~\citep{FeldmanSZ21}.
\item When $\F$ is given by matchings in a graph, a $0.337$ approximation is known~\citep{EzraFGT22}.
\end{itemize}
For the monotone property, note that we are comparing to the clairvoyant optimum $\OPT^*$ (not the optimal policy).  So, it suffices to prove down-monotonicity   for $\OPT^*$, which is immediate by  stochastic dominance. The online policies in all these results are $2$-threshold policies. For any fixed policy $\sigma$ and leaf $\ell$ 
(see \eqref{eq:representcts}),  we have  
\begin{equation}\label{eq: combinatorial PI}
    f_{\sigma,\ell}(\bx) = \sum_{i\in S_\ell}  x_i ,
\end{equation}
where  $S_{\ell}$ denotes the set of r.v.s selected by the policy $\sigma$ along the root-$\ell$ path. It is clear that this function is coordinate-wise monotone. Therefore, using Theorem~\ref{thm:main-cont}, we obtain $\alpha$-regret $O(n U \sqrt{T\log T})$ for all the above  combinatorial prophet inequalities where $\alpha$ is the best (offline) approximation ratio; we assume that the r.v.s are $[0,U]$ bounded. 

Some approximate regret bounds can also be obtained from previous work on single-sample prophet inequalities~\citep{AzarKW14,RubinsteinWW20,CaramanisDFFLLP22}. 
While the regret bounds via this approach are  better (independent of $T$), they need to compare to approximation ratios that are often  worse than the usual (known distribution) setting. In particular, \cite{FLTWW024} obtained an approximation ratio of $(\frac14-\epsilon)$ for general matroids using $O_\epsilon(n\log^4n)$ total samples, which implies a learning algorithm having $\left(\frac14-\epsilon\right)$-regret  of $O_\epsilon(n\log^4n)$. In contrast, 
our results  imply $\frac12$-regret of $\sqrt{T\log T }$, exactly matching the best-possible approximation ratio from \cite{KleinbergW19}.

\subsection{Pandora's Box} \label{subsec:pandora}In this  problem~\citep{weitzman1979},   we are given distributions $\calD_1, \ldots, \calD_n$ 
such that r.v. $\varX_i \sim \calD_i$.
The realization 
of $\varX_i$ can be ascertained by paying
a known inspection cost $c_i$.
Now, the goal is to find a policy to (adaptively) inspect a subset $S \sse [n]$ of the random variables to maximize $\E \left[\max_{i \in S} \varX_i - \sum_{i \in S} c_i\right]$. Note that any policy only inspects a subset of items  and we only receive feedback from these items, which corresponds to semi-bandit feedback.  
\cite{weitzman1979} obtained an optimal policy based on the ``reservation value'' for each item and 
probing items according to this value until the reward for an  item exceeds all remaining reservation values.   
The reservation value $r_i$ for an item $i$ is such that $\E[(\varX_i - r_i)_+] = c_i$.
We note that this optimal policy is  $2$-threshold.

For the online setting, we assume that the r.v.s $\varX_i$ are $[0,U]$ bounded; the distributions may be discrete or continuous.   It was shown in \cite{GuoHT+21} (see Lemma~31 in that paper) that the Pandora's box problem is {strongly monotone}, which implies that it is down-monotone. Similar to Prophet Inequality, for any policy $\sigma$ and leaf $\ell$ , 
the corresponding reward function is  $f_{\sigma,\ell}(\bx) = \max_{i\in S_\ell} x_i -\sum_{i\in S_\ell}  c _{i},$
where $S_\ell$ is the  set of inspected r.v.s at leaf $\ell$. Hence $f_{\sigma,\ell}$ is coordinate-wise monotone. Then combined with Theorem~\ref{thm:main-cont}, we get a  polynomial time  online learning algorithm  for  the Pandora's box problem with unknown distributions that has $1$-regret $O(n (C+U) \sqrt{T\log T})$ where $C = \sum_{i \in [n]} c_i$ is the total cost. Our regret bound improves upon the $O(n^{4.5} (C+U) \sqrt{T} \log T)$ bound from \cite{GatmiryKS+22}, although the previous result holds in the stronger bandit-feedback model.

\paragraph{Variants of Pandora's Box.} Our framework also applies to more general versions of Pandora's problem that have been studied in prior work. We mention two such variants here.

In Pandora's box with {\em  order constraints}, in addition to the $n$ r.v.s, there are  precedence constraints that enforce that any  r.v. $\varX_i$ may be selected only after all its predecessors have been selected. 
Although the original policy of \cite{weitzman1979} does not apply to this extension, \cite{BoodaghiansFLL23} obtained a different optimal policy when the precedence constraints  form a directed tree: this policy is also a 2-threshold policy. Hence,  Theorem~\ref{thm:main-cont} implies an online learning algorithm  for  Pandora's box with tree order constraints having  $1$-regret $O(n (C+U) \sqrt{T\log T})$, where $C = \sum_{i \in [n]} c_i$ is the total cost and $U$ is the bound on the r.v.s. 

\def\cM{\cal M}

In the matroid Pandora's box problem~\citep{KleinbergWW16,singla2018price}, in addition to the $n$ r.v.s $\{\varX_i\}_{i=1}^n$, there is a matroid $\cM$ with ground set $[n]$. The goal is to inspect a subset $S\subseteq [n]$ of r.v.s and  select a subset $B\sse S$ such that $B$ is independent in matroid $\cM$. The objective is to maximize $\E\left[\sum_{j\in B} \varX_j - \sum_{i\in S} c_i\right]$, the difference between the total selected value  and  inspection cost. 
We recover the original Pandora's box problem when $\cM$ is a rank-1 uniform matroid. There is an optimal 2-threshold policy known for this variant~\citep{singla2018price}.  
This policy is also non-adaptive. By a similar argument as above, we can apply Theorem~\ref{thm:main-cont}. 
Thus, we obtain an online learning algorithm  for  matroid Pandora's box having  $1$-regret $O(n (C+U) \sqrt{T\log T})$, where $C = \sum_{i \in [n]} c_i$ is the total cost and $U$ is the bound on the r.v.s.

\subsection{Stochastic Knapsack}\label{subsec:stoch-knap} 
This is a classic problem in stochastic optimization, which was introduced in \cite{DGV08} and has been studied extensively since~\citep{BGK11,GuptaKMR11,Ma18}. 
There are $n$ items with deterministic rewards $\{r_i\}_{i=1}^n$
and random costs $\{C_i \sim \D_i \}_{i=1}^n$. The realized cost $C_i$ of  item $i$ is only known after selecting it. Given a knapsack budget $B$, a policy   selects items sequentially until the total cost exceeds $B$. The objective is to maximize the expected total  reward from items that fit in the knapsack. If there  is an item that overflows the budget then it does not contribute to the objective. Note that only some subset of items is selected by a policy, and we only observe those costs as  feedback.
In \Cref{lem:stoch-knap-monotone} below, we show that this problem is up-monotone (assuming that the costs are discrete r.v.s).

There is an
{\em adaptive} algorithm for stochastic knapsack, which is  a  $(\frac1{2}-\epsilon)$ approximation  (for  any $\epsilon>0$)~\citep{Ma18}.  
Theorem~\ref{thm:main} then implies an
{ online learning algorithm with   $(\frac{1}{2}-\epsilon)$-regret of $O(n R \sqrt{kT\log T })$ where $k$ is the maximum support size and $R = \sum_{i=1}^n r_i$ is the total reward.} Our results also apply to the more general {\em stochastic orienteering} problem, where items are located at vertices in a metric space and 
we want to find a {\em path} with budget $B$ on the total distance (from the edges in the path) plus  cost (of the visited items). There is  a non-adaptive $\Omega(\frac1{\log\log B})$ approximation algorithm for this problem~\citep{GuptaKNR15}. Here,  we obtain $\Omega(\frac1{\log\log B})$-regret of $O(nR \sqrt{kT\log  T })$.

\begin{lemma}\label{lem:stoch-knap-monotone}
The stochastic  (maximum) knapsack problem is up-monotone.
\end{lemma}
The proof requires the following theorem (also referred to as \emph{Strassen's monotone coupling theorem}) on coupling random variables under stochastic dominance~\citep{strassen1965existence,liggett1985interacting}.

\begin{theorem}[Coupling under Stochastic Dominance]
\label{thm:coupling-sd}
Consider r.v.s $\varX$ and $\var{Y}$ with finite support $[k] = \{1, \ldots, k\}$, and distributed according to $\calD$ and $\calE$ respectively.
Furthermore, suppose that $\calD \sd \calE$. 
Then, there exist { non-negative} values $\left\{ f_{\ell_1, \ell_2} : k\ge  \ell_1 \geq \ell_2\ge 1\right\}$ such that:
\(
\sum_{\ell_2: \ell_2 \leq \ell_1}  f_{\ell_1, \ell_2} = \pr_{\calE}(\var{Y} = \ell_1), \, \forall \ell_1\in[k] \mbox{ and }  \sum_{\ell_1: \ell_1 \geq \ell_2}  f_{\ell_1, \ell_2} = \pr_{\calD}(\var{X} = \ell_2), \, \forall \ell_2\in[k].\)
\end{theorem}

\begin{proof}{\it  of \Cref{lem:stoch-knap-monotone}.}
Consider an arbitrary instance $\calI$ and product probability distributions $\bld{D}=\{\calD_i\}_{i=1}^n$ and $\bld{E}=\{\calE_i\}_{i=1}^n$  where $\bld{D} \sd   \bld{E}$. For each item $i\in [n]$, 
let $f_{\ell_1, \ell_2}^i$ denote coupled values for distributions $\calD_i\sd \calE_i$, constructed as per \Cref{thm:coupling-sd}.
Let $\sigma_{\bld{E}}$ denote 
the optimal policy for $\calI$ under distribution $\bld{E}$, and let $\OPT_\calI(\bld{E})$ denote its reward.
Recall that a policy is given by a decision tree, where each node is labeled by an item
to probe next, and the branches out of a node correspond to the random realization of the probed
item.
We construct a policy $\widehat{\sigma}$ that works under distribution $\bld{D}$, by closely following  $\sigma_{\bld{E}}$. 

When policy $\sigma_{\bld{E}}$ probes some item $i$, policy $\widehat{\sigma}$ does the same. Now, suppose that $\widehat{\sigma}$  observes $C_i = \ell_2 \in [k]$. Then, for each $\ell_1\ge \ell_2$, with probability  $\frac{f_{\ell_1,\ell_2}^i}{\pr_{\calD_i}(\var{X} = \ell_2)}$, policy $\widehat{\sigma}$ continues to follow policy $\sigma_{\bld{E}}$'s branch under the outcome $\ell_1$ (i.e., it wastes $\ell_1 - \ell_2$ space in the knapsack).
{ Note that the total probability $\sum_{\ell_1\ge \ell_2} \frac{f_{\ell_1,\ell_2}^i}{\pr_{\calD_i}(\var{X} = \ell_2)}=1$ by definition of the coupled values in Theorem~\ref{thm:coupling-sd}.}
Moreover,   the probability of reaching any node $s$ in policy $\sigma_{\bld{E}}$ (under $\bE$) equals the probability of reaching the same node $s$ in policy 
 $\widehat{\sigma}$ (under $\bD$). 
Thus, the expected rewards of the two policies are the same, 
$\E_{\bld{X} \sim \bld{D}}[f(\widehat{\sigma}, \bld{X})] = \E_{\bld{X} \sim \bld{E}}[f(\sigma_{\bld{E}}, \bld{X})]$.
We note that $\widehat{\sigma}$ is a randomized policy; however, as noted in \S\ref{sec:prelim}, there is always an optimal deterministic policy for these stochastic problems.  
So, we obtain $\OPT_\calI(\bld{D})\ge  \OPT_\calI(\bld{E})$ as needed. \hfill 
\end{proof}

{

}

\subsection{Stochastic Matching and Probing} In the {\em stochastic matching} problem \citep{ChenIKMR09,BGLMNR12}, there is an undirected graph $G=(V,E)$ with edge-weights $\{w_e\}_{e\in E}$,   edge-probabilities  $\{p_e\}_{e\in E}$, and vertex bounds $\{t_v\}_{v\in V}$. Each edge $e$ is {\em active} independently with  probability $e$. However, the status (active/inactive) of an edge can only be determined by {\em probing} it. There is also a constraint on the set of probed edges: for any vertex $v$, the  number of probed edges incident to $v$ must be at most $t_v$. A solution/policy needs to  to probe a subset of edges and select a matching $M$ consisting of active edges. The objective is  to   maximize the expected weight of $M$.  Finally, there is a ``query commit'' requirement that any probed edge which is active {\em must} be included in the selected matching $M$. Observe that in any policy execution, we only see the status of some subset of edges, which corresponds to semi-bandit feedback. 
This problem has been extensively studied, see e.g., \citet{ChenIKMR09,BGLMNR12,Adamczyk11,BrubachGMS21}. The current best approximation ratio is $0.5$ for the unweighted  case \citep{Adamczyk11} and  $0.382$ for the weighted case \citep{BrubachGMS21}.
We show in Lemma~\ref{lem:stoch-probe-mon} below that the stochastic matching problem is down-monotone. So, Theorem~\ref{thm:main} with $k=2$ (as all r.v.s are binary) implies an online learning algorithm  for stochastic matching having $0.382$-regret of 
 $O(n W \sqrt{T\log  T })$ where $W = \sum_{e\in E}  w_e$ is the total weight. 

In fact, our result applies to the much more general {\em stochastic probing problem}, as defined in \cite{GN13}. Here, we have a set $E$ of stochastic items 
with weights $\{w_e\}_{e\in E}$ and probabilities $\{p_e\}_{e\in E}$. Each item is  active  independently with  probability $p_e$, and this  status can only be   determined by  probing $e$. We now have two downward-closed constraints: an {\em inner} constraint $\F_{in}$ and an  {\em outer} constraint $\F_{out}$. We want to probe a set $Q$ of items subject to  the outer constraint  (i.e. $Q\in \F_{out}$) and select a subset $S\sse Q$ of active (probed) items satisfying the inner constraint  (i.e. $S\in \F_{in}$). The objective is to maximize the expected weight of the selected items $S$. We again have the query-commit requirement that any active probed item must be selected into the solution $S$. When both inner/outer constraints are $k_{in}$ and $k_{out}$ systems,
there is a  $\frac{1}{k_{in} + k_{out}}$ approximation for the unweighted case and an $\Omega(\frac{1}{(k_{in} + k_{out})^2})$ approximation for the weighted case~\citep{GN13}. When the $k$-systems are   intersections of matroids, \cite{AdamczykSW16} gave an improved adaptive algorithm with approximation ratio $\frac{1}{k_{in} + k_{out}}$, even for the weighted case.

\begin{lemma}\label{lem:stoch-probe-mon}
    The stochastic probing problem is down-monotone.
\end{lemma}
\begin{proof}
    Consider an arbitrary instance $\calI$ and product probability distributions $\bld{D}=\{p_e\}_{e\in E}$ and $\bld{E}=\{\bp_e\}_{e\in E}$  where 
    $\bld{D} \sd   \bld{E}$. As the r.v.s are binary in this problem, the stochastic dominance just means $p_e\le \bp_e$ for all $e\in E$.

Let $\sigma$ denote 
the optimal policy for $\calI$ under distribution $\bld{D}$, and let $\OPT_\calI(\bld{D})$ denote its expected weight.
Recall that a policy is given by a decision tree, where each node is labeled by an item
to probe next, and the branches out of a node correspond to the random $0-1$ realization of the probed
item.
We construct a policy $\widehat{\sigma}$ that works under distribution $\bld{E}$, by closely following  $\sigma$. 

When policy $\sigma$ probes  item $e$ at any node, policy $\widehat{\sigma}$ does the following:
\begin{itemize}
    \item probe $e$ with probability $p_e/\bp_e$, and follow the  active/inactive  branch based on its outcome.
    \item skip directly to the inactive branch  with probability $1-p_e/\bp_e$.
\end{itemize}  
Note that conditioned on reaching such a node, the probability that $\widehat{\sigma}$ goes to the active branch is $\frac{p_e}{\bp_e}\cdot \bp_e = p_e$; so the probability that $\widehat{\sigma}$ goes to the inactive branch is $1-p_e$. Therefore, the probability of reaching any node $v$ in policy  $\widehat{\sigma}$ (under distribution $\bE$) equals the probability of reaching   node $v$ in policy  ${\sigma}$ (under distribution $\bD$). Thus the expected values of these   two policies are the same, 
$\E_{\bld{x} \sim \bld{E}}[f(\widehat{\sigma}, \bld{x})] = \E_{\bld{x} \sim \bld{D}}[f(\sigma , \bld{x})]$. We note  that $\widehat{\sigma}$ also satisfies the query-commit requirement because $\sigma$ does (whenever $e$ is probed and found active, it will be included in the solution).
Finally, observe   that $\widehat{\sigma}$ is a randomized policy; however, as noted in \Cref{sec:prelim}, there is always an optimal deterministic policy for these stochastic problems.  
So, we obtain $\OPT_\calI(\bld{E})\ge  \OPT_\calI(\bld{D})$ as needed. 
    \end{proof}

\subsection{Stochastic Covering Knapsack}

Here, we have $n$ items with deterministic costs $
\{c_i\}_{i \in [n]}$ and (independent) 
random rewards $\{\var{R}_i\}_{i \in [n]}$.
There is a known probability distribution $\calD_i$ for $\var{R}_i$; we assume discrete distributions here. 
Given a target $Q$, we need to select items until the total
observed reward is at least $Q$ (or we run out of items). The objective is to minimize the expected cost. Note that we only observe items until the target is achieved, which corresponds to semi-bandit feedback. There is an adaptive $3$-approximation algorithm for this problem~\citep{deshpande2016approximation}. By a proof identical to that of Lemma~\ref{lem:stoch-knap-monotone}, we get:
\begin{lemma}\label{lem:stoch-min-ks-monotone}
    The stochastic covering knapsack problem is up-monotone. 
\end{lemma}
So, using Theorem~\ref{thm:main},  we obtain a polynomial time
online learning algorithm for stochastic min-knapsack having  $3$-regret $O(n C \sqrt{kT\log T })$ where $C = \sum_{i=1}^n c_i$ is the total cost and $k$ is the maximum support size.

\paragraph{Stochastic $Q$-TSP.}

This is the stochastic variant of the classical $Q$-TSP problem~\citep{Garg05}. An instance of the problem comprises a metric $(V, d)$ with a root vertex $r \in V$. There are random rewards $\{\var{R}_v\}_{v \in V}$ associated with the vertices.
All reward distributions are given as input, but the realized reward $\var{R}_v$ is only known when vertex $v$ is visited. 
We assume that all reward distributions are discrete.
Given a target $Q$, the goal is to find an (adaptive) tour originating from $r$ that collects a
total reward at least $Q$ (or runs out of vertices), at the minimum expected length.

A constant-factor approximation algorithm for stochastic $Q$-TSP was obtained in  \cite{jiang2020ktsp}. Applying \Cref{thm:main}, we obtain a  polynomial time  online learning algorithm for stochastic $Q$-TSP having $O(1)$-regret $O(nD \sqrt{kT\log T })$ where $D = \sum_{v\in V} d_{rv}$ corresponds to an upper bound on the (full) TSP cost.

\subsection{Stochastic Submodular Optimization}

Consider a monotone submodular function $f:2^E\rightarrow \mathbb{Z}_+$ defined on a groundset $E$. Constrained submodular maximization involves selecting a subset $S\sse E$ of items in a downward-closed family $\F\sse 2^E$, that maximizes the function value $f(S)$. This is a fundamental problem in combinatorial optimization and   good approximation algorithms are known for a variety of constraints $\F$ such as cardinality, matroids and knapsacks \citep{NemhauserWF78,CalinescuCPV11,KulikST13,ChekuriVZ14}. We consider a stochastic variant, as in \cite{asadpour2016maximizing}, \cite{gupta2017adaptivity}, and \cite{Bradac0Z19},  where each item $i\in E$ 
is {\em active} independently with some probability $p_i$ (and inactive otherwise).\footnote{These papers   also handle a more general setting where  items have an arbitrary outcome space (not just binary).} Moreover, the function value is accrued only on active items. The active/inactive status of any item is only known after selecting it. The goal now is to select a subset $S\in \F$  sequentially (and adaptively) so as to maximize the expected function value from the selected active items, i.e., $\E_{S,A}\left[ f(S\cap A)\right]$ where $A\sse E$ is the random set of active items. Note that we can only select (and observe) a constrained  subset of items, which corresponds to semi-bandit feedback. When $\F$ is a single matroid constraint, \cite{asadpour2016maximizing} obtained a $(1-\frac1e-\epsilon)$-approximation algorithm for the stochastic submodular maximization problem. Later, \cite{gupta2017adaptivity} and \cite{Bradac0Z19} showed that a  $\rho$-approximation algorithm  for  {\em deterministic} submodular maximization under   any   constraint $\F$ can be used to obtain a  $\frac\rho{2}$-approximation algorithm  for the stochastic problem under  constraint $\F$.
It is easy to see that this problem is down-monotone. So, Theorem~\ref{thm:main} with $k=2$ (binary r.v.s) implies a  $(1-\frac1e-\epsilon)$-regret of $O(n F \sqrt{T\log T })$ for stochastic submodular maximization over a matroid constraint, where $F$ is the maximum function value. Combined with the result of \cite{Bradac0Z19}, we also obtain online learning algorithms for  more general constraints $\F$.
Notably, our result improves upon the prior result of \cite{gabillon2013adaptive} that obtains $\left(1 - \frac{1}{e}\right)$-regret of $\widetilde{O}(nFT^{2/3})$ in the semi-bandit feedback model for the stochastic submodular maximization problem when $\mathcal F$ corresponds to a cardinality constraint.

\newcommand{\ssc}{$\mathtt{SSC}$}

\paragraph{Stochastic Submodular Cover.} In the classic submodular cover problem~\citep{W82}, there is a monotone submodular function $f:2^E\rightarrow \mathbb{Z}_+$ defined on a groundset $E$, where each item $i\in E$ has some cost $c_i$. We assume that the function is integer-valued with maximal value $Q$.  The goal is to select a subset $S\sse E$ that ``covers'' the function, i.e., $f(S)=Q$. The objective is to minimize the total cost $\sum_{i\in S} c_i$ of the chosen items.  Note that the submodular cover problem generalizes set cover. It is well-known that this problem can be approximated within a factor  $1+\ln(Q)$ \citep{W82}, and no better approximation is possible~\citep{DinurS14}. We consider a stochastic version, where each item $i\in E$ is {\em active} independently with   probability $p_i$, and
the function value is accrued only on active items. The goal now is to select items sequentially (and adaptively) so as to cover function $f$ (or run out of  items to select), at the minimum expected cost. 
This problem has been studied extensively \citep{GolovinK-arxiv,im2016minimum,HellersteinKP21}, and \cite{HellersteinKP21} obtained a  $1+\ln(Q)$ approximation algorithm   for it.
Using the fact that the function $f$ is monotone, it is easy to see that  stochastic submodular cover is up-monotone. So, using  Theorem~\ref{thm:main} with $k=2$,  we obtain a polynomial time
online learning algorithm for stochastic submodular cover having  $(1 + \ln Q)$-regret $O(nC \sqrt{T\log T})$ where $C = \sum_{i=1}^n c_i$ is the total cost.

\newcommand{\setC}{\mathcal{C}}
\def\srm{\ensuremath{\mathsf{SRM}}\xspace}

\subsection{Single Resource Revenue Management}
We now consider the single resource revenue management problem (\srm). A firm (say, airline) has $C$ units of a certain resource (for example, seats on a flight).
This resource must be allocated among $n$ fare classes, where each unit  sold in  class $i$ generates revenue $p_i$. Demands for these fare classes arrive sequentially: we use the  convention from \cite{GT19} that classes arrive in the order $n, n-1,\ldots ,1$. This problem is often studied under the ``low before high'' assumption that  $p_n < \cdots < p_1$, which reflects the fact that cheaper fares are offered earlier. 

The demands are random and independent across classes: we use $\varX_i\sim \calD_i$ to denote the  demand for  class $i$.
We  assume that each $\varX_i$ is a discrete r.v. with   support $\{0, 1, \ldots, C\}$. The goal  is to decide how much  capacity to make  available for each fare class so as to maximize the expected total revenue.
For each class $i=n,\ldots,1$, the firm needs to first decide on the number $b$ of units offered (which cannot exceed the  available units)  and then the demand $\varX_i$ materializes, which  reduces  the available units by $\min(\varX_i , b)$.   When the demand distributions are unknown, \emph{censored} feedback is   the most natural,  where we observe only the minimum of the actual demand  and the offered capacity. As an application of Theorem~\ref{thm:main for censored} we obtain a polynomial-time online learning algorithm for \srm with $1$-regret $O(nC^2 p_{\max}\sqrt{CT\log(CT)})$ where $p_{\max}$ is the maximum price.  
This result directly improves the dependence on $T$ compared to \citet{van2000revenue}, whose algorithm achieves a regret of $\widetilde{O}\left(T^{1-\frac{\beta}{2^{n-1}}}\right)$ for some constant $\beta \in (0, 1)$.
Moreover, it also improves upon the result of \citet{huh2006adaptive}, who obtain regret $O( \sqrt{p_{\max}}C^{2n-3} \cdot \sqrt{T})$, thereby removing the exponential dependence on $n$ in the regret.
  
In order to apply Theorem~\ref{thm:main for censored}, we first note that the objective satisfies Assumption~\ref{asmp:censored-f}. Indeed, any policy for \srm involves setting a threshold at each step $i$ corresponding to  the number of units offered to class $i$. Moreover,  
the increase in the \srm objective at any step is completely determined by the (censored) observation at that step: so the function value at each leaf node of the policy is a constant.
Next, we prove that \srm is down-monotone. This proof relies on some known structural properties of the   exact dynamic programming algorithm for \srm (with    known distributions). We now summarize these properties; see \citep{GT19} for more details.

\medskip
\noindent{\bf \srm dynamic program.} The value functions in the dynamic program (DP) are as follows. For each $i\in [n]$ and $c\le C$, let $V_i(c)$ denote the optimal total expected revenue from   classes $i,i-1, \ldots, 1$ given that we have $c$ units of capacity remaining before the arrival of  class $i$ demand. The optimal value of the \srm instance is $\OPT=V_n(C)$.  

The decision for class $i$ is to determine the number of units offered to class $i$ demand; equivalently, we need to decide the number of units to  \emph{protect} for future  classes $i-1, \ldots, 1$.
Specifically, if $b \in \{0, 1, \ldots, c\}$ units are protected, then $c - b$ units are available for class $i$ demand. The amount  sold to   class $i$ equals $\min(c-b, \varX_i)$ and the remaining capacity is $c - \min(c-b, \varX_i)$, which are both  random. 
The revenue generated, also a random variable, is $p_i \cdot \min(c-b, \varX_i)$.
We let $W_i(b, c)$ denote the optimal expected  revenue from  classes $i,i-1, \ldots, 1$ given that we have $c$ units of capacity remaining before the arrival of class $i$ and we protect $b \leq c$ units. Based on the discussion above,  

\begin{align}
    W_i(b, c) &= \,\, p_i \cdot \E[\min(c- b, \varX_i)] \notag \\ 
    & \qquad +  \E[V_{i-1}(c - \min(c-b,  \varX_i))].\label{eqn:decision-dp}
\end{align}
The first term is  the expected revenue obtained from  class $i$ given that we offered $c-b$ units, and the second term is the expected   revenue  from   classes $i-1, \ldots, 1$ when $\max(b, c-\varX_i)$ units of capacity remain. Now, observe that $V_i(c) = \max \left\{  W_i(b, c)\,:\, 0\le b\le c\right\}$. The base case $V_1(c)=p_1\cdot \E[\min(c,\varX_1)]$ involves protecting $0$ units regardless of $c$ (as there is no demand after this).

We will use the following two properties of this DP (see Proposition~1.5 and Theorem~1.6 in \citet{GT19}).
\begin{lemma}\label{lem:srrm-structure}
 Letting  $\Delta V_i(x):= V_i(x) - V_i(x-1)$, the value functions  satisfy the following.
    \begin{itemize}
        \item $\Delta V_i(x)$ is decreasing in $x = 1, \ldots, c$, and 
        \item $\Delta V_i(x)$ is increasing in $i = 1, \ldots, n$.
    \end{itemize}
\end{lemma}

\begin{lemma}\label{lem:protection-level-property}
For $i\in [n]$ and $c\le C$, the maximizer of $W_i(b,c)$ over $0\le b\le c$ equals $\min\{y^*_{i-1},c\}$ where

    \[
    y^*_{i-1} = \max\{ 0\le y \le C: \Delta V_{i-1}(y) > p_i\}
    \]
is the protection-level in class $i$. So, the optimal policy is fully specified by  $y^*_{n-1}, y^*_{n-2}, \ldots, y^*_1, y^*_0$.  
\end{lemma}

\begin{lemma}
    The single-resource revenue management problem  is down-monotone.
\end{lemma}
\begin{proof}
Consider any   instance  of the problem with $n$ classes. Let $\bld{D}=\calD_1 \times \cdots \calD_i \times \cdots \calD_n$ and $\bld{E}=\calD_1 \times \cdots \calE_i \times \cdots \calD_n$  be two  probability distributions that differ only in the distribution of one class $i$. It suffices to show that $\OPT(\bld{E})\ge \OPT(\bld{D})$ whenever $\calD_i\sd \calE_i$.     

For any $j\in [n]$ and $c\le C$, let $V_j(c)$ and $\overline{V}_{j}(c)$ denote the optimal value functions under distributions $\bld{D}$ and $\bld{E}$ respectively. So, $\OPT(\bld{D})=V_n(C)$ and $\OPT(\bld{E})=\overline{V}_n(C)$. In order to prove $V_n(C)\le \overline{V}_n(C)$ it   suffices to show that $V_i(c)\le \overline{V}_i(c)$ for all $c\le C$, which we do in the rest of the proof.

Let $y_{n-1}^*, \ldots, y_1^*, y_0^* = 0$ denote the optimal protection levels   under product distribution $\bld{D}$ (from Lemma~\ref{lem:protection-level-property}). 
    Recall that $V_i(c)=W_i(\min(y^*_{i-1},c) , c)$. If $c<y^*_{i-1}$ then $V_i(c)=V_{i-1}(c)$ and we obtain $$\overline{V}_i(c)\ge \overline{V}_{i-1}(c)= {V}_{i-1}(c)=V_i(c),$$ 
    where the first inequality is by ignoring class $i$ and the first equality uses the fact that the demands in classes $i-1,\ldots,  1$ have the same distributions in $\bld{D}$ and $\bld{E}$.

    \def\bxi{\widetilde{{\varX_i}}}
    \def\hxi{\widehat{{\varX_i}}}
 Below, we assume that $c\ge y^*_{i-1}$, which means  $V_i(c)=W_i( y^*_{i-1}  , c)$.  
    To reduce notation, let $y=y^*_{i-1}$ and  $\Delta_d:=\Delta V_{i-1}(d)$ for any $d\le C$. 
Let $\bxi=\min(c-y , \varX_i)$ be the number of class $i$ units sold, when we implement  the optimal policy for $V_i(c)$. Note that $\varX_i\sim \calD_i$.
We have

\begin{align*}
    &V_i(c) = W_i( y  , c) \,=\, p_i \cdot  \E_{\calD_i}[\bxi] +  \E_{\calD_i}[V_{i-1}(c-\bxi)]\\
   &=  p_i \cdot  \E_{\calD_i}[\bxi] + V_{i-1}(y ) +\sum_{d=y  +1}^c \Delta_d \cdot \Pr_{\calD_i}[\bxi\le c-d]\\
   &=p_i \cdot (c-y ) - p_i\cdot \E[c-y-\bxi] + V_{i-1}(y )  +\sum_{d=y  +1}^c \Delta_d \cdot \Pr [\bxi\le c-d]\\
   &= p_i \cdot (c-y )+ V_{i-1}(y ) - p_i \sum_{t=1}^{c-y} \Pr[\bxi\le c-y-t]  + \sum_{d=y  +1}^c \Delta_d \cdot \Pr [\bxi\le c-d]\\
      &= p_i \cdot (c-y )+ V_{i-1}(y ) - p_i \sum_{d=y+1}^{c} \Pr[\bxi\le c-d]  + \sum_{d=y  +1}^c \Delta_d \cdot \Pr [\bxi\le c-d]\\
      &= p_i \cdot (c-y )+ V_{i-1}(y ) + \sum_{d=y+1}^{c} \left( \Delta_d -p_i\right)\cdot  \Pr[\bxi\le c-d] \end{align*}

Let $\hxi=\min(c-y , \varX_i)$ where $\varX_i\sim \calE_i$ (the class $i$ demand under $\bld{E}$). We now have

\begin{align*}
    &\overline{V}_i(c) \ge \,\,   p_i \cdot  \E [\hxi] +  \E [\overline{V}_{i-1}(c-\hxi)]\\
    &= \,\,  p_i \cdot  \E [\hxi] +  \E [V_{i-1}(c-\hxi)]\\
      &= \,\, p_i \cdot (c-y )+ V_{i-1}(y ) + \sum_{d=y+1}^{c} \left( \Delta_d -p_i\right)\cdot  \Pr[\hxi\le c-d] \\
      &\ge p_i \cdot (c-y )+ V_{i-1}(y ) + \sum_{d=y+1}^{c} \left( \Delta_d -p_i\right)\cdot  \Pr[\bxi\le c-d]
      \end{align*}
The first equality uses $\overline{V}_{i-1}(\cdot)=V_{i-1}(\cdot)$. The second equality follows  by the same sequence of steps as for $V_i(c)$ above. The last inequality uses two facts (i) $\hxi$    stochastically dominates $\bxi$  which implies $\Pr[\hxi\le z] \le \Pr[\bxi\le z]$ for all $z\in \R$, and (ii) $ \Delta_d -p_i\le 0$ for all $d\ge y+1$ by definition of $y^*_{i-1}$ in Lemma~\ref{lem:protection-level-property}. \hfill 
\end{proof}

\subsection{Sequential Posted Pricing}\label{subsec: SPM}
We now consider a pricing problem.
A seller offers a  service to $n$ customers, where each customer's valuation is modeled as a discrete independent random variable $\varX_i \sim \calD_i$. We assume (by scaling) that the values are bounded in $[0,1]$ and that $k$ denotes the maximum support size. 
The seller must satisfy a feasibility constraint given by a downward-closed set family $\mathcal{F} \subseteq 2^{[n]}$. This means that a subset $S$ of customers can be served if and only if $S\in {\cal F}$. We are interested in designing a truthful mechanism that maximizes the expected revenue. \cite{Myerson81} provides  an optimal mechanism  for this class of problems, but it is computationally hard to implement. Therefore, there has been much focus on simpler mechanisms that are approximately optimal. We consider   

Sequential Posted-Price Mechanisms (SPMs) that were introduced by \citet{ChawlaHMS10}. Here, each buyer $i$ is offered a  take-it-or-leave-it price  $p_i$ in some sequence. We consider two variants: 
\emph{Adaptive Sequential Posted-Price Mechanism} (ASPM) where the   seller can (adaptively) choose the order in which to  offer  customers, and \emph{Fixed Order Sequential Posted-Price Mechanism} (FSPM) where the  order of customers is fixed and known upfront.
Formally, in ASPM, the seller's policy is represented as a decision tree, where each internal node is labeled with a tuple $(i, p_i)$, denoting that customer $i$ is offered price $p_i$; each node has two outgoing branches corresponding to whether the customer accepts or rejects the offer. 
In FSPM, the policy only needs to choose the prices $p_i$ as the customers will be considered in a fixed order $1, \ldots, n$. The seller collects revenue from all customers who accept their offered price, and the selected set of customers must lie in $\mathcal{F}$.

For the setting with  unknown distributions $\calD_i$, we   use our online framework  with  \emph{binary feedback}: in each period, the seller chooses a  policy and   only observes which customers accepted their offers. As an application of Theorem~\ref{thm:main for binary}, we obtain $\alpha$-regret of $O\left(n \sqrt{kT \log(knT)}\right)$, where $\alpha$ is the best approximation ratio for the SPM instance with known distributions. 
Below, we list some concrete SPM applications.
\begin{itemize}
    \item When $\mathcal{F} = \{S \subseteq [n] : |S| \leq \kappa\}$ (uniform matroid of rank $\kappa$), the optimal policy for FSPM can be computed efficiently using dynamic programming; so $\alpha=1$.
    For ASPM under this feasibility constraint,  \citet{ChakrabortyEGMM10} provide a polynomial-time approximation scheme, so $\alpha=1-\epsilon$ for any fixed $\epsilon>0$.

    \item When $\mathcal{F}$ corresponds to the independent sets of a general matroid, \citet{Yan11}  provides a $(1-\frac1e)$- approximation algorithm for  ASPM.
    
    \item When $\mathcal{F}$ is an intersection of a constant number of matroid, knapsack and matching constraints, \citet{FeldmanSZ21} provide a  framework
    to achieve an $\alpha=\Omega(1)$-approximation algorithm for  ASPM.
\end{itemize}

Previous work by~\citet{SW2024} obtained online  algorithms for FSPM under a uniform matroid of rank one, i.e., the seller can serve only one customer and the customers arrive in a fixed order. Notably, they  considered the \emph{bandit feedback} model, which, in this specific setting, coincides with the \emph{binary feedback} model.
Although the bandit feedback model only reveals the final revenue, one can recover the identity of the customer who accepted the service by introducing arbitrarily small perturbations to the posted prices.
Our result for this case is a $1-$regret of $O\left(n \sqrt{kT \log(knT)}\right)$, which improves upon the $O(n^{2.5}k\sqrt{T} \log T+ n^5 k \log^2 T)$ bound from ~\citet{SW2024}. We note that
this paper also obtained $poly(n)\cdot \sqrt{T}$ regret for  continuous distributions (under some regularity assumption).

Note that any SPM policy chooses a threshold (the price $p_i$) for each customer $i$. Moreover, the increase in the objective value for  any customer  $i$ is determined by the binary observation $\1_{\varX_i \geq p_i}$. So, SPM  satisfies Assumption~\ref{asmp:censored-f}. We now prove  monotonicity of SPM.
Without loss of generality, for any decision tree (policy) $\Pi$, we assume that the left branch out of a node $(i,p_i)$ corresponds to customer $i$ accepting
while the right branch corresponds to rejection.
\begin{lemma}
 ASPM with any downward-closed constraint ${\cal F}$ is down-monotone.
\end{lemma}
\begin{proof}

Consider two product distributions $\bld{D}=\{\mathcal {D}_i\}_{i=1}^n$ and  $\bld{E}=\{\mathcal {E}_i\}_{i=1}^n$  such that $\bld{D} \sd  \bld{E}$. We will show that $\OPT(\bld{D}) \leq \OPT(\bld{E})$ by induction on the number of customers $n$. 

\textbf{Base case ($n = 1$).}  

Fix two distributions $\calD_1 \sd \calE_1$. Let $p$ be the price offered to this customer under the optimal policy for $\calD_1$. Then the optimal revenue satisfies:
\(\OPT(\bld{D}) = p \cdot\)\(\Pr_{\varX_1\sim \calD_1}(\varX_1 \geq p) 
\leq p \cdot \Pr_{\varX_1\sim \calE_1}(\varX_1 \geq p)\) 
\(\leq \OPT(\bld{E}),\)
where the first inequality follows from stochastic dominance.

\textbf{Inductive step.}  
Assume the statement holds for all instances with $n-1$ customers. Now consider an instance with $n$ customers and two product distributions $\bld{D} \sd \bld{E}$. Let $\Pi^*$ denote the optimal policy under distribution $\bld{D}$. Let $(r, p)$ denote the first node in decision tree $\Pi^*$, where $r\in [n]$   and $p$ is the price offered to  customer $r$.
Let $V_L(\bld{D})$ and $V_R(\bld{D})$ denote the optimal revenues of the subproblems induced by the left and right subtrees of the root, respectively. Note that the feasibility constraints for the left and right subtrees are ${\cal F}_L=\{S\sse[n]\setminus r : S\cup r\in {\cal F}\}$ and  ${\cal F}_R=\{S\sse[n]\setminus r : S \in {\cal F}\}$; these are downward-closed because ${\cal F}$ is. Both of these SPM instances involve $n-1$ customers and  downward-closed feasibility constraints.  

The optimal revenue of the original instance can be expressed as 
\(\OPT(\bld{D}) = \)\(\Pr_{\calD_r}(\varX_r \ge p) \cdot (p + V_L(\bld{D}))\) 
\(+ \Pr_{\calD_r}(\varX_r < p) \cdot V_R(\bld{D})\), which can be further reduced to obtain:
{\small \begin{align}
\OPT(\bld{D}) = \Pr_{\calD_r}(\mathsf{X}_r \geq p) \cdot (p - V_R(\bld{D}) &+ V_L(\bld{D})) + V_R(\bld{D}) \label{eq:spm-opt}
\end{align}}
We now claim that  without loss of generality 

\(p \ge V_R(\bld{D}) - V_L(\bld{D}).\)
Indeed, if $p< V_R(\bld{D}) - V_L(\bld{D})$ then raising the price to $\bar{p} = V_R(\bld{D}) - V_L(\bld{D})$ can only increase the expression~\eqref{eq:spm-opt}.

Using this, we have:
{\small \begin{align*}
&\OPT(\bld{D}) 
= \Pr_{\varX_r\sim \calD_r}(\varX_r \ge p) \cdot (p - V_R(\bld{D}) + V_L(\bld{D}))  + V_R(\bld{D}) \\
&\le \Pr_{\varX_r\sim \calE_r}(\varX_r \ge p) \cdot (p - V_R(\bld{D}) + V_L(\bld{D}))  + V_R(\bld{D}) \\
&= \Pr_{\varX_r\sim \calE_r}(\varX_r \ge p) \cdot (p + V_L(\bld{D})) + \Pr_{\varX_r\sim \calE_r}(\varX_r < p) \cdot V_R(\bld{D}) \\
&\le \Pr_{\varX_r\sim \calE_r}(\varX_r \ge p) \cdot (p + V_L(\bld{E})) + \Pr_{\varX_r\sim \calE_r}(\varX_r < p) \cdot V_R(\bld{E})\\
&\le \,\, \OPT(\bld{E}).
\end{align*}}
Above, the first inequality follows from stochastic dominance and \(p \ge V_R(\bld{D}) - V_L(\bld{D})\). The second inequality follows from the induction. 
This completes the inductive step.\hfill
\end{proof}

Using an identical proof as above, we can also prove that FSPM (the fixed order variant) is down-monotone.

\section{Conclusion}
In this paper, we presented a general online learning framework that achieves $\widetilde{O}(\sqrt{T})$ regret for monotone stochastic optimization problems under semi-bandit feedback. We also extended our results to the more restrictive censored and binary feedback settings.
We conclude by highlighting several interesting directions for future work.

\paragraph{Strong monotonicity.} Our framework requires only the (weaker) monotonicity condition. A natural question is whether leveraging the stronger condition of strong monotonicity~\citep{GuoHT+21} can yield improved dependence on $n$ and $f_{\max}$ in the regret bound. Recent work of \citet{liu2025improved} has made progress in this direction for Pandora's box and prophet inequality, obtaining bounds with an improved $\sqrt{n}$ factor. 
Extending such improvements to all strongly monotone problems remains an interesting open question.

\paragraph{Dynamic environments.} Our framework assumes that the underlying distributions remain identical across rounds, which is necessary for sublinear regret in the worst case (as evidenced by the $\Omega(T)$ lower bound in the adversarial setting for prophet inequality~\citep{GatmiryKS+22}). Designing algorithms for more structured dynamic environments (e.g., through the lens of distributionally robust optimization) is an interesting direction for future work.

\paragraph{Restrictive action selection.} Another natural extension is to settings where the set of available actions itself varies across rounds; for example, a random subset of items is available in each round, akin to the \emph{sleeping bandits} setting. Such an extension would require redefining the benchmark because a static optimal policy may no longer be sufficient.

\section*{Acknowledgement}
We thank Ziyun Chen for pointing out some missing assumptions in our result for continuous distributions, which was stated in an earlier version of this paper. This observation lead us to the corrected statement and proof of Theorem~\ref{thm:main-cont}.  

\bibliographystyle{plainnat}
\bibliography{ref}
\appendix

\section*{Appendix}

\section{Sampling-based Algorithms}\label{app:sampling-algs}

A stronger monotonicity condition is that of \emph{strong monotonicity}~\citep{GuoHT+21}, defined next.

\begin{definition}[Strong Monotonicity]\label{def:strong-monotone}
A stochastic problem is \emph{strongly up-monotone} if for any instance $\calI$ and probability distributions $\bld{D}$ and $\bld{E}$  with  $\bD\sd \bld{E}$, we have $\E_{\bld{x}\sim \bld{E}}[f(\sigma_{\bld{D}}, \bld{x})] \le  \OPT_{\calI}(\bld{D})$, where $\sigma_{\bld{D}}$ is the optimal policy for instance $\calI$ under distribution $\bld{D}$.  Strong  down-monotonicity is defined similarly.  

\end{definition}

We note that strong monotonicity implies monotonicity. 
In prior work, \cite{GuoHT+21} gave optimal sample complexity bounds for stochastic optimization problems that exhibit the {strong monotonicity} condition. They also proved this  property for  problems including   Prophet inequality and Pandora's box.
We make use of this property in \S\ref{sec:apps}, when we apply our  result to these problems. Below, we assume that $f_{max}=1$ by scaling. 

\begin{theorem}[Theorem 17~\cite{GuoHT+21}]\label{thm:guo-sample-strong}
For any strongly monotone stochastic problem, suppose the number of samples is at least:
$$ C \cdot \frac{n}{\epsilon^2} \log\left(\frac{n}{\epsilon}\right)\log\left(\frac{nT}{\epsilon}\right) $$
where $C > 1$ is a sufficiently large constant. Then, there is
an algorithm that gets an $\epsilon$-additive approximation to the optimum
with probability at least $1 - \frac{1}{T^2}$.
\end{theorem}

Below, we   discuss what regret bounds can be achieved via this sampling-based approach,  in order to compare to our results. In the full-feedback model, one can obtain regret bounds of 
$\sqrt{nT\log T}$, which is nearly optimal. 
However, for \emph{semi-bandit} feedback   that we consider (and for bandit feedback), such a ``reduction'' from sample-complexity bounds only provides a sub-optimal 
$\widetilde{O}(T^{2/3})$ regret.  Therefore, we need  new ideas to get the optimal $\widetilde{O}(\sqrt{T})$ regret bound in the semi-bandit model, which we do in this paper.

\paragraph{Full-feedback.}
Recall that, in the full-feedback model, we get one sample from each r.v. in every time-step.
We now describe the algorithm. The algorithm is straightforward: for each $t = 1, 2, \ldots T$, 
we use estimates from the prior $t-1$ time steps to obtain a policy to use for the $t^{\text{th}}$ time-step. As a consequence of Theorem~\ref{thm:guo-sample-strong},
with probability at least $1-\frac{1}{T^2}$, we have that
$\OPT(\bld{\widehat{D}_t})$ is an $\epsilon_t$-additive approximation to $\OPT(\bld{D})$ where $\epsilon_t \leq \sqrt{\frac{Cn\log(nT)}{t}}$. (To keep calculation simple, we ignore the $\log\frac1\epsilon$ dependence in Theorem~\ref{thm:guo-sample-strong}; so we are actually assuming   a slightly stronger sample-complexity bound.)
By union bound, this is true for all $t = 1, \ldots, T$
with probability at least $1 - \frac{1}{T}$. We consider this to be a \emph{good} event, $G$.
Under this event, the total regret, say $R_T$, can be bounded as follows. $$ \E[R_T \mid G] \leq  \sum_{t=1}^{T}\sqrt{\frac{{Cn\log\left(nT\right)}}{t}} = O\left(\sqrt{nT\log(nT)}\right). $$
By law of total expectation we have $$ \E[R_T] \leq \E[R_T \mid G] \cdot \pr(G) + \E[R_T \mid \overline{G}] \cdot \pr (\overline{G}) \leq O\left(\sqrt{nT\log(nT)}\right) + T \cdot \frac{1}{T} = O\left(\sqrt{nT\log (nT)}\right).  $$

\paragraph{Semi-bandit feedback.} 
A significant challenge in the semi-bandit feedback model arises from our lack of control over the r.v.s from which we get samples. 
One potential strategy to address this issue is to artificially generate samples for an item $i$ by  probing  item $i$   first   in the algorithm's policy for that period. 
However, this approach comes with an inherent drawback –  probing item $i$ first may result in a poor policy, and so we  suffer a high regret in such  periods. 
The standard \emph{explore-then-exploit} algorithm that 
first gets $T^{2/3}$ samples for each r.v., and then plays the optimal policy (for the empirical distribution) for the remaining time steps only achieves  regret $\widetilde{O}(T^{2/3})$. 
 We are not aware of any  generic approach that reduces sample-complexity bounds to regret minimization in the semi-bandit have $o(T^{2/3})$ regret.

\paragraph{Remark 1.} 
The sampling approach   
for   semi-bandit feedback   also applies to many  problems  (e.g., prophet inequality, 
Pandora's box and series testing)  
in the {\em bandit} feedback model. Basically, for each item $i$ we need a policy where its objective corresponds to the value of   r.v. $\varX_i$.   So, we can directly get $\widetilde{O}(T^{2/3})$ regret for these problems even with bandit feedback.

\paragraph{Remark 2.} 
We note that \cite{GuoHT+21} also give (slightly worse) sample complexity bounds for a broader class of problems that need not satisfy any monotonicity property, but have a finite support-size  $k$. Using the approaches described above, we can convert these sample-complexity  guarantees to obtain $O\left( \sqrt{nkT\log(nT)} \right)$ regret 
under  full-feedback and $O\left((nkT)^{2/3}(\log T)^{1/3}\right)$ regret under  semi-bandit (and bandit) feedback.

\section{Lower Bounds}\label{app:lower-bound} 
We now state some   regret lower bound results for monotone problems under the \emph{semi-bandit} and \emph{binary} feedback models. We highlight particular applications for which lower bounds are already known, and observe that since these problems can be captured by our framework, the corresponding lower bounds also apply to our setting.

\begin{theorem} Any   semi-bandit    learning algorithm for  stochastic monotone problems has  regret at least 
$\Omega\left( \sqrt{n T}\right).$
\end{theorem}
 \begin{proof} 
We observe that the classical multi-armed bandit (MAB) problem with Bernoulli rewards is a special case of our setting. Indeed, consider the stochastic problem with $n$ independent  Bernoulli r.v.s $\{\varX_i\}_{i=1}^n$ where a policy is allowed to probe only one random variable and receives reward $\varX_i$ for probing $i$. This problem is clearly down-monotone, and it is equivalent to MAB.  As shown in \citet{lattimore2020bandit}, any algorithm for MAB has regret   at least $\Omega(\sqrt{nT})$.
\end{proof}
Hence, our upper bounds match this lower bound  up to a factor of $\sqrt{n}$ and logarithmic terms. 
The next result, also observed in \cite{SW2024},  shows that stochastic problems under \emph{binary} feedback with continuous distributions are fundamentally harder to learn than with discrete distributions.

\begin{theorem} Any        learning algorithm for  stochastic monotone problems with 
 binary feedback and 
  continuous  distributions has  regret at least  $\Omega\left(T^{\frac{2}{3}}\right).$
\end{theorem}
\begin{proof}
Consider the fixed-order SPM under a uniform matroid of rank one, which is a down-monotone stochastic problem. \citet{LemeSTW23} proved an $\Omega(1/\epsilon^3)$ query complexity lower bound for this problem under continuous distributions with binary-threshold queries. This directly implies an $\Omega(T^{2/3})$ regret lower bound for online learning for SPM  with binary feedback. Suppose (for a contradiction) that there is an online algorithm  with regret $o(T^{\frac{2}{3}})$. Then after $T$ rounds, we can achieve $o(T^{-\frac{1}{3}})$ per-round regret, which means that we can obtain a policy with $\epsilon$ additive error  in $T=o(1/\epsilon^3)$ rounds, which   contradicts  the sample complexity result.
\end{proof}

\end{document}